\documentclass[sigconf]{acmart}
\usepackage{amsmath}
\usepackage{amsthm}
\usepackage{graphicx}
\usepackage{wrapfig}
\usepackage{subfig}
\usepackage{mathtools}
\usepackage{relsize}
\usepackage{pifont}
\theoremstyle{plain}
\newtheorem{theorem}{Theorem}[section]
\newtheorem{proposition}[theorem]{Proposition}
\newtheorem{lemma}[theorem]{Lemma}

\newtheorem{definition}[theorem]{Definition}
\newtheorem{assumption}[theorem]{Assumption}
\newtheorem{remark}[theorem]{Remark}
\AtBeginDocument{%
  }
\newcommand{\cmark}{\ding{51}}
\setcopyright{acmlicensed}
\copyrightyear{2026}
\acmYear{2026}
\setcopyright{cc}
\setcctype{by}
\acmConference[KDD 2026] {Proceedings of the 32nd ACM SIGKDD Conference on Knowledge Discovery and Data Mining V.2}{August 9--13, 2026}{Jeju Island, Republic of Korea.}
\acmBooktitle{Proceedings of the 32nd ACM SIGKDD Conference on Knowledge Discovery and Data Mining V.2 (KDD 2026), August 9--13, 2026, Jeju Island, Republic of Korea}
\acmISBN{XXX-X-XXXX-XXXX-X/2026/08}
\acmDOI{XXXXXXX.XXXXXXX}
\begin{document}

\title{Treatment Effect Estimation with Differentiated Networked Effect on Graph Data}

\author{Xiaofeng Lin}
\authornote{The corresponding author.}
\affiliation{%
\institution{Kyoto University}
  \city{Kyoto} 
  \country{Japan}
}
\email{lxf@ml.ist.i.kyoto-u.ac.jp }
\orcid{https://orcid.org/0000-0002-4147-3037}

\author{Han Bao}
\affiliation{%
\institution{The Institute of Statistical Mathematics}
  \city{Tokyo} 
  \country{Japan}
}
\affiliation{%
\institution{Tohoku University}
  \city{Sendai} 
  \country{Japan}
}
\affiliation{%
\institution{RIKEN AIP}
 \city{Tokyo} 
 \country{Japan}
}
\email{bao.han@ism.ac.jp}
\orcid{https://orcid.org/0000-0002-4473-2604}

\author{Hisashi Kashima}
\affiliation{%
\institution{Kyoto University}
  \city{Kyoto}
  \country{Japan}
}
\email{kashima@i.kyoto-u.ac.jp}
\orcid{https://orcid.org/0000-0002-2770-0184}

\renewcommand{\shortauthors}{Lin et al.}

\begin{abstract}
 Estimating individual treatment effect (ITE) from observational graph data is crucial for decision-making in the fields such as commerce and medicine.
This task is challenging due to \textit{interference}, where individual outcomes can be influenced by the treatments and covariates of their neighbors.
Existing methods attempt to model such interference for accurate ITE estimation. 
However, a critical issue is often overlooked: \textit{differentiated networked effect} (DNE), an effect caused by local networks consisting of neighbors with varying importance and scales. 
Capturing DNE is vital; otherwise, we will end up with imprecise ITE estimation due to an erroneous characterization of interference, which can result in misguided decisions.
To address this challenge, we propose a novel interference modeling mechanism that incorporates two partial attention mechanisms and a message amplifier.
The partial attention mechanisms automatically estimate the importance of different neighbors in contributing to interference, while the message amplifier adjusts the results of the interference modeling mechanism based on the scale of neighbors, 
all of which enables the model to capture DNE. 
Experiments on three real-world graphs demonstrate that our methods outperform existing approaches for ITE estimation from graph data, which corroborates the importance of explicitly capturing DNE.
\end{abstract}

\ccsdesc[500]{Mathematics of computing~Causal networks}
\ccsdesc[500]{Mathematics of computing~Graph algorithms}
\ccsdesc[500]{Information systems~Social networks}

\keywords{Causal Inference; Treatment Effect Estimation;  Interference}

\maketitle

\section{Introduction}\label{sec:int}
Treatment effect estimation from graph data has been applied to decision-making in various areas, such as medicine~\cite{chang2023estimating,ma2022assessing,schnitzer2022estimands} and commerce~\cite{nabi2022causal}.
For example, it enables business owners to assess whether an advertisement stimulates a customer to purchase.
This supports making reasonable decisions on promotional strategies. A crucial task in this context is estimating the individual treatment effect (ITE),\footnote{Sometimes known as conditional average treatment effect (CATE)~\cite{ma2022learning, pmlr-v70-shalit17a}.} which quantifies the differences in the outcome of an individual with and without treatment~\cite{pmlr-v70-shalit17a}.

We aim to estimate ITE from observational graph data, which typically include covariates,   treatments, outcomes of individuals, and a network structure among individuals.
In this case, the outcomes of individuals can receive influence from treatments and covariates of their neighbors~\cite{pmlr-v130-ma21c,rakesh2018linked}, a phenomenon known as \emph{interference}~\cite{rakesh2018linked}.
Such interference can propagate among individuals and their multi-hop neighbors, 
which is referred to as \emph{networked interference}~\cite{pmlr-v130-ma21c}.
The multi-hop neighbors of an individual, along with connections among them, form a local network that contributes to networked interference received by the individual. 
Properly modeling networked interference is critical; otherwise, we end up with unreasonable decisions due to inaccurate ITE estimation.

\begin{figure}
	\centering
\includegraphics[width=1.0\columnwidth]{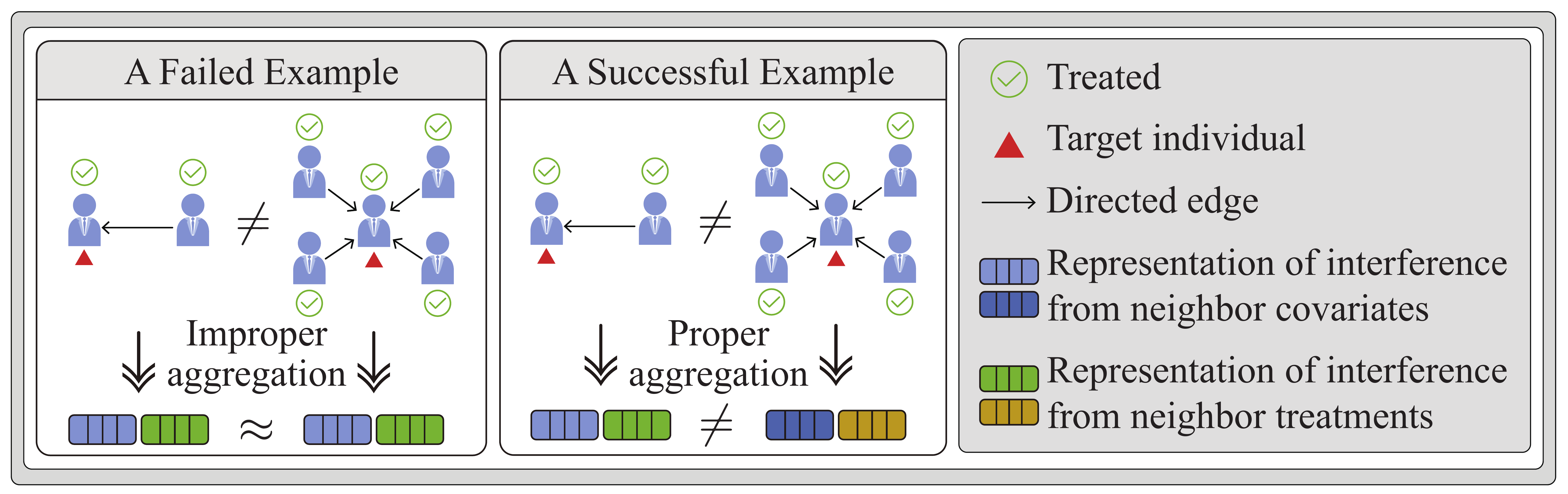} 
  \caption{
    Comparison of representations generated by improper (left) and proper (right) interference modeling mechanisms. 
    Individuals with similar clothes have similar covariates.
    Each example shows two similar target individuals exposed to local networks with different scales of neighbors, which leads to different levels of interference. Proper  modeling mechanisms capture this by generating distinct representations. In contrast, improper mechanisms, such as mean aggregation, cannot generate distinct representations.
  }  \label{exp:fails}
  \end{figure}

\begin{figure*}
	\centering
\includegraphics[width=1.5\columnwidth]{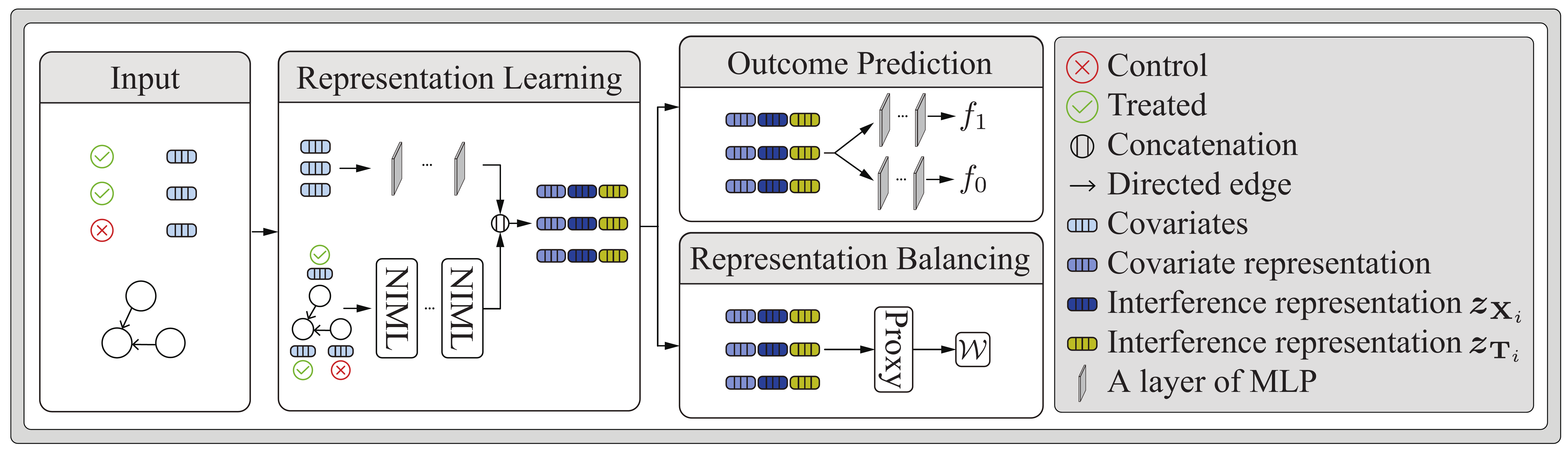}
\caption{Architecture of GITE. Here, we show an example of three individuals. NIML represents an NIM layer. %
    }
 \label{fig:arGITE}
\end{figure*}

Despite the significant contributions of previous studies in demonstrating the effectiveness of modeling interference for ITE estimation from observational graph data, their modeling mechanisms still have limitations in properly capturing networked interference, which can result in imprecise ITE estimation.
Several methods model interference (detailed in Section~\ref{sec:relatedwork}) by applying a mean aggregation or graph convolutional network (GCN)~\cite{kipf2016semi}, such as  \citet{pmlr-v130-ma21c}, \citet{jiang2022estimating}, \citet{chen2024doubly}, and \citet{LIN2024SITE}.
However, such aggregation mechanisms cannot fully capture networked interference, since a critical issue is overlooked: \textit{differentiated networked effect} (DNE),
an effect caused by local networks consisting of neighbors with varying  importance and scales.
Specifically, DNE consists of two key sub-issues.
(I) The importance of different neighbors in contributing to interference varies~\cite{huang2023modeling,lin2023estimating, ma2022learning,adhikari2025inferring}.\footnote{In literature~\cite{huang2023modeling,lin2023estimating, ma2022learning,adhikari2025inferring}, it is called  heterogeneous interference or influence. However, heterogeneous interference or influence do not consider the sub-issue (II) of DNE. Therefore, DNE can be considered as a  refinement issue of them.} 
For instance,  purchase behaviors of customers are usually more significantly influenced by their closer friends than by others.
(II) The scale of neighbors varies, leading to different levels of  interference (see Figure~\ref{exp:fails}). An individual with many neighbors may experience more severe interference than one with few neighbors.
For example, purchase behaviors of customers may be more significantly influenced by advertisements shared by many friends than by those shared by only a few friends.  
Methods based on mean aggregation or GCN do not address both issues (I) and (II), as they lack explicit mechanisms to automatically estimate the importance of interference from each neighbor, and  may fail to generate distinct representations for interference received by individuals from their local networks with different scales of neighbors.
Although a line of work takes issue (I) into account by applying a graph attention mechanism (GAT)~\cite{2017gat} to estimate the importance of different neighbors~\cite {huang2023modeling,lin2023estimating,ma2022learning,zhao2024learning}, their methods also do not fully address issue (II), as they may degenerate to a mean aggregation when individuals are similar (see Figure~\ref{exp:fails}).
A detailed proof in Appendix~\ref{proof:attfail} shows that capturing DNE remains a challenge for most existing interference modeling methods.
If DNE is not properly captured by jointly addressing issues (I) and (II), ITE estimation deteriorates and the subsequent decision making is misguided.

To overcome the challenge  of capturing DNE, we propose graph-based individual treatment effect estimation (GITE),
which models the propagation of networked interference while capturing DNE.
A novel networked interference modeling (NIM) layer forms the core of our approach. It is designed to capture DNE through two partial attention mechanisms and a message amplifier.
Specifically, we design two partial attention mechanisms: individual partial attention (IPAtt) and structure partial attention (SPAtt), which are intended to adaptively capture the varying contributions of neighbors to interference based on two key factors that influence their relative importance.
IPAtt estimates the individual partial importance of interference between two individuals based on their interference-related information, which assists in addressing issue (I).
SPAtt estimates the structure partial importance of interference between two individuals based on the structures of their local networks, which 
assists in addressing both issues (I) and (II). 
Subsequently,  the NIM layer conducts aggregations using the estimated partial importance and integrates the aggregated results adaptively through a learnable summary function. 
Each partial attention mechanism can be implemented using either GAT~\cite{2017gat} or the attention mechanism of Transformer~\cite{vaswani2017attention,NEURIPS2021_f1c15925}.  To precisely capture structural information on the local network of every individual for applying the SPAtt mechanism, we apply a graph isomorphism network (GIN)~\cite{xu2019powerful}. 
To address issue (II), we design a message amplifier to vary integrated results based on the degree of individuals, which is inspired by~\citet{corso2020principal}.
Details of the NIM layer are described in Section~\ref{sec:RL}. 
 Furthermore, we propose a representation balancing strategy for ITE estimation from observational graph data, as detailed in Section~\ref{sec:RB}. 
 We theoretically analyze the error bound of ITE estimation based on this strategy in Appendix~\ref{app:bound}.

 We summarize three contributions of this study, as follows: \vspace{-0pt}
 \begin{itemize}
     \item We propose the NIM layer to address the challenging issue of DNE, and further introduce a representation balancing strategy for ITE estimation from observational graph data.
     \item We discuss that capturing DNE remains a challenge for most existing methods (see Appendix~\ref{proof:attfail}) and theoretically analyze the error bound of ITE estimation based on the proposed balancing strategy (see Appendix~\ref{app:bound}).
     \item Results of extensive experiments reveal that the proposed method outperforms existing methods in ITE estimation with networked interference, which suggests the importance of capturing DNE.
 \end{itemize}
\section{Related work}\label{sec:relatedwork}
\textbf{ITE estimation from observational data with interference.} 
Although many studies model interference by assuming a \textit{neighbor interference}, where interference exists among close neighbors only~\cite{cai2023generalization,chen2024doubly,huang2023modeling,jiang2022estimating,rakesh2018linked,RePEc:arx:papers:1906.10258,wucausal}, real-world data often involves networked interference, where interference propagates widely among individuals and their multi-hop yet influential  neighbors~\cite{lin2023estimating,LIN2024SITE,ma2022learning,pmlr-v130-ma21c,sui2024invariant}. 
Specifically, existing methods model neighbor interference by applying a mean aggregation~\cite{cai2023generalization,chen2024doubly,doi:10.1080/01621459.2020.1768100,forastiere2022estimating,jiang2022estimating}, 
GCN~\cite{cai2023generalization,chen2024doubly,huang2023modeling,jiang2022estimating,wucausal}, or GAT~\cite{huang2023modeling,zhao2024learning}.
Several studies model the propagation of  networked interference through a mean aggregation~\cite{pmlr-v130-ma21c} and GCN~\cite{pmlr-v130-ma21c, adhikari2025inferring}. 
To accelerate the training of GNN-based estimators, \citet{LIN2024SITE} aggregates interference-related information before training. 
To estimate ITE from more convoluted graphs,
\citet{ma2022learning}  and \citet{lin2023estimating} propose methods to estimate ITE from hypergraphs and heterogeneous graphs, respectively.
 They use GAT~\cite{2017gat} to model interference when they are applied to an ordinary graph. Despite the valuable contributions that previous studies have made to ITE estimation with interference, their methods still face challenges in capturing DNE, as detailed in Appendix~\ref{proof:attfail}. A comparison table and studies on treatment effect estimation in other settings are detailed in Appendix~\ref{app:relatedwork}.

\vspace{1em}
\noindent\textbf{Graph machine learning methods.} 
Beyond GNN-based methods for modeling interference in ITE estimation,  several methods have been proposed in the graph machine learning (GML) community. Although these methods contribute significantly to GML tasks, they are unable to estimate ITE from observational graph data with DNE.
Sum aggregation~\cite{liu2024empowering,xu2019powerful} cannot capture DNE, as they do not consider the different importance of neighbors in contributing to interference. 
Pooling-based methods have also been explored~\cite{liu2022graph}, where the max-pooling method~\cite{hamilton2017inductive} is one of the most widely used methods.
The max-pooling operation does not capture DNE, as proved in Appendix~\ref{proof:attfail}. 
Although several studies proposed methods to enhance the expressive power of representations for addressing some complex GML tasks~\cite{corso2020principal,ma2023graph},  ITE estimation from observational graph data differs from standard GML tasks due to confounding and interference biases (see Section~\ref{sec:problemsetting}), which are absent in GML~\cite{jiang2022estimating}.  
Thus, GML methods lack mechanisms to address these biases, which can result in biased ITE estimation. 
For a broader overview of GML, refer to recent literatures~\cite{liu2022graph,liu2024empowering}.

\section{Problem setting}
\label{sec:problemsetting}
In this study, we aim to estimate ITE from observational graph data with networked interference. 
We use $\boldsymbol{x}_i \in \mathbb{R}^{c}$ to denote the covariates of the individual $i$, $t_i \in \{0,1\}$ to denote the treatment assigned to the individual $i$, $y_{i} \in \mathbb{R}$ to denote the factual or observed outcome with assigned $t_i$, and $N$ to denote the number of individuals.
Let $\boldsymbol{X}\in \mathbb{R}^{N\times c}$ be the covariates of all individuals, $\boldsymbol{T}=[t_1,...,t_N]$ be all treatment assignments, and $\boldsymbol{Y}=[y_1,...,y_N]$ be all factual outcomes.
Moreover, we use uppercase letters (e.g., $Y$) to denote random variables.  We show a notation table in Appendix~\ref{app:notationtable}.

\vspace{1em}
\noindent\textbf{Observational graph data.}
 $\mathcal{D} = (\boldsymbol{X},\boldsymbol{T},\boldsymbol{Y},\boldsymbol{A})$ denotes an observational graph data,  
  where $\boldsymbol{A} \in \{0,1\}^{N \times N}$ denotes the adjacency matrix of a directed graph.
If there is an edge from an individual $k$ to an individual $i$, $A_{ik} = 1$;
otherwise, $A_{ik} = 0$. 
Let $\bold{N}_{i}$ denote the set of neighbors of the individual $i$,  $\bold{G}_i$ denote the set of related individuals who can reach the individual $i$ in the graph $\boldsymbol{A}$~\cite{cheng2012k},
$\boldsymbol{x}_{\bold{G}_i}=\{\boldsymbol{x}_k\mid k\in \bold{G}_i\}$ denote the set of covariates in $\bold{G}_i$, and $\boldsymbol{t}_{\bold{G}_i}=\{t_k\mid k\in \bold{G}_i\}$ denote the set of treatments in   $\bold{G}_i$. Importantly, only individuals in $\bold{G}_i$ can interfere with the individual $i$. 

  \vspace{1em}
\noindent\textbf{Challenges.}
ITE estimation from observational graph data suffers from four challenges:
\vspace{-5pt}
\begin{itemize}
    \item \textbf{Networked interference.} 
In a graph, the outcome of an individual $i$ can receive influence from covariates $\boldsymbol{x}_{\bold{G}_i}$ and treatments $\boldsymbol{t}_{\bold{G}_i}$.
 This phenomenon is referred to as networked interference~\cite{pmlr-v130-ma21c}.
We consider there exists an issue of DNE, see Section~\ref{sec:int}. 
\item \textbf{Counterfactual outcome.}
Counterfactual outcome is the outcome with an alternative treatment $1-t$ and unobserved from observational data but needed for ITE estimation~\cite{yao2021survey}.
\item \textbf{Confounding bias.}
Confounders of an individual are a part of covariates $\boldsymbol{x}_i$ that affect the treatment assignment and outcome jointly~\cite{yao2021survey},
which can introduce confounding bias in ITE estimation.
For example, consider a scenario where customers are treated with advertisements.
Younger customers may be more likely to receive advertisements and also go shopping than elderly customers.
In this case, age acts as a confounder.
Due to the existence of networked interference among individuals,
many studies for ITE estimation from a graph suggest that treatment assignment of an individual can be affected by confounders of his/her related individuals $\boldsymbol{x}_{\bold{G}_i}$~\cite{chu2021graph,guo2020learning,guo2021ignite,ma2021deconfounding}.
In this study, we call confounders of an individual from $\boldsymbol{x}_i$ as \textit{individual confounders} and the confounders from $\boldsymbol{x}_{\bold{G}}$ as \textit{networked confounders}.
In observational graph data, confounding bias results in $p(t\mid \boldsymbol{x},\boldsymbol{x}_{\bold{G}}) \neq p(1-t\mid\boldsymbol{x},\boldsymbol{x}_{\bold{G}} )$~\cite{jiang2022estimating}.
\item \textbf{Interference bias.} A bias issue might also exist in networked interference~\cite{jiang2023cf},
resulting in $p(\boldsymbol{t}_{\bold{G}}\mid \boldsymbol{x},\boldsymbol{x}_{\bold{G}},t)\neq p(\boldsymbol{t}_{\bold{G}}\mid \boldsymbol{x},\boldsymbol{x}_{\bold{G}},1-t)$, which is called \textit{interference bias}.
This can result in additional bias in ITE estimation.
For instance, younger customers are not only more likely to be treated but also tend to have younger friends, who in turn have higher exposure to advertisements,
whereas elderly customers have more elderly friends.
\end{itemize}

\begin{figure}
	\centering
 \includegraphics[width=0.35\textwidth]{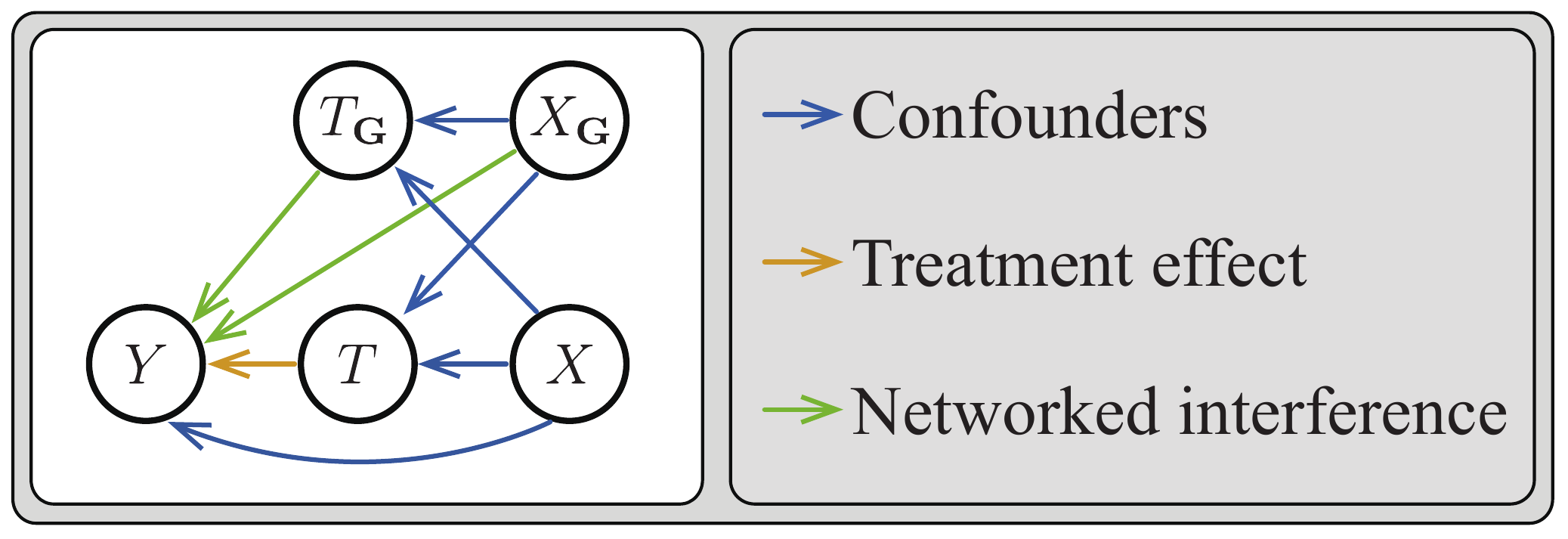}
	\caption{An example of a causal graph for an individual in graph data. 
    In the causal graph, 
     {\color[cmyk]{0.88,0.65,0,0}{blue}} arrows represent the effect caused by confounders, {\color[cmyk]{0.54,0.13,0.96,0}{green}} arrows represent the effect caused by networked interference, and  {\color[cmyk]{0.15,0.40,0.94,0}{yellow}} represent the effect caused by treatments assigned to the individual. 
     Here, $T$ and $T_{\bold{G}}$ are associated~\cite{zhao2024learning}, but there is no causal edge between them.
    }
 \label{exp:causalgraph}
\end{figure}

\vspace{1em}
\noindent \textbf{ITE estimation from observational graph data.}
In observational graph data $\mathcal{D}=(\boldsymbol{X},\boldsymbol{T},\boldsymbol{Y},\boldsymbol{A})$, we assume the existence of both confounders and networked interference with DNE. A causal graph is shown in Figure~\ref{exp:causalgraph}.
The potential outcomes of the individual $i$ with the individual treatment $t=1$ and $t=0$,
along with treatments of related individuals $\boldsymbol{t}_{\bold{G}_i}$, are denoted by $y_i(1,\boldsymbol{t}_{\bold{G}_i})$ and $y_i(0,\boldsymbol{t}_{\bold{G}_i})$, 
 respectively.
Then, ITE with confounders and networked interference can be defined as:
\begin{equation}
\tau_i\coloneqq \mathbb{E}\left[Y(1,\boldsymbol{t}_{\bold{G}_i})- Y(0,\boldsymbol{t}_{\bold{G}_i})\mid \boldsymbol{x}_i,\boldsymbol{x}_{\bold{G}_i}\right].    
\end{equation}
This definition is extended from that of ITE with neighbor interference in \citet{jiang2022estimating}.

\vspace{1em}
\noindent \textbf{Identifiability of ITE.}
 We now discuss that ITE is identifiable from observational graph data with a set of assumptions. 
First, we extend consistency assumption to networked interference~\cite{doi:10.1080/01621459.2020.1768100}, as follows:
 \begin{assumption}\label{ass2}
 $y_i = y_i(t_i, \boldsymbol{t}_{\bold{G}_i})$ for the individual $i$ with $t_i$ and $\boldsymbol{t}_{\bold{G}_i}$.
\end{assumption}
This assumption means that the potential outcome is equal to the observed outcomes with given $t_i$ and  $\boldsymbol{t}_{\bold{G}_i}$. 
Then, we extend unconfoundedness assumption for confounders of neighbors~\cite{jiang2022estimating} to networked confounders:
\begin{assumption}\label{ass3} 
For any individual $i$, given the covariates of the individual and individual's related individuals,  treatments of the individual and individual's related individuals are independent of potential outcomes, i.e., $T_i,T_{\bold{G}_i} \perp   Y(1,\boldsymbol{t}_{\bold{G}_i}),Y(0,\boldsymbol{t}_{\bold{G}_i}) \mid \boldsymbol{x}_i,\boldsymbol{x}_{\bold{G}_i}$.
\end{assumption}
This assumption says that, confounders that describe the difference between the treated and the control groups are observed in individual covariates and covariates of related individuals.
Lastly, we extend the overlap assumption for neighbor interference~\cite{chen2024doubly} to networked interference, as follows:
\begin{assumption}\label{ass4} 
    Given the covariates of any individual and individual's related individuals,  
     the treatment pair $(t_i,\boldsymbol{t}_{\bold{G}_i})$ has a non-zero probability, i.e., $0<p(t_i,\boldsymbol{t}_{\bold{G}_i}\mid \boldsymbol{x}_i,\boldsymbol{x}_{\bold{G}_i})<1, \forall \boldsymbol{x}_i, \forall  \boldsymbol{x}_{\bold{G}_i}$.
\end{assumption}
This assumption means that the treatment assignment is nondeterministic~\cite{chen2024doubly}.

\begin{theorem} \label{Theorem:identi}   
With assumptions  \ref{ass2}, \ref{ass3}, and \ref{ass4}, ITE is identifiable from observational graph data.
\end{theorem}

We prove Theorem~\ref{Theorem:identi} as follows:
\begin{proof} 

\begin{equation}
\begin{aligned}
 ~&\quad\; \mathbb{E}\left[Y(1,\boldsymbol{t}_{\bold{G}_i})-Y(0,\boldsymbol{t}_{\bold{G}_i})\mid \boldsymbol{x}_i,\boldsymbol{x}_{\bold{G}_i}\right] \\
 &=\mathbb{E}\left[Y(1,\boldsymbol{t}_{\bold{G}_i})\mid \boldsymbol{x},\boldsymbol{x}_{\bold{G}_i}\right]-\mathbb{E}\left[Y(0,\boldsymbol{t}_{\bold{G}_i})\mid \boldsymbol{x}_i,\boldsymbol{x}_{\bold{G}_i}\right]\\
 & =  \mathbb{E}\left[Y(1,\boldsymbol{t}_{\bold{G}_i})\mid \boldsymbol{x}_i,\boldsymbol{x}_{\bold{G}_i},1,\boldsymbol{t}_{\bold{G}_i}\right]-\mathbb{E}\left[Y(0,\boldsymbol{t}_{\bold{G}_i})\mid \boldsymbol{x}_i,\boldsymbol{x}_{\bold{G}_i},0,\boldsymbol{t}_{\bold{G}_i}\right]\\ & =  \mathbb{E}\left[Y\mid \boldsymbol{x}_i,\boldsymbol{x}_{\bold{G}_i},1,\boldsymbol{t}_{\bold{G}_i}]-\mathbb{E}[Y\mid \boldsymbol{x}_i,\boldsymbol{x}_{\bold{G}_i},0,\boldsymbol{t}_{\bold{G}_i}\right], \nonumber
\end{aligned}
\end{equation}
where the second equality is based on the assumption~\ref{ass3}  and the third equality is based on the assumptions \ref{ass2} and \ref{ass4}. 
Let $\bold{C}_{i}$ denote individuals that cannot reach the individual $i$ in the graph $\boldsymbol{A}$, $\boldsymbol{x}_{\bold{C}_{i}}$ denote covariates of individuals  $\bold{C}_{i}$, and $\boldsymbol{t}_{\bold{C}_{i}}$ denote  treatments of individuals $\bold{C}_{i}$. In this case, the covariates $\boldsymbol{x}_{\bold{C}_{i}}$ and treatments $\boldsymbol{t}_{\bold{C}_{i}}$ cannot interfere with the outcome of the individual $i$, which means that the outcome of individual $i$ can only receive interference  from  $\boldsymbol{x}_{\bold{G}_{i}}$ and $\boldsymbol{t}_{\bold{G}_{i}}$ in the graph $\boldsymbol{A}$. 
This tells that,  given $\boldsymbol{x}_i$, $t_i$, $\boldsymbol{x}_{\bold{G}_i}$, and $\boldsymbol{t}_{\bold{G}_i}$, the outcome of the individual $i$ does not receive interference from  $\boldsymbol{x}_{\bold{C}_{i}},\boldsymbol{t}_{\bold{C}_{i}}$ and $\boldsymbol{A}$. 
Then, we have: 
\begin{equation}
\begin{aligned}
 &\;\mathbb{E}\left[Y\mid \boldsymbol{x}_i,\boldsymbol{x}_{\bold{G}_i},t,\boldsymbol{t}_{\bold{G}_i}\right]\\=  &\;\mathbb{E}\left[Y\mid \boldsymbol{x}_i,t,\boldsymbol{x}_{\bold{G}_i},\boldsymbol{t}_{\bold{G}_i},\boldsymbol{x}_{\bold{C}_i},\boldsymbol{t}_{\bold{C}_i},\boldsymbol{A}\right]\\=&\;\mathbb{E}\left[Y\mid \boldsymbol{x}_i,t,\boldsymbol{X},\boldsymbol{T},\boldsymbol{A}\right], 
\end{aligned}
\end{equation} where $\boldsymbol{x}_{\bold{G}}$ and $\boldsymbol{x}_{\bold{C}}$ constitute covariates  $\boldsymbol{X}$, while $\boldsymbol{t}_{\bold{G}}$ and $\boldsymbol{t}_{\bold{C}}$ constitute treatments $\boldsymbol{T}$ for all individuals.
\end{proof}
This shows that we can recover the ITE from observational graph data. However, as the propagation mechanism of networked interference among individual $i$ and its related individuals $\bold{G}_i$ is unknown and complex, it is hard to model networked interference received by individual $i$ from $\boldsymbol{x}_{\bold{G}_i}$ and  $\boldsymbol{t}_{\bold{G}_i}$ on the data $(\boldsymbol{X},\boldsymbol{T},\boldsymbol{A})$.
A common choice for addressing this issue is to generate representations of $\boldsymbol{x}_{\bold{G}_i}$ and  $\boldsymbol{t}_{\bold{G}_i}$ by applying some aggregation function~\cite{LIN2024SITE,ma2022learning,pmlr-v130-ma21c}.
 In this case, it is crucial to design a proper aggregation function; otherwise, we will erroneously characterize networked interference with DNE received by individuals from $\boldsymbol{x}_{\bold{G}_i}$ and $\boldsymbol{t}_{\bold{G}_i}$, which can result in imprecise ITE estimation.

\section{Proposed method}\label{Sec3:proposed}
Given observational graph data $\mathcal{D} = (\boldsymbol{X},\boldsymbol{T},\boldsymbol{Y},\boldsymbol{A})$,
we aim to estimate ITE with confounders and networked interference,
which incorporates DNE.
To this end, we propose GITE, as illustrated in Figure~\ref{fig:arGITE}.\footnote{The code is released in \hyperlink{https://github.com/LINXF208/GITE/tree/main}{https://github.com/LINXF208/GITE/tree/main}.}
GITE contains three key modules:
representation learning, representation balancing, and outcome prediction modules.
The representation learning module generates covariate representations of $\boldsymbol{x}_i$, as well as representations of networked interference derived from $\boldsymbol{x}_{\bold{G}_i}$ and $\boldsymbol{t}_{\bold{G}_i}$,
which are called interference representations.
 
\subsection{Representation learning module}
\label{sec:RL}
In this section, we aim to generate covariate and interference representations.
We generate covariate representations by an MLP, which are expected to capture individual confounders.
To capture networked interference, it is important to model the propagation of interference among individuals.
We model it by leveraging some layer-by-layer aggregation function to aggregate interference-related information of individuals within the same local network for every individual.
While modeling interference representation, we need to carefully design the aggregation function due to the existence of DNE, which consists of two sub-issues.
(I) The importance of different neighbors in contributing to interference varies~\cite{huang2023modeling,lin2023estimating,ma2022learning}. 
(II) The scale of neighbors varies, leading to different levels of  interference (see Figure~\ref{exp:fails}). 
Despite several powerful aggregation functions proposed for modeling interference, these methods still have limitations in capturing DNE, as discussed in Appendix~\ref{proof:attfail}.
For example, the mean aggregation~\cite{pmlr-v130-ma21c} and GCN~\cite{kipf2016semi} cannot capture DNE, as they do not address both issues (I) and (II). 
Although GAT-based methods~\cite{huang2023modeling,lin2023estimating,ma2022learning} estimate importance based on individual information to take issue (I) into account, they can not fully address issue (II), as they may degenerate to a mean aggregation when individuals are similar.
 Sum aggregation~\cite{xu2019powerful} can address issue (II), but
does not address issue (I) and may suffer from
a numerical explosion issue in some graphs with many connections, as shown in the results of ablation experiments (see results of GITE$_{\rm{NA}}$ in Table~\ref{table4:ablation}). 
Therefore,  a proper aggregation function needs to address both issues (I) and (II), while avoiding the risk of numerical explosion.
\begin{figure*}
	\centering
\includegraphics[width=1.6\columnwidth]{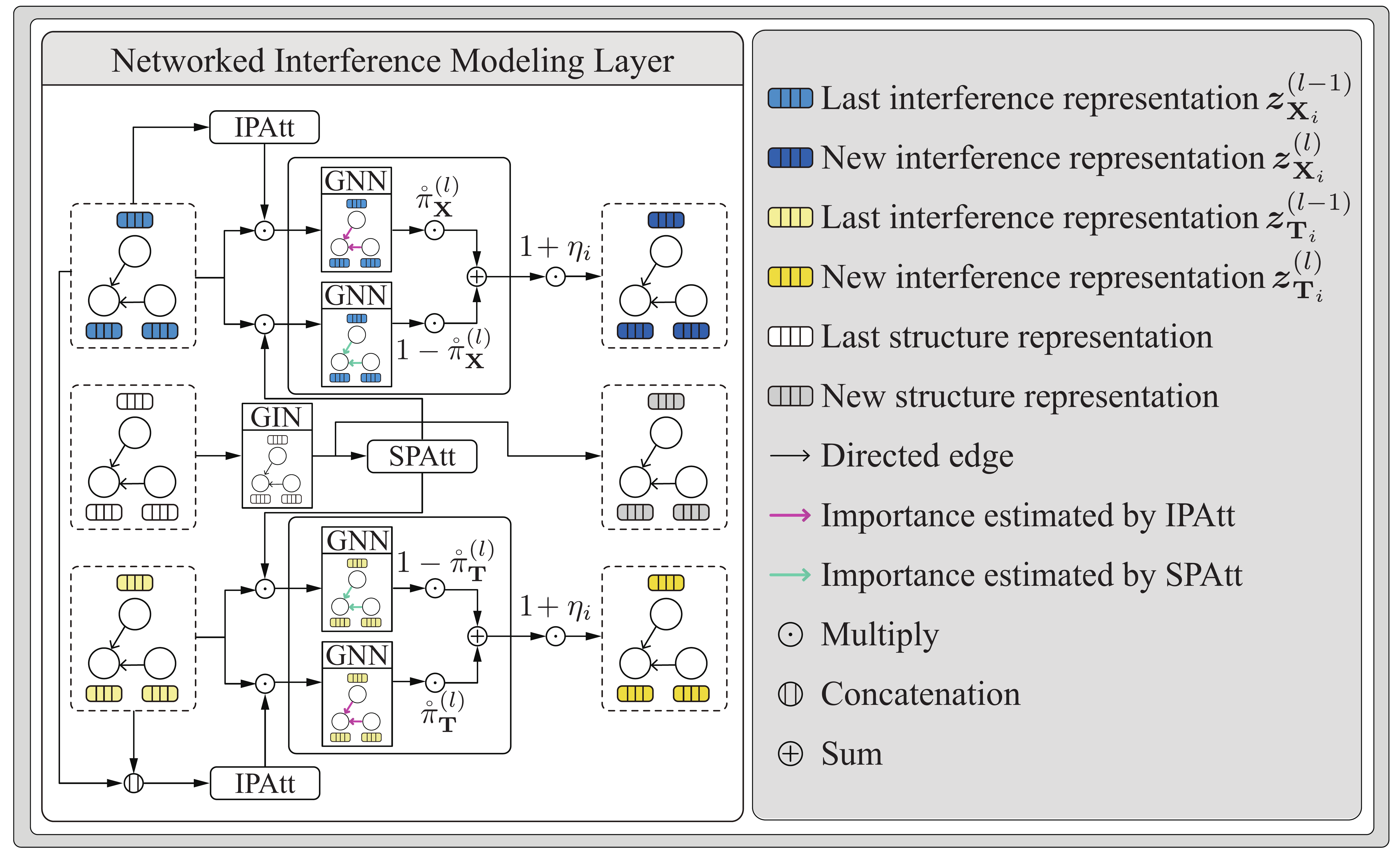}
  \caption{Architecture of an NIM layer. We show an example where the underlying graph consists of three individuals. 
  }
  \label{fig:arNIML}
\end{figure*}

To address these nontrivial issues, 
we design an NIM layer, as illustrated in Figure~\ref{fig:arNIML}. 
To address issue (I), an NIM layer contains two partial attention mechanisms: IPAtt and SPAtt mechanisms, which are intended to adaptively capture the varying contributions of neighbors to interference based on two key factors that influence the importance of interference. 
 Specifically, IPAtt mechanism estimates the importance of interference among individuals based on their interference representations, 
 whereas 
 SPAtt mechanism estimates the importance of interference based on the structures of their local networks, which can also contribute to addressing issue (II).
 The two estimated partial importance guides separate aggregations, the results of which are adaptively integrated by a learnable summary function for each layer.
To tackle issue (II), the NIM layer applies a message amplifier to vary integrated results with the degrees of individuals and update interference representations. 
By aggregating interference-related information with NIM layers,  we can jointly address issues (I) and (II), which enables the model to capture DNE.

Specifically, given covariates $\boldsymbol{X}$, treatments $\boldsymbol{T}$, and network $\boldsymbol{A}$, we aim to generate covariate representation $\boldsymbol{z}_i$, interference representation 
$\boldsymbol{z}_{\bold{X}_i}$ of $\boldsymbol{x}_{\bold{G}_i}$, and interference representation $\boldsymbol{z}_{\bold{T}_i}$ of $\boldsymbol{t}_{\bold{G}_i}$. 
Here, information of network $\boldsymbol{A}$ is encoded into $\boldsymbol{x}_{{\bold{G}_i}}$ and $\boldsymbol{t}_{{\bold{G}_i}}$.
We use two aggregation functions to generate $\boldsymbol{z}_{\bold{X}_i}$ and $\boldsymbol{z}_{\bold{T}_i}$ separately due to the different dimensions and scales among covariates and treatments.
Let $\boldsymbol{z}_{\bold{S}_i}$ denote the representation of the structure of the local network of the individual $i$. We use $\boldsymbol{Z}$ with a subscript to denote corresponding representations of all individuals, such as $\boldsymbol{Z}_{\bold{X}}$ and $\boldsymbol{Z}_{\bold{S}}$, and  
use a superscript $(l)$ to denote the representation generated by the $l$-layer, such as $\boldsymbol{z}_{\bold{S}_i}^{(l)}$. Let $\sigma$ be an activation function,  $\boldsymbol{W}^{(l)}$ with any subscript be a learnable parameter matrix,  and $\boldsymbol{w}$ with any subscript be a learnable parameter vector. 

We generate $\boldsymbol{z}_i$ independently by an MLP, i.e., $\textrm{MLP}(\boldsymbol{x}_i) = \boldsymbol{z}_i$, 
 as covariates are typically important to treatment assignment and outcome prediction of an individual.  

Now, we describe the architecture of an NIM layer that consists of IPAtt, SPAtt, an encoder $\phi_{\bold{S}}$ for generating $\boldsymbol{z}_{\bold{S}}$,
a covariate aggregation function $\phi_{\bold{X}}$ for generating $\boldsymbol{z}_{\bold{X}}$,
and a treatment aggregation function $\phi_{\bold{T}}$ for generating $\boldsymbol{z}_{\bold{T}}$.
Let $\pi_{\bold{S}}^{(l)}$ be a learnable one-dimensional parameter.
In each NIM layer, we first generate
$\boldsymbol{z}_{\bold{S}_i}^{(l)}$ by applying $\phi_{\bold{S}}\biggr(\boldsymbol{Z}_{\bold{S}}^{(l-1)},\bold{N}_i\biggr)$ with a GIN layer~\cite{xu2019powerful}, chosen for its strong ability to capture the structural information, as follows: 
\begin{equation}
    \boldsymbol{z}_{\bold{S}_i}^{(l)} = \sigma\Biggl(\boldsymbol{W}_{\bold{S}}^{(l)}\biggl((1+\pi_{\bold{S}}^{(l)})\cdot\boldsymbol{z}_{\bold{S}_i}^{(l-1)}+\sum_{k \in \bold{N}_i}\boldsymbol{z}_{\bold{S}_k}^{(l-1)}\biggr)\Biggr).
\end{equation}
Next, we estimate individual partial importance  $\alpha_{ik}^{\rm{in}}$ and structure partial importance  $\alpha_{ik}^{\rm{st}}$ by IPAtt and SPAtt mechanisms, respectively.
Let $\boldsymbol{p}_i^{\rm{in}}=\boldsymbol{z}_{\bold{X}_i}^{(l-1)}$ for the individual partial importance of $\phi_{\bold{X}}$,   $\boldsymbol{p}_i^{\rm{in}}=[\boldsymbol{z}_{\bold{X}_i}^{(l-1)}\Vert\boldsymbol{z}_{\bold{T}_i}^{(l-1)}]$ for that of $\phi_{\bold{T}}$ due to the consideration that covariates can influence $\boldsymbol{t}_{\bold{G}_i}$ (see Figure~\ref{exp:causalgraph}), and $\boldsymbol{p}_i^{\rm{st}}=\boldsymbol{z}_{\bold{S}_i}^{(l-1)}$ for structure partial importance estimation.  
Here, $\Vert$ denotes the concatenation operation.  
$\alpha_{ik}^{\rm{in}}$ and   $\alpha_{ik}^{\rm{st}}$ are estimated as follows:
\begin{align}
\alpha_{ik}^{\rm{in}}=\mathrm{Norm}\biggl(a\left(\boldsymbol{p}_i^{\rm{in}},\boldsymbol{p}_k^{\rm{in}}\right)\biggr), \; \alpha_{ik}^{\rm{st}}=\mathrm{Norm}\biggl(a\left(\boldsymbol{p}_i^{\rm{st}},\boldsymbol{p}_k^{\rm{st}}\right)\biggr), 
\label{eq:importance}
\end{align}
where $a$ estimates the importance between two individuals based on its inputs, which can be implemented by several attention mechanisms, such as GAT~\cite{2017gat} and the attention mechanism of Transformer~\cite{vaswani2017attention,NEURIPS2021_f1c15925}, as detailed in Appendix~\ref{app:detailsforatt}. $\rm{Norm}$ is a normalization operation, which adjusts estimated importance based on the sum of estimated importance of all neighbors $ \tilde{\bold{N}}_{i}=\bold{N}_{i}\cup{\{i\}}$ of the individual $i$ to prevent the risk of numerical explosion.

Subsequently, we update interference representations $\boldsymbol{z}_{\bold{X}_i}^{(l-1)}$  and 
$\boldsymbol{z}_{\bold{T}_i}^{(l-1)}$ by using two  functions  $\phi_{\bold{X}}\left(\boldsymbol{Z}_{\bold{X}}^{(l-1)},\tilde{\bold{N}}_i,\{\alpha^{\rm{in}}_{ik},\alpha^{\rm{st}}_{ik}\}_{k \in \tilde{\bold{N}}_i}\right)$ and $\phi_{\bold{T}}\left(\boldsymbol{Z}_{\bold{T}}^{(l-1)},\tilde{\bold{N}}_i,\{\alpha^{\rm{in}}_{ik},\alpha^{\rm{st}}_{ik}\}_{k \in \tilde{\bold{N}}_i}\right)$, respectively, 
 as follows:
 \begin{equation}
\begin{aligned}
\mathring{\boldsymbol{z}}_{\bold{X}_i}^{(l)} =  & \; \mathring{\pi}_{\bold{X}}^{(l)}\cdot \sigma\Biggl(\sum_{k\in\tilde{\bold{N}}_i} \alpha_{ik}^{\rm{in}}\cdot\boldsymbol{W}_{\bold{X}_{\rm{in}}}^{(l)}\boldsymbol{z}_{{\bold{X}_k}}^{(l-1)}\Biggr) + \\ &  \left(1-\mathring{\pi}_{\bold{X}}^{(l)}\right)\cdot \sigma\Biggl(\sum_{k\in\tilde{\bold{N}}_i} \alpha_{ik}^{\rm{st}}\cdot\boldsymbol{W}_{\bold{X}_{\rm{st}}}^{(l)}\boldsymbol{z}_{{\bold{X}_k}}^{(l-1)}\Biggr),   \\
\mathring{\boldsymbol{z}}_{\bold{T}_i}^{(l)} = & \; \mathring{\pi}_{\bold{T}}^{(l)}\cdot \sigma\Biggl(\sum_{k\in\tilde{\bold{N}}_i} \alpha_{ik}^{\rm{in}}\cdot\boldsymbol{W}_{\bold{T}_{\rm{in}}}^{(l)}\boldsymbol{z}_{\bold{T}_k}^{(l-1)}\Biggr)  +\\ & \left(1-\mathring{\pi}_{\bold{T}}^{(l)}\right)\cdot \sigma\Biggl(\sum_{k\in\tilde{\bold{N}}_i} \alpha_{ik}^{\rm{st}}\cdot\boldsymbol{W}_{\bold{T}_{\rm{st}}}^{(l)}\boldsymbol{z}_{\bold{T}_k}^{(l-1)}\Biggr).
\end{aligned}
 \end{equation}
Here $\mathring{\pi}_{\bullet}^{(l)}$ for $\bullet\in\{\bold{X},\bold{T}\}$ is a learnable one-dimensional parameter with the range of $[0,1]$, which is used in the summary function to adaptively integrate aggregated results with different partial importance.
To prevent the normalization operation in IPAtt and SPAtt from exacerbating the issue~(II) of DNE (as observed in Appendix~\ref{app:motivation:degree}),
we design a message amplifier $\eta_i$ to vary interference representations with the degrees of individuals, inspired by ~\citet{corso2020principal}:
\begin{equation}
\begin{aligned}
\boldsymbol{z}_{{\bold{X}_i}}^{(l)} = &\;(1 + \eta_i) \cdot \mathring{\boldsymbol{z}}_{\bold{X}_i}^{(l)}, \;
\boldsymbol{z}_{\bold{T}_i}^{(l)} = (1 + \eta_i) \cdot \mathring{\boldsymbol{z}}_{\bold{T}_i}^{(l)}, \\ &
\eta_i=\pi_{\eta}\cdot\left({\log(\tilde{d}_i)}/{\sum_{i=1}^{n_{\rm{tr}}}{\log(\tilde{d}_i)}}\right),\label{eq:MAP}
\end{aligned}
\end{equation}
where    
$\pi_{\eta}$ is a learnable parameter or hyperparameter, which adjusts the amplification level, 
$\tilde{d}_i$ represents the degree of $i$ including the self-loop, and $n_{\rm{tr}}$ represents the size of training set. 
Implementation and hyperparameter details are described in  Appendix~\ref{app:Implementation}.
By applying the NIM layer with the message amplifier to generate interference representation, the model can address the issue (II) of DNE, as stated in Proposition~\ref{sec:pro3} and proved in the Appendix~\ref{app:motivation:degree}.
\begin{proposition}\label{sec:pro3}
Interference representation generated by NIM layers after applying the message amplifier can address the issue (II) of DNE, even in local networks where all individuals have similar interference-related information.
\end{proposition}

\subsection{Representation balancing module}\label{sec:RB}
In this section, we aim to address issues of both confounding bias $p(t\mid \boldsymbol{x},\boldsymbol{x}_{\bold{G}}) \neq p(1-t\mid\boldsymbol{x},\boldsymbol{x}_{\bold{G}} )$
and interference bias $p(\boldsymbol{t}_{\bold{G}}\mid \boldsymbol{x},\boldsymbol{x}_{\bold{G}},t)\neq p(\boldsymbol{t}_{\bold{G}}\mid \boldsymbol{x},\boldsymbol{x}_{\bold{G}},1-t)$.  
Most ITE estimators for graph data mitigate these bias issues by a strategy of separately balancing representations $\boldsymbol{z}_i$, $\boldsymbol{z}_{\bold{X}_i}$, and $\boldsymbol{z}_{\bold{T}_i}$ between treated and control groups,
such as~\citet{pmlr-v130-ma21c}, \citet{ma2022learning}, \citet{jiang2022estimating}, and \citet{LIN2024SITE}. 
If we use their strategies for representation balancing, we need to achieve two goals:
$p(\boldsymbol{z},\boldsymbol{z}_{\bold{X}}\mid t=1) \approx p(\boldsymbol{z},\boldsymbol{z}_{\bold{X}}\mid t=0)$
and $p(\boldsymbol{z}_{\bold{T}}\mid \boldsymbol{z},\boldsymbol{z}_{\bold{X}},t=1)\approx p(\boldsymbol{z}_{\bold{T}}\mid \boldsymbol{z},\boldsymbol{z}_{\bold{X}},t=0)$
for mitigating confounding and interference biases, respectively.
This requires multiple hyperparameters to trade off each loss term for different representation balancing targets,
leading to costly hyperparameter selection.
Furthermore, 
there might exist bias caused by some unobserved variables~\cite{wang2023optimal}, 
which cannot be addressed by the strategy of separate balancing.  
To address these limitations, we propose a new representation balancing strategy that
can jointly balance representations $\boldsymbol{z}_i$, $\boldsymbol{z}_{\bold{X}_i}$, and $\boldsymbol{z}_{\bold{T}_i}$ between treated and control groups
while mitigating the bias issue caused by unobserved variables through a proximal factual outcome regularizer (PFOR)~\cite{courty2017joint,wang2023optimal}. 
Our key insight stems from the observation $p(\boldsymbol{z},\boldsymbol{z}_{\bold{X}}\mid t)p(\boldsymbol{z}_{\bold{T}}\mid \boldsymbol{z},\boldsymbol{z}_{\bold{X}},t)=p(\boldsymbol{z},\boldsymbol{z}_{\bold{X}},\boldsymbol{z}_{\bold{T}}\mid t)$.
This implies that we can balance $\boldsymbol{z}_i$, $\boldsymbol{z}_{\bold{X}_i}$, and $\boldsymbol{z}_{\bold{T}_i}$ between treated and control groups by balancing their joint distribution.
In this case, the model can adaptively prioritizes which parts of representations $\boldsymbol{z}_i$, $\boldsymbol{z}_{\bold{X}_i}$, and $\boldsymbol{z}_{\bold{T}_i}$ need to be balanced at each training iteration. 
Importantly, the proposed strategy has the potential to mitigate bias caused by some unobserved variables through a Wasserstein discrepancy~\cite{pmlr-v70-shalit17a} (see Definition~\ref{def:wass}) with a PFOR.

Specifically, let $\mathcal{W}(p_1,p_2)$ denote the Wasserstein discrepancy between two distributions $p_1$ and $p_2$, and $\Phi$ denote a map function that is achieved by combining the MLP and NIM layers (see Section~\ref{sec:RL}).
Let $\boldsymbol{u}_i=(\boldsymbol{x}_i,\boldsymbol{x}_{{\bold{G}_i}},\boldsymbol{t}_{{\bold{G}_i}})$ and 
$\boldsymbol{r}_i=\Phi(\boldsymbol{u}_i)=(\boldsymbol{z}_i,\boldsymbol{z}_{{\bold{X}_i}},\boldsymbol{z}_{\bold{T}_i})$ for simplicity. 

We minimize Wasserstein discrepancy with PFOR~\cite{courty2017joint, wang2023optimal} to balance representations $\boldsymbol{r}_i$ between different treatment groups. The definition for Wasserstein discrepancy is detailed in  Definition~\ref{def:wass}. 
As the scales of values of representations $\boldsymbol{z}$, $\boldsymbol{z}_{\bold{X}}$, and $\boldsymbol{z}_{\bold{T}}$ may be uneven, we propose a proxy module that contains a normalization layer followed by a projection function without nonlinear activation function.
Then, we balance representations $\boldsymbol{z}$, $\boldsymbol{z}_{\bold{X}}$, and $\boldsymbol{z}_{\bold{T}}$ jointly by balancing the output of the proxy module. 
There might be unobserved variables that introduce an additional bias issue~\cite{wang2023optimal}.
By applying the proposed joint balancing strategy,
we can use the PFOR~\cite{courty2017joint, wang2023optimal} to mitigate the bias from unobserved variables. 
Let $\boldsymbol{r}'_i$ be the output of the proxy module. 
 In this case, if we calculate unit-wise distance $D_{ij}=\Vert \boldsymbol{r}'_i-\boldsymbol{r}'_j\Vert^2$ for Wasserstein discrepancy, we can only mitigate the bias introduced by observed data.  Let $\boldsymbol{v}$ be the unobserved variables.
 To take $\boldsymbol{v}$ into account, the individual-wise distance needs to be modified as $D_{ij}=\Vert \boldsymbol{r}'_i-\boldsymbol{r}'_j\Vert^2 + \Vert \boldsymbol{v}_i-\boldsymbol{v}_j\Vert^2$. However, we do not have information about $\boldsymbol{v}$. 
 Inspired by \citet{wang2023optimal}, who designed a PFOR for the no-interference setting, we design a PFOR for scenarios involving  networked interference.
 Specifically, when we have balanced $\boldsymbol{r}$  (or $\boldsymbol{r}'$), and identical $t$, the only variable reflecting the variation of $\boldsymbol{v}$ is the outcome. Therefore, we can use information of outcomes to 
  replace $\boldsymbol{v}$ in the modified individual-wise distance by the  PFOR~\cite{courty2017joint,wang2023optimal}. Then, the individual-wise distance can be modified as:
  \begin{equation}
 \begin{aligned}
     D_{ij}^{\lambda_{D}} = \Vert \boldsymbol{r}'_i-\boldsymbol{r}'_j\Vert^2+
     \lambda_{D}\cdot&\biggl(\Vert y_i(1,\boldsymbol{t}_{\bold{G}_i})-y_j(1,\boldsymbol{t}_{\bold{G}_i})\Vert^2+\\&\;\;\Vert y_i(0,\boldsymbol{t}_{\bold{G}_i})-y_j(0,\boldsymbol{t}_{\bold{G}_i})\Vert^2\biggr), 
\end{aligned}
  \end{equation}
 where $\lambda_{D}$ is a hyperparameter. 
As only a part of potential outcomes is observed, we use factual and predicted outcomes to replace potential outcomes for $D_{ij}^{\lambda_{D}}$, then we have:
\begin{align}
     D_{ij}^{\lambda_{D}} = \Vert \boldsymbol{r}'_i-\boldsymbol{r}'_j\Vert^2+
     \lambda_{D}\cdot\biggl(\Vert y_i-\hat{y}_j\Vert^2+\Vert \hat{y}_i-y_j\Vert^2\biggr), 
 \end{align}
 where $i$ and $j$ represent two individuals in different treatment groups, and $\hat{y}_i$ represents the predicted outcome.
 \begin{remark}
Importantly, it cannot address bias introduced by unobserved variables through PFOR if $\boldsymbol{z}$, $\boldsymbol{z}_{\bold{X}}$,  and $\boldsymbol{z}_{\bold{T}}$ are balanced separately,  since  
 the outcome is not the solo variable that reflects the variation of unobserved variables $\boldsymbol{v}$ with the strategy of separate balancing.
 \end{remark}

\noindent\textbf{Variant.}
We propose a variant named GITE$_{\rm{v}}$ that uses an MLP with a nonlinear activation function for the proxy module, since the expressive ability of projection function without nonlinear activation function is often limited.
Let $\boldsymbol{r}_i'$  
be the output of the proxy module.
We add a term into loss: $\mathcal{L}_P = \Vert\boldsymbol{r}_i-\boldsymbol{r}_i' \Vert^2$ for GITE$_{\rm{v}}$ to ensure that $\boldsymbol{r}_i$ is close to $\boldsymbol{r}_i'$ when balancing the representations. 

\begin{table*}
\centering
\caption{Results (mean and standard errors) on the test sets. Results are averaged over ten executions.   
Results with \textbf{boldface} represent the lowest mean error, whereas results with \underline{underlines} represent the second and third lowest mean error. Here, the AMZ-N dataset is a sparse graph as its size of nodes $\approx$ that of edges, whereas the Flickr and Blog datasets are far more dense than the AMZ-N dataset.
}
\vspace{-0.17cm}
\begin{tabular}{l|ccccccccccc}
\hline  
 &  \multicolumn{2}{c}{AMZ-N}       & \multicolumn{2}{c}{Flickr}& \multicolumn{2}{c}{Blog}\\ 
{Method} & $\sqrt{\epsilon_{\textrm{MSE}}}$ & $\sqrt{\epsilon_{\textrm{PEHE}}}$    & 
$\sqrt{\epsilon_{\textrm{MSE}}}$ & $\sqrt{\epsilon_{\textrm{PEHE}}}$    & 
$\sqrt{\epsilon_{\textrm{MSE}}}$ & $\sqrt{\epsilon_{\textrm{PEHE}}}$    & 
\\ [0.05cm]\hline
\hline
TARNet~\cite{pmlr-v70-shalit17a}            & 0.97 ± 0.02                & 1.43 ± 0.01 
& 4.86 ± 0.12 & 4.44 ± 0.12 
& 25.18 ± 1.10 & 17.82 ± 1.11 
\\
BNN~\cite{Johansson2016}   & 1.00 ± 0.00                & 1.45 ± 0.00 
& 6.11 ± 0.00 & 4.57 ± 0.00 
& 
38.57 ± 0.00& 18.85 ± 0.00 
\\
CFR-MMD~\cite{pmlr-v70-shalit17a} & 0.95 ± 0.01   & 1.41 ± 0.01 
& 
4.92 ± 0.29 & 4.47 ± 0.13 
& 
26.16 ± 1.68& 17.67 ± 0.85
\\
CFR-Wass~\cite{pmlr-v70-shalit17a}  & 0.95 ± 0.01   & 1.40 ± 0.01  
& 4.87 ± 0.28& 4.44 ± 0.20 & 
25.75 ± 0.61& 17.96 ± 1.12
\\
ESCFR~\cite{wang2023optimal}  & 0.97 ± 0.01 & 1.42 ± 0.01  
& 5.34 ± 0.20& 4.38 ± 0.09
&
24.88 ± 1.68&18.41 ± 1.09
\\
RERUM~\cite{he2024rankability}  & 0.94 ± 0.00 & 1.41 ± 0.00  
& 4.81 ± 0.10& 4.44 ± 0.06
&
26.22 ± 1.17&17.58 ± 0.35
\\
NetDeconf~\cite{guo2020learning}  & 1.00 ± 0.00   & 1.46 ± 0.00  
& 6.10 ± 0.00& 4.60 ± 0.00
&33.71 ± 0.12&18.30 ± 0.69
\\
NetEst~\cite{jiang2022estimating}    &0.96 ± 0.02 &  1.47 ± 0.01  
& 7.53 ± 0.13& 4.75 ± 0.12& 
35.49 ± 0.27& 18.55 ± 0.19 
 \\
SPNet~\cite{huang2023modeling} & 1.09 ± 0.10  & 1.51 ± 0.07 
&
6.12 ± 0.01& 4.60 ± 0.03& 
32.12 ± 1.88& 22.42 ± 2.65
\\
DWR~\cite{zhao2024learning} & 0.87 ± 0.01  & 1.35 ± 0.01 
&
4.50 ± 0.55 & 4.43 ± 0.37& 
23.51 ± 0.55& \underline{16.82 ± 0.75}
\\
CauGramer~\cite{wucausal} & 1.00 ± 0.01  & 1.46 ± 0.01 
&
4.14 ± 0.16 & 4.43 ± 0.06& 
\underline{17.05 ± 2.57}& 23.86 ± 1.26
\\
GCN-HSIC~\cite{pmlr-v130-ma21c} & 0.82 ± 0.02 & 1.34 ± 0.02& 
4.86 ± 0.37 & 4.55 ± 0.24& 
25.07 ± 1.35& 16.90 ± 0.64
\\
SAGE-HSIC~\cite{pmlr-v130-ma21c} & 0.83 ± 0.02   & 1.34 ± 0.02 &
4.07 ± 0.11& 4.24 ± 0.12& 
22.08 ± 1.20& 18.78 ± 1.99
\\
 SITE~\cite{LIN2024SITE} 
 &\underline{0.77 ± 0.01}  &  1.34 ± 0.06 
 & 4.14 ± 0.27& 4.22 ± 0.21& 
19.24 ± 1.05& 17.87 ± 1.37
\\
IDENet~\cite{adhikari2025inferring} 
 &1.00 ± 0.00  &  1.46 ± 0.00
 & 3.82 ± 0.11& 4.13 ± 0.06& 
25.17 ± 3.08& 19.97 ± 2.09
\\
HyperSCI~\cite{ma2022learning} & 0.81 ± 0.02   & \underline{1.28 ± 0.02} & 
3.84 ± 0.17& 4.03 ± 0.23& 
18.32 ± 1.49& 17.90 ± 1.68
\\
HINITE~\cite{lin2023estimating}& 0.79 ± 0.02   & 1.30 ± 0.05 
&
\underline{3.81 ± 0.14}& \underline{3.87 ± 0.10}& 
17.17 ± 1.39& 17.30 ± 2.60
\\ \hline
\hline 
GITE (ours) & \textbf{0.75 ± 0.01}   & \textbf{1.21 ± 0.01} 
&\textbf{3.41 ± 0.17}& \textbf{3.46 ± 0.11}& 
\underline{14.36 ± 0.82}& \textbf{13.69 ± 1.46}
\\
GITE$_{\rm{v}}$ (ours) & \textbf{0.75 ± 0.01}   & \textbf{1.21 ± 0.01}  
&
\underline{3.44 ± 0.11}&\underline{3.47 ± 0.10}& 
\textbf{13.87 ± 0.98}& \underline{14.13 ± 1.88}
\\\hline
\end{tabular}
\label{table1:main}
\end{table*}

\subsection{Outcome prediction module}
Given the covariate representation $\boldsymbol{z}_i$, interference representations  $\boldsymbol{z}_{\bold{X}_i}$ and $\boldsymbol{z}_{\bold{T}_i}$, and treatment assignment $t_i$,
we train two predictors to infer the outcomes with different values of $t$. 

Specifically, let $f_{0}$ and $f_{1}$  denote the predictor for potential outcome with $t=0$ and $t=1$, respectively. Each predictor is achieved by an MLP.  
Let $\Theta$ be all learnable parameters of GITE.  We add L2 regularization into our loss function to avoid model overfitting.
The loss function $\mathcal{L}_{\rm{total}}$ of the proposed GITE consists of mean square error (MSE) between predicted and factual outcomes, Wasserstein discrepancy of representations between different treatment groups, and L2 regularization (denoted as $\Vert\Theta\Vert^2)$.  Each term in $\mathcal{L}_{\rm{total}}$ is traded off by hyperparameters $\beta$ and $\lambda$, as follows: 
\begin{align}
\mathcal{L}_{\rm{total}} = &\frac{1}{n_{\rm{tr}}}\sum_{i=1}^{n_{\rm{tr}}} \left(f_{t}(\boldsymbol{z}_i,\boldsymbol{z}_{{\bold{X}_i}},\boldsymbol{z}_{\bold{T}_i}) - y_{i}\right)^2 + \beta \cdot \mathcal{W} + \lambda \cdot \Vert \Theta \Vert^2.\label{eq:maintotalloss}
\end{align}
The parameters of GITE are optimized by minimizing $\mathcal{L}_{\rm{total}}$.
For GITE$_{\rm{v}}$, we add $\lambda_P \cdot \mathcal{L}_{P}$ into the Equation~(\ref{eq:maintotalloss}), where $\lambda_P$  is a hyperparameter.   
After training,  ITE can be estimated using the trained predictors and generated representations, as follows: \begin{equation}
    \hat{\tau}_i=f_{1}(\boldsymbol{z}_i,\boldsymbol{z}_{{\bold{X}_i}},\boldsymbol{z}_{\bold{T}_i})-f_{0}(\boldsymbol{z}_i,\boldsymbol{z}_{{\bold{X}_i}},\boldsymbol{z}_{\bold{T}_i}).
\end{equation}

\vspace{1em}
\noindent\textbf{Error bounds.}
To estimate ITE from observational data, the training of the model relies on factual outcomes instead of true ITE due to the absence of counterfactual outcomes. However,  confounding and interference biases may cause the performance of factual outcomes to poorly reflect that of ITE estimation, which raises concerns about the ability of the trained model to guide decision-making. To address this, \citet{pmlr-v70-shalit17a} and \citet{cai2023generalization} analyze the error bound for ITE estimation by building a bridge between errors in ITE estimation and factual outcome prediction. However, the former applies only to non-graph data, and the latter assumes neighbor interference rather than networked interference. Therefore, we analyze the error bound of ITE with networked interference based on the proposed representation balancing strategy, as detailed in Appendix~\ref{app:bound}.

\section{Experiments}\label{sec:all_exps}

In this section, we conducted experiments to answer the following research questions (RQs).
\vspace{-0pt}

\begin{itemize}
    \item \textbf{RQ~1}: do the proposed methods outperform baseline methods in ITE estimation with confounders and networked interference?
    \item \textbf{RQ~2}: are proposed components important to the proposed methods?
    \item \textbf{RQ~3}:  how sensitive are the proposed methods to their hyperparameters?
\end{itemize}

\begin{table*}
\centering
\caption{Results (mean and standard errors) of ablation experiments. 
 Results are averaged over ten executions. 
}
\vspace{0.0cm}
\begin{tabular}{l|ccccccccccc}
\hline
  & \multicolumn{2}{c}{AMZ-N}       
  & \multicolumn{2}{c}{Flickr}& \multicolumn{2}{c}{Blog}\\
Method & $\sqrt{\epsilon_{\textrm{MSE}}}$ & $\sqrt{\epsilon_{\textrm{PEHE}}}$    & 
$\sqrt{\epsilon_{\textrm{MSE}}}$ & $\sqrt{\epsilon_{\textrm{PEHE}}}$    & 
$\sqrt{\epsilon_{\textrm{MSE}}}$ & $\sqrt{\epsilon_{\textrm{PEHE}}}$     
\\ \hline \hline
GITE$_{\rm{NR}}$           & 0.78 ± 0.01          & 1.27 ± 0.03 
&
\underline{3.45 ± 0.11} & 3.56 ± 0.07 
& 14.60 ± 1.33& 14.91 ± 2.62 
\\
GITE$_{\rm{NB}}$   & \underline{0.76 ± 0.01}                & \underline{1.22 ± 0.01} 
&
3.47 ± 0.12& \underline{3.50 ± 0.14} 
& 14.80 ± 1.04 & 15.24 ± 3.36
\\
GITE$_{\rm{NS}}$    & \underline{0.76 ± 0.01}   & \underline{1.22 ± 0.01}  
&
3.49 ± 0.06 & 3.63 ± 0.09 
& 14.80 ± 1.31& 14.79 ± 2.19 
\\
GITE$_{\rm{NM}}$    & \underline{0.76 ± 0.01}   & \underline{1.22 ± 0.01}  
&
3.47 ± 0.21 & 3.61 ± 0.06 
& 14.54 ± 1.15& \underline{14.56 ± 1.24} 
\\
GITE$_{\rm{NATT}}$    & 0.78 ± 0.03   & 1.24 ± 0.02  
&
6.11 ± 0.00 & 4.57 ± 0.03 
& 26.98 ± 9.28& 20.24 ± 3.55 
\\
 GITE$_{\rm{NA}}$       &0.77 ± 0.01  &  1.23 ± 0.01 
 & 6.96 ± 1.79 & 5.63 ± 1.86 
 & 22.61 ± 6.06 & 23.23 ± 5.94 
 \\
 GITE$_{\rm{NP}}$ & \underline{0.76 ± 0.01}  & \underline{1.22 ± 0.01} 
&3.46 ± 0.08 & 3.52 ± 0.08 
& \underline{14.47 ± 0.97} & 14.58 ± 1.31 
\\
GITE$_{\rm{BS}}$ & 0.77 ± 0.01  & \underline{1.22 ± 0.01} 
&3.52 ± 0.10 & 3.53 ± 0.17 
& 15.24 ± 1.65& 14.65 ± 1.46 
\\
\hline \hline
GITE & \textbf{0.75 ± 0.01}   & \textbf{1.21 ± 0.01} 
&
\textbf{3.41 ± 0.17}& \textbf{3.46 ± 0.11}
& \underline{14.36 ± 0.82}& \textbf{13.69 ± 1.46}
\\
GITE$_{\rm{v}}$  & \textbf{0.75 ± 0.01}   & \textbf{1.21 ± 0.01} 
&
\underline{3.44 ± 0.11}& \underline{3.47 ± 0.10}& 
\textbf{13.87 ± 0.98}& \underline{14.13 ± 1.88}
\\\hline
\end{tabular}
\label{table4:ablation}
\end{table*}

\subsection{Experimental setting}\label{sec4.1:expset}
\textbf{Datasets.}
We conducted experiments on three public datasets widely used in the task of ITE estimation with interference:
Amazon negative (abbreviated as AMZ-N) dataset~\cite{he2016ups}, Flickr dataset~\cite{wang2013learning}, and BlogCatalog dataset (abbreviated as Blog)~\cite{li2015unsupervised,li2019adaptive}.
The AMZ-N dataset contains $14,538$ items with $15,011$ edges. For the AMZ-N dataset, we used the covariates, treatments, outcomes, and ITE, all of which were provided by \citet{rakesh2018linked}.
The Flickr dataset contains $7,575$ users with $479,476$ directed edges.
We used the $1,206$-dimensional embeddings of user profiles that were provided by \citet{guo2020learning}. 
The Blog dataset contains $5,196$ units with $343,486$ directed edges.
We used the $2,198$-dimensional embeddings of user profiles that were provided by \citet{guo2020learning}.
The details of each dataset are described in Appendix~\ref{app:datasets}.
The ground truth of ITE is hard to collect due to the lack of ground truth regarding counterfactual outcomes.
Following existing works~\cite{cai2023generalization,jiang2022estimating,pmlr-v130-ma21c}, we transformed the covariates and graph structures to simulate treatments and outcomes with confounders and networked interference for the Flickr and Blog datasets, as detailed in Appendix~\ref{App:simulation}.

\vspace{1em}
\noindent\textbf{Baselines.}
Baseline methods can be  divided into the following four categories.
(I) ITE estimators for non-graph data: BNN~\cite{Johansson2016}, CFR-MMD~\cite{pmlr-v70-shalit17a},  CFR-Wass~\cite{pmlr-v70-shalit17a}, 
     TARNet~\cite{pmlr-v70-shalit17a},  ESCFR~\cite{wang2023optimal}, and RERUM~\cite{he2024rankability}. 
(II) ITE estimator for graph data without addressing interference: NetDeconf~\cite{guo2020learning}.
(III) ITE estimators for graph data with addressing neighbor interference only: NetEst~\cite{jiang2022estimating},  
     SPNet~\cite{huang2023modeling},  DWR~\cite{zhao2024learning}, and CauGramer~\cite{wucausal}.
(IV) ITE estimators for graph data with addressing networked interference: GCN-HSIC~\cite{pmlr-v130-ma21c}, SAGE-HSIC~\cite{pmlr-v130-ma21c},  SITE~\cite{LIN2024SITE}, IDENet~\cite{adhikari2025inferring} model networked interference by using GCN or mean aggregation function;
HyperSCI~\cite{ma2022learning} and HINITE~\cite{lin2023estimating} model networked interference by using GAT.
We describe the details of each baseline in Appendix~\ref{app:detailsbaselines}.

\vspace{1em}
\noindent\textbf{Metrics.} Following \citet{pmlr-v130-ma21c} and \citet{LIN2024SITE}, we consider two widely used metrics $\sqrt{\epsilon_{\rm{MSE}}}$ and $\sqrt{\epsilon_{\rm{PEHE}}}$ for all datasets.
$
\sqrt{\epsilon_{\mathrm{MSE}}}$ quantifies the performance in outcome prediction, while $\sqrt{\epsilon_{\text{PEHE}}}
$ quantifies the performance in ITE estimation. 
They are defined as follows: 
\begin{equation}
\begin{aligned}
\sqrt{\epsilon_{\mathrm{MSE}}}=\sqrt{\frac{1}{n_{\rm{te}}}\sum^{n_{\rm{te}}}_{i=1}(\hat{y}_i-y_{i})^2}, \; \sqrt{\epsilon_{\text{PEHE}}}=\sqrt{\frac{1}{n_{\rm{te}}}\sum^{n_{\rm{te}}}_{i=1} \left( \hat{\tau}_i-\tau_i\right)^2},\nonumber
\end{aligned}
\end{equation}
where $n_{\rm{te}}$ is the size of the test set.
We randomly partitioned all datasets into training/validation/test splits with a ratio of $70\%/15\%/15\%$ and averaged results over ten repeated executions.  We defer implementation and hyperparameter details to Appendix~\ref{app:Implementation}.

\begin{figure*}
	\centering
 \subfloat[ AMZ-N,   $\beta$,  $\sqrt{\epsilon_{\textrm{MSE}}}$.]{\includegraphics[width=.24\textwidth]{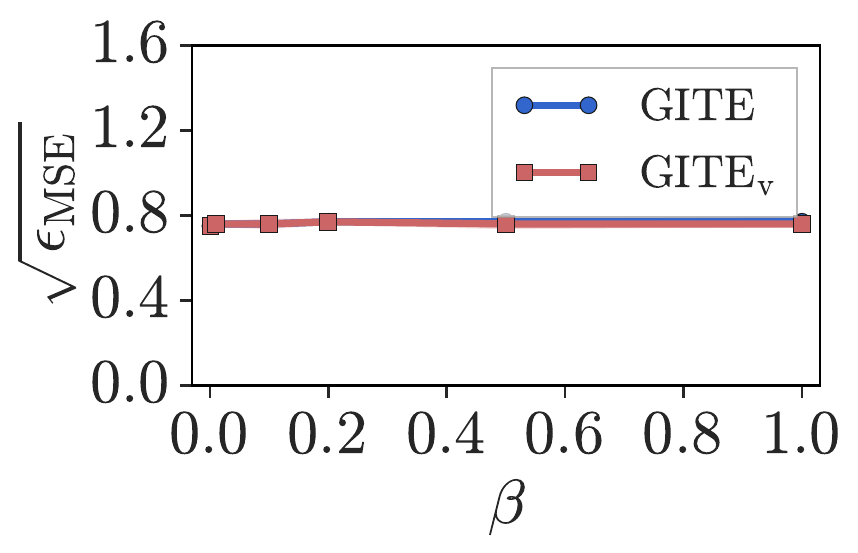}} \hspace{1pt}
	\subfloat[ AMZ-N, $\beta$, $\sqrt{\epsilon_{\textrm{PEHE}}}$.]
 {\includegraphics[width=.24\textwidth]{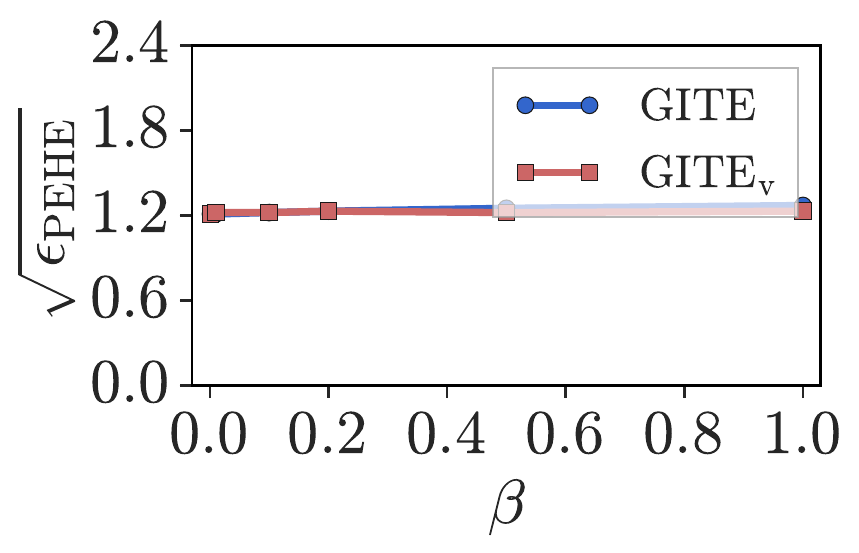}} \hspace{1pt}
 \subfloat[AMZ-N,  $\lambda$,  $\sqrt{\epsilon_{\textrm{MSE}}}$.]{\includegraphics[width=.24\textwidth]{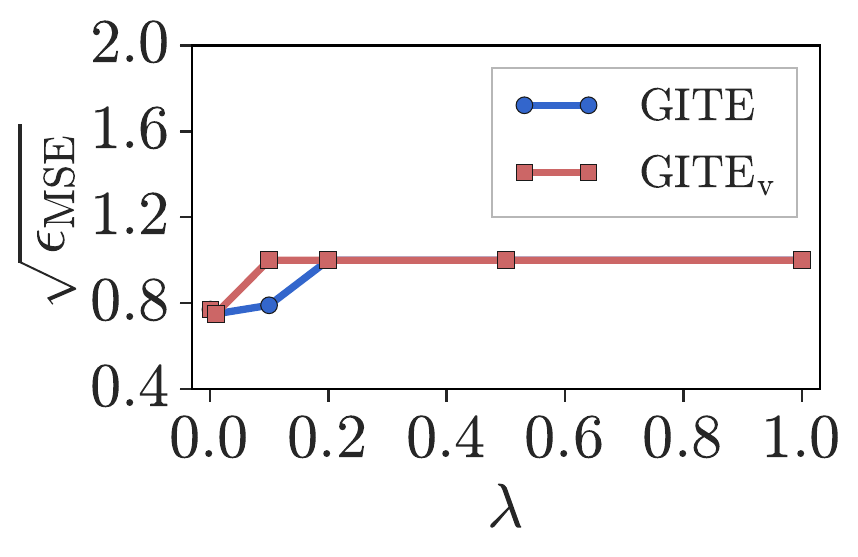}} \hspace{1pt}
	\subfloat[AMZ-N, $\lambda$,  $\sqrt{\epsilon_{\textrm{PEHE}}}$.]
 {\includegraphics[width=.24\textwidth]{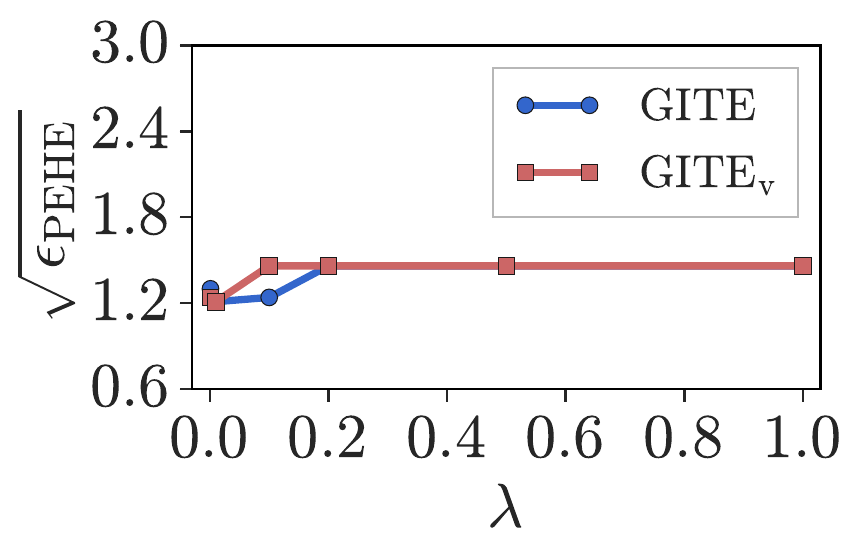}}\\
\hspace{-0pt}	\subfloat[ Flickr, $\beta$, $\sqrt{\epsilon_{\textrm{MSE}}}$.]{\includegraphics[width=.242\textwidth]{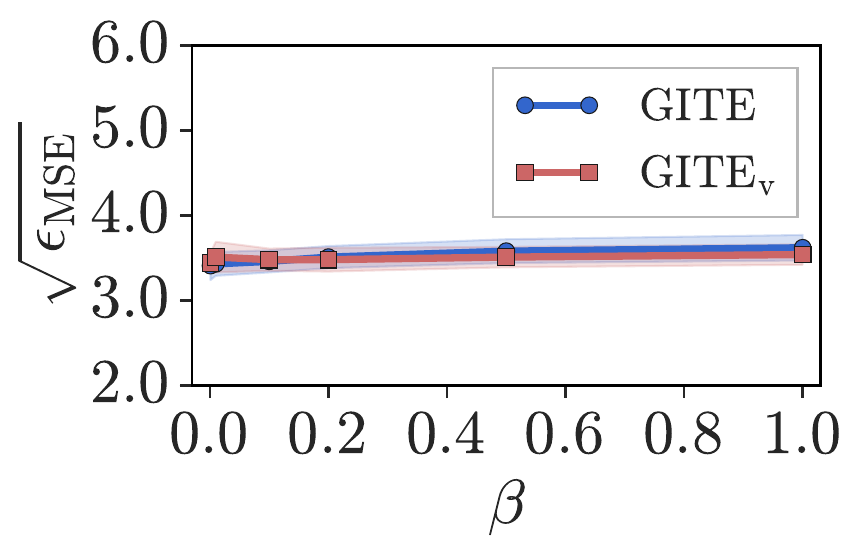}} \hspace{1pt}
	\subfloat[ Flickr, $\beta$, $\sqrt{\epsilon_{\textrm{PEHE}}}$.]
 {\includegraphics[width=.242\textwidth]{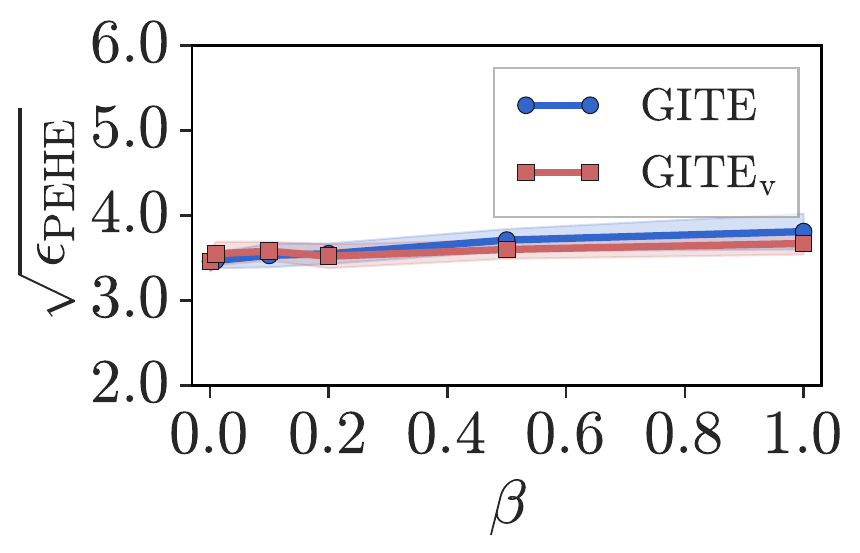}} \hspace{1pt}
	\subfloat[  Flickr, $\lambda$, $\sqrt{\epsilon_{\textrm{MSE}}}$.]{\includegraphics[width=.242\textwidth]{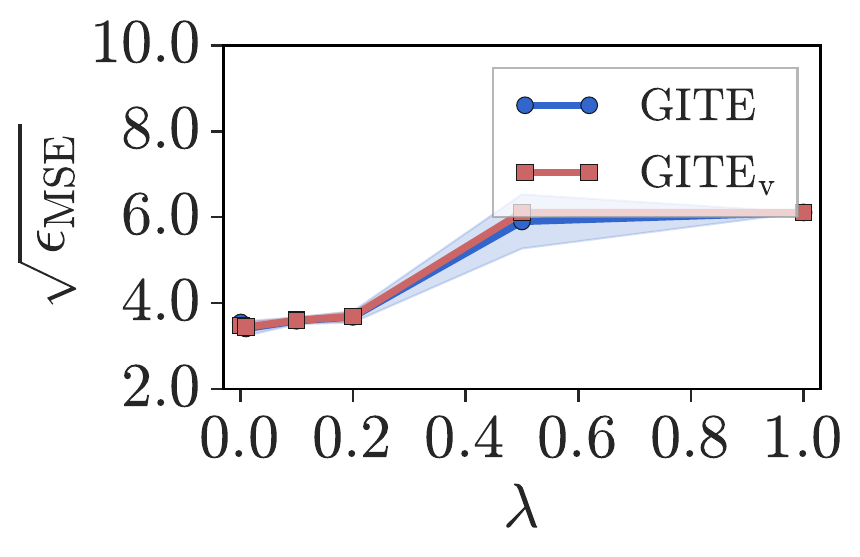}} \hspace{1pt} \subfloat[ Flickr, $\lambda$, $\sqrt{\epsilon_{\textrm{PEHE}}}$.]{\includegraphics[width=.242\textwidth]{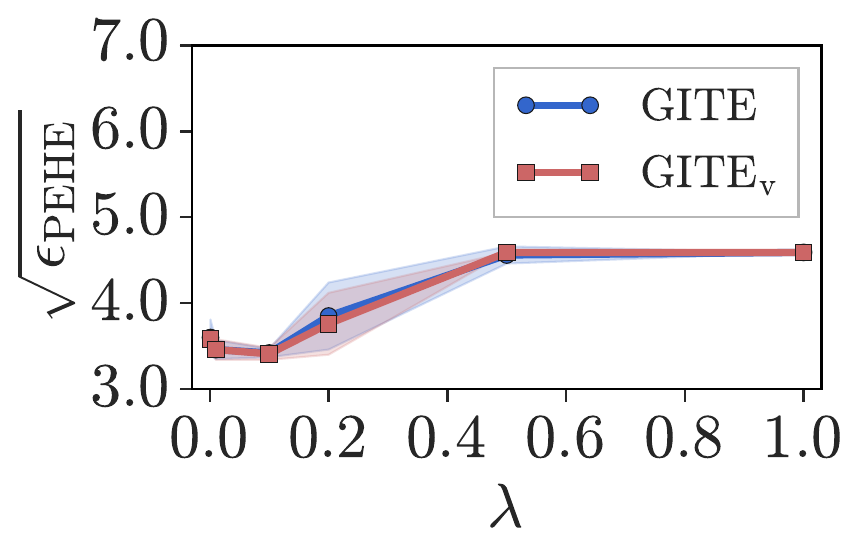}} \\
 \hspace{-2pt}\subfloat[Blog, $\beta$, $\sqrt{\epsilon_{\textrm{MSE}}}$.]{\includegraphics[width=.25\textwidth]{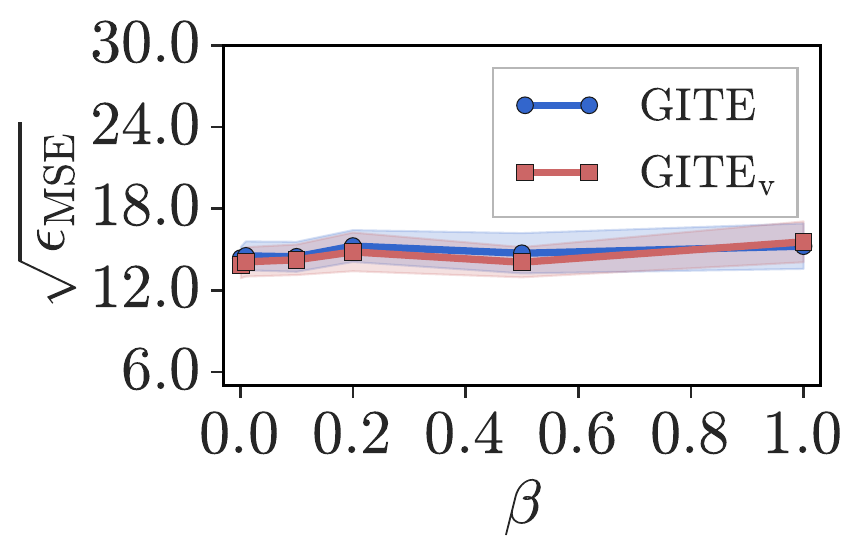}} 
	\subfloat[Blog,  $\beta$, $\sqrt{\epsilon_{\textrm{PEHE}}}$.]
 {\includegraphics[width=.25\textwidth]{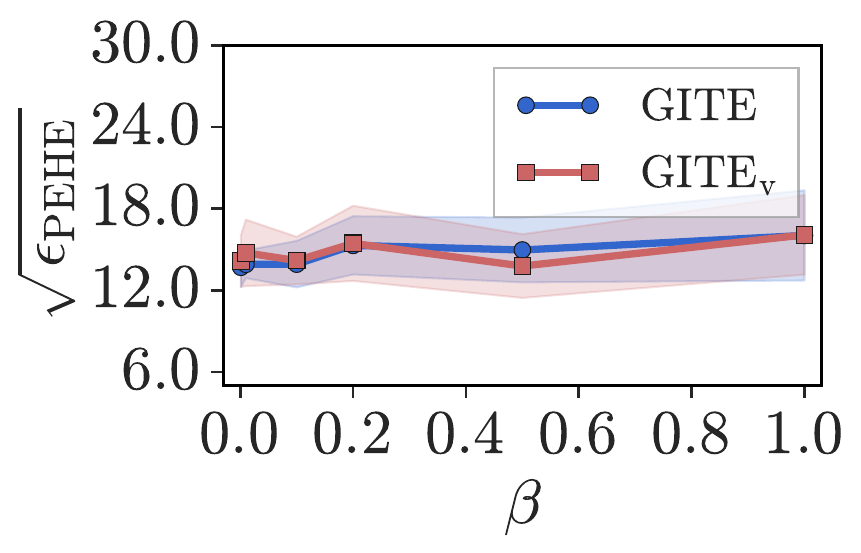}}
 \subfloat[Blog, $\lambda$, $\sqrt{\epsilon_{\textrm{MSE}}}$.]{\includegraphics[width=.25\textwidth]{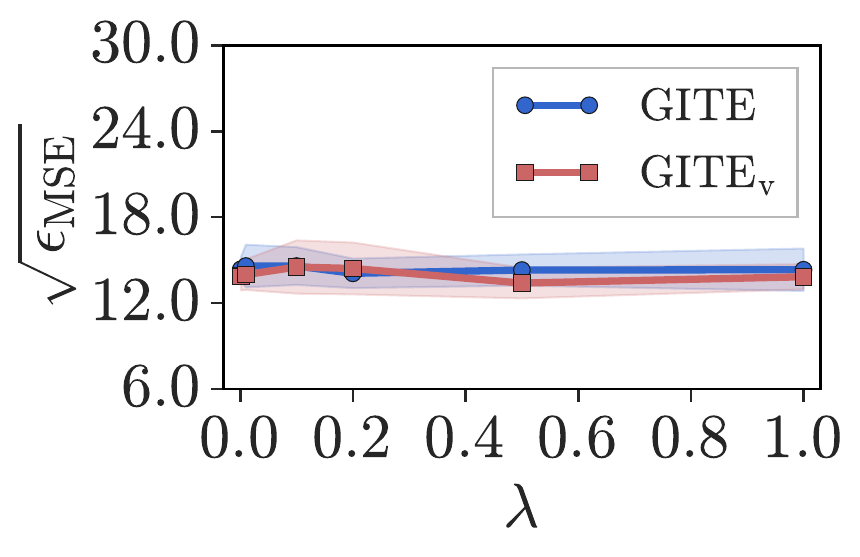}} \hspace{-3pt}
	\subfloat[Blog, $\lambda$, $\sqrt{\epsilon_{\textrm{PEHE}}}$.] 
{\includegraphics[width=.25\textwidth]{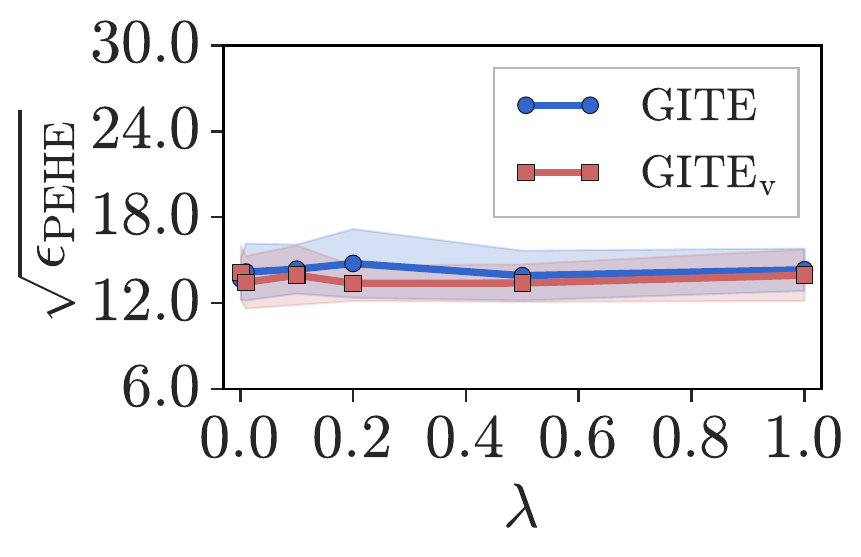}} 

\caption{Results (mean and standard errors) of sensitivity experiments for hyperparameters $\beta$ and $\lambda$. 
Results are averaged over ten executions.
}
\label{fig:expsen}
\end{figure*}

\subsection{Performance evaluation experiments (RQ~1)} As shown in Table~\ref{table1:main}, we conducted experiments to answer RQ~1.
Overall, GITE and GITE$_{\rm{v}}$ outperform all baseline methods, which shows their effectiveness. 
Specifically, 
GAT-based methods (e.g., HyperSCI, HINITE) generally outperform GCN/GNN-based (GCN-HSIC, SAGE-HSIC, and SITE) methods in ITE estimation under DNE, as they partially account for differing importance of neighbors, but still ignore variations in neighbor size. GITE outperforms GAT-based methods by using two partial attention mechanisms and a message amplifier to properly model DNE.
Furthermore, we can observe that the improvements achieved by the proposed methods in the performance of outcome prediction and ITE estimation on the Flickr and Blog datasets are more significant than those on the AMZ-N dataset.
We consider this is because the Flickr and Blog datasets contain far more edges than AMZ-N, which introduces more severe DNE,
whereas 
the AMZ-N is a sparse graph, which contains $14,538$ nodes with only $15,011$ directed edges. 
On the Flickr and Blog datasets, the baseline methods may generate inappropriate interference representations due to being unable to properly capture DNE,
which degrades the performance in both outcome prediction and ITE estimation.
This reveals that properly capturing DNE is important to ITE estimation from graph data.
On the AMZ-N dataset, GCN, mean aggregation, and GAT performed well in generating interference representations,
as the network is relatively sparse, and they suffer less from DNE.
It is noteworthy that even on such a sparse graph, GITE achieves approximately $2.6\%$ improvement in outcome prediction performance and $5.5\%$ improvement in ITE estimation performance.

\subsection{Ablation experiments (RQ~2)} 
 To answer RQ~2, we conducted ablation experiments.
We first introduce several variants of the original GITE with ablation.
GITE$_{\rm{NR}}$ removes the L2 regularization by setting $\lambda=0$.
GITE$_{\rm{NB}}$ removes representation balancing module by setting $\beta=0$.
GITE$_{\rm{NS}}$ removes the message amplifier and importance estimated by SPAtt.
GITE$_{\rm{NM}}$ removes the message amplifier module.
GITE$_{\rm{NATT}}$ removes attention but keeps the message amplifier.
GITE$_{\rm{NA}}$ removes both importance estimated by SPAtt and IPAtt, i.e., using the sum operation for networked interference modeling.
GITE$_{\rm{NP}}$ removes the proxy module but still jointly balances representations.
GITE$_{\rm{BS}}$ balances representations separately.

Results of ablation experiments are shown in Table~\ref{table4:ablation}.
Overall, the results show that each component is important to the proposed methods.
In particular, we observe that the performance declines significantly on the Flickr and Blog datasets when 
removing both partial importance estimated by SPAtt and IPAtt. 
This is because GITE$_{\rm{NA}}$ uses the sum operation for aggregation,
which leads to an issue of numerical explosion for individuals with many neighbors and connections in graphs.
This makes it difficult to train the model.
We can observe that the issue is not serious on the AMZ-N dataset, which is very sparse.   

\subsection{Sensitivity experiments (RQ~3)} 
To answer RQ~3, we tested GITE and GITE$_{\rm{v}}$
with different values of $\beta$ and $\lambda$ in the range $\{0.001,0.01,0.1,0.2,0.5,1.0\}$. 
Results are shown in Figure~\ref{fig:expsen}. 
We observe that no significant changes in performance with different values of $\beta$ on the AMZ-N and Blog datasets.
However, we can also observe that there is a performance degradation when $\beta>0.2$ on the Flickr dataset, thus we recommend searching the value of $\beta$ in the range $(0, 0.2]$.
Furthermore, we observe that setting a large value of $\lambda$ ($> 0.2$) can result in significant performance degradation on the AMZ-N and Flickr datasets,
as models cannot update their weights with a large value of $\lambda$.
Thus, we recommend searching the value of $\lambda$ in the range $(0, 0.2]$.
More sensitivity experiments are detailed in Appendix~\ref{app:expsen}.

\subsection{Additional experiments} Additional experiments that address further RQs~4, 5, and 6 are detailed in the Appendices \ref{app:trainingtime}, \ref{appexp:diffatt},  and \ref{appexp:diffMA}. 


\section{Conclusion}
In this study, we study an important issue: DNE, which remains a challenge for previous approaches. We proposed a novel method to address this issue and conducted experiments to demonstrate the effectiveness of our method in ITE estimation from graph data with DNE, which reveals the importance of capturing DNE. We also introduce a representation balancing strategy,  while theoretically analyzing error bound based on this strategy in Appendix~\ref{app:bound}.
 Future research directions are discussed in Appendix~\ref{app:futurework}.

\section*{Acknowledgments}
This work was supported by JSPS KAKENHI (Grant-in-Aid for Scientific Research) B: Grant Number 26K02984, and supported by JST SPRING, Grant Number JPMJSP2110.


\bibliographystyle{ACM-Reference-Format}
\bibliography{sample-base}

\appendix

\theoremstyle{plain}
\clearpage
\section{Notation table}\label{app:notationtable}
We provide a notation table in Table~\ref{tab:notation}.
\begin{table*}
\centering
\caption{Notation Table.}
\begin{tabular}{lc}
\hline 
Notation & Definition  \\
\hline 
\hline 
$N$ / $n_{\text{tr}}$ / $n_{\text{te}}$& \small The number of all / training / test individuals  \\
$\boldsymbol{x}_i$ / $\boldsymbol{X}$& Covariates of the  individual $i$ / all individuals \\
$t$ / $\boldsymbol{T}$& \small Treatments of the  individual $i$ / all individuals  \\
$y_i$ / $\boldsymbol{Y}$& \small Factual outcomes of the  individual $i$ / all individuals \\
$\bold{G}_i$ &  Related individuals of $i$ \\
$\boldsymbol{x}_{\bold{G}_i}$ & Covariates of the related individuals of $i$ \\
$t_{\bold{G}_i}$ & Treatments of the related individuals of $i$ \\
$y_i(1, t_{\bold{G}_i})$ / $y_i(0, t_{\bold{G}_i})$& \small Potential outcomes of the individual $i$ with interference $t_{\bold{G}}$ and $t = 1$ / $0$\\
$\boldsymbol{A}$ & \small Adjacency matrix of the graph \\
$\bold{N}_i$ & \small Neighbors of the individual $i$  \\
$\boldsymbol{z}_i$ /
$\boldsymbol{z}_{\bold{X}_i}$ / $\boldsymbol{z}_{\bold{T}_i}$   & \small Representations of $\boldsymbol{x}_i$ / $\boldsymbol{x}_{\bold{G}_i}$ / $t_{\bold{G}_i}$  
\\
$\boldsymbol{z}_{\bold{S}_i}$   & Representations of the structure of
the local network of  $i$
 \\
$\phi_{\bold{S}}$ / $\phi_{\bold{X}}$ / $\phi_{\bold{T}}$ & \small Map function for generating representations $\boldsymbol{z}_{\bold{S}_i}$ / $\boldsymbol{z}_{\bold{X}_i}$
 / $\boldsymbol{z}_{\bold{T}_i}$\\
$f_1$ / $f_0$&  \small Outcome predictors for $y_i(1, t_{\bold{G}})$ / $y_i(0, t_{\bold{G}})$\\
\hline
\end{tabular}
\label{tab:notation}
\end{table*}

\section{Attention mechanisms for IPAtt and SPAtt}\label{app:detailsforatt}
Several attention mechanisms can be implemented for the proposed IPAtt and SPAtt mechanisms. We consider two widely used attention mechanisms: GAT~\cite{2017gat} and the attention mechanism of Transformer based on query and key vectors (abbreviated as QK-based AT)~\cite{vaswani2017attention,NEURIPS2021_f1c15925}. Let $a(\boldsymbol{p}_i,\boldsymbol{p}_k)$ denote the mechanism that estimates the importance between two individuals based on the inputs.  

The implementation of  $a(\boldsymbol{p}_i,\boldsymbol{p}_k)$ for  IPAtt and SPAtt mechanisms with GAT~\cite{2017gat} is
as follows:
\begin{align}
a\left(\boldsymbol{p}_i,\boldsymbol{p}_k\right)=\mathrm{LeakyReLU}\biggr(\left(\boldsymbol{w}^{(l)}\right)^{\top}[\boldsymbol{W}^{(l)}\boldsymbol{p}_i\Vert\boldsymbol{W}^{(l)}\boldsymbol{p}_k]\biggr), \label{eq:gatcomputation}
\end{align}
where $\mathrm{LeakyReLU}$ denotes $\mathrm{LeakyReLU}$ activation function~\cite{maas2013rectifier}, $\boldsymbol{w}^{(l)}$ denotes a learnable parameter vector, and $\boldsymbol{W}^{(l)}$ denotes a learnable parameter matrix.

The implementation of  $a(\boldsymbol{p}_i,\boldsymbol{p}_k)$ for  IPAtt and SPAtt mechanisms with the QK-based AT~\cite{vaswani2017attention,NEURIPS2021_f1c15925}, is as follows:
\begin{align}
a\left(\boldsymbol{p}_i,\boldsymbol{p}_k\right)=\Biggr(\frac{\boldsymbol{p}_{\text{Q}_i}\cdot\left(\boldsymbol{p}_{\text{K}_k}\right)^{\top}}{\sqrt{c_{\text{K}}}}\Biggr), \;\boldsymbol{p}_{\text{Q}_i}=\boldsymbol{W}^{(l)}_{\text{Q}}\boldsymbol{p}_i, \; \boldsymbol{p}_{\text{K}_k}=\boldsymbol{W}^{(l)}_{\text{K}}\boldsymbol{p}_k,
 \end{align}
 where $\boldsymbol{p}_{\text{Q}}$ denotes a query vector, $\boldsymbol{p}_{\text{K}}$ denotes a key vector, $c_{\text{K}}$ denotes the dimension of key vector, $\boldsymbol{W}^{(l)}_{\text{Q}}$ denotes a learnable parameter matrix for the query vector, and $\boldsymbol{W}^{(l)}_{\text{K}}$ denotes a learnable parameter matrix for the key vector. 
 
 For our experiments in Section~\ref{sec:all_exps}, we used the implementation with GAT by default for both IPAtt and SPAtt mechanisms. Additional experiments using implementation with QK-based AT are detailed in Appendix~\ref{appexp:diffatt}.

\section{Proof of existing interference modeling methods cannot fully capture DNE}\label{proof:attfail}

In this section, we prove that existing interference modeling methods cannot capture DNE for some local networks. DNE  consists of two sub-issues:
(I) the importance of different neighbors in contributing to interference varies~\cite{huang2023modeling, lin2023estimating, ma2022learning}, and 
(II) the scale of neighbors varies, leading to different levels of  interference (see Figure~\ref{exp:fails}).

 Let  $\bold{N}_{i}$ denote the set of neighbors of the individual $i$, $d_{i}$ denote degree of the individual $i$, 
$\boldsymbol{p}\in \mathbb{R}^{c_p}$ denote the interference-related information of an individual, which serves as the input of a mean aggregation, GCN, or GAT layer. Here, $\boldsymbol{p}$ is typically initialized as individual information, such as covariates and treatment of the individual, and $c_p$ depends on the specific initialization strategy of $\boldsymbol{p}$.
Let $\boldsymbol{p}'\in \mathbb{R}^{c_{p'}}$ denote the interference representation generated by such a layer for an individual and $\tilde{\mathbf{N}}_i = \mathbf{N}_i\cup i$ denote the set of neighbors of individual $i$ with the self-loop.
\begin{definition}~\label{def:weighting}
    Given an individual $i$ and the set of neighbors $\bold{N}_{i}$, let $a(i,k)$ be a learnable importance estimation mechanism that can assign a non-negative weight to each neighbor $k \in \bold{N}_{i}$. This weight adaptively captures the importance of the neighbor $k$ in contributing to the interference received by the individual $i$, as determined by their interference-related and structural information.  
\end{definition}
\begin{definition}~\label{def:mathDNE}
    Given an individual $i$ and the set of related individuals $\bold{G}_i$, let $\delta_i = f_{\bold{G}}(\boldsymbol{x}_{\bold{G}_i},t_{\bold{G}_i})$ be the effect of interference received by the individual $i$ with DNE. $f_{\bold{G}}(\boldsymbol{x}_{\bold{G}_i},t_{\bold{G}_i}) = g_{\bold{X}}(\boldsymbol{x}_{\bold{G}_i}) + g_{\bold{T}}(t_{\bold{G}_i})$, where $g_{\bold{X}}(\boldsymbol{x}_{\bold{G}_i}) $ and $ g_{\bold{T}}(t_{\bold{G}_i})$ are two aggregation functions.
\end{definition}

Here, we can consider an example for aggregation functions $g_{\bold{X}}(\boldsymbol{x}_{\bold{G}_i}) $ and $ g_{\bold{T}}(t_{\bold{G}_i})$. For simplicity, we consider that $\bold{G}_i$ contains two-hop neighbors for each individual.  In this case, $g_{\bold{X}}(\boldsymbol{x}_{\bold{G}_i}) = s(\bold{N}_i)\cdot\sum_{j \in \bold{N}_i}\alpha_{ij}\cdot s(\bold{N}_j)
\cdot \sum_{k \in \bold{N}_j}\alpha_{jk}\cdot(\boldsymbol{w}_{\bold{X}})^{\top}\cdot\boldsymbol{x}_k$ and  $g_{\bold{T}}(t_{\bold{G}_i}) = s(\bold{N}_i)\cdot\sum_{j \in \bold{N}_i}\alpha_{ij}\cdot s(\bold{N}_j)
\cdot \sum_{k \in \bold{N}_j}\alpha_{jk}\cdot w_{\bold{T}}\cdot t_k$, where $s(\cdot)$ denotes a mechanism that identify the neighbor scales of the individual $i$, $\alpha_{ij}$ denotes the importance of interference between individuals $i$ and $j$, $\boldsymbol{w}_{\bold{X}}$ and $w_{\bold{T}}$ denotes weight parameters, which can be learned from data.

 To address issue (I), a proper method should contain a mechanism, as defined in Definition~\ref{def:weighting}. Such a mechanism can be learned in a data-driven manner, which enables the model to adaptively estimate the importance of different neighbors, rather than assigning equal or fixed weights to all neighbors. Due to the existence of DNE, different neighbors of an individual contribute equally to the interference received by this individual only if they have both similar interference-related information and local network structures.
 To address issue (II),  a proper method 
 needs to generate distinct representations for individuals exposed to different local networks where the scales of neighbors differ. To this end, the mehod needs to include a mechanism that can identify scales of neighbors in $g_{\bold{X}}(\boldsymbol{x}_{\bold{G}_i}) $ and $ g_{\bold{T}}(t_{\bold{G}_i})$. 
 If an interference modeling method fails to address either issue (I) or (II) for individuals exposed to different local networks, it cannot capture the DNE for them.
Specifically, this failure arises when the method either lacks a mechanism to assign importance to neighbors by adaptively estimating their contributions to interference,  which do not address issue (I), or when it generates identical representations for individuals exposed to local networks of different scales of neighbors, which do not address issue (II).

 Existing methods for interference modeling apply mean aggregation~\cite{cai2023generalization,chen2024doubly,doi:10.1080/01621459.2020.1768100,forastiere2022estimating,jiang2022estimating,pmlr-v130-ma21c}, GCN~\cite{cai2023generalization,chen2024doubly,huang2023modeling,jiang2022estimating,LIN2024SITE,pmlr-v130-ma21c,adhikari2025inferring}, or GAT~\cite{huang2023modeling,lin2023estimating,ma2022learning, zhao2024learning} to generate interference representations.    
The mean aggregation can be defined as follows~\cite{pmlr-v130-ma21c}:
\begin{equation}
    \boldsymbol{p}_i'=\sigma\left(\sum_{k \in \tilde{\mathbf{N}}_i} \frac{1}{|\tilde{\mathbf{N}}_i|}\boldsymbol{W}\boldsymbol{p}_k\right), \label{pro:meanagg}
\end{equation}
where $\boldsymbol{W}$ denotes a learnable parameter matrix. Notably, $\boldsymbol{W}$ is not always applied for mean aggregation, as seen in methods for treatment aggregation~\cite{cai2023generalization,jiang2022estimating}. 
The individual-level aggregation of GCN can be defined as follows~\cite{jiang2022estimating,kipf2016semi}:  
\begin{equation}
\boldsymbol{p}_i'=\sigma\left(\sum_{k \in \tilde{\mathbf{N}}_i}  \frac{1}{\sqrt{\tilde{d}_i \tilde{d}_k}}\boldsymbol{W}\boldsymbol{p}_k\right), 
\label{pro:GCNagg}
\end{equation}
where $\tilde{d}_i$ is the degree of the individual $i$ with the self-loop. The individual-level aggregation of GCN is transformed from the aggregation mechanism of the original GCN~\cite{kipf2016semi}.
The individual-level aggregation of GAT can be defined as follows~\cite{2017gat}: 
\begin{equation}
\boldsymbol{p}'_i=\sigma\left(\sum_{k \in \tilde{\mathbf{N}}_i} \alpha_{ik} \cdot \boldsymbol{W}\boldsymbol{p}_k\right),
\label{pro:GATagg}
\end{equation}
where $\alpha_{ik}$ is the estimated attention weight between individuals $i$ and $k$, computed by the graph attention mechanism~\cite{2017gat}. Here, $\alpha_{ik}$ is typically normalized by softmax, such that $\sum_{k \in \tilde{\mathbf{N}}_i} \alpha_{ik} = 1$, as seen in previous methods~\cite{huang2023modeling,lin2023estimating,ma2022learning,2017gat}.
 These methods cannot capture DNE for some local networks, as stated in  Proposition~\ref{app:pro1}. 

\begin{proposition}\label{app:pro1}
 Interference representation generated by the mean aggregation, GCN, or GAT cannot capture DNE for some local networks.
\end{proposition}

We now provide a proof for Proposition~\ref{app:pro1}.
\begin{proof}
 To prove Proposition~\ref{app:pro1}, we need to provide examples in which the interference representation generated by the mean aggregation, GCN, or GAT cannot capture DNE. For simplicity, we consider cases of local networks of individuals, which consist of individuals and their 1-hop neighbors. In this case,  the neighbor set with the self-loop of an individual can constitute the local network of the individual, e.g., $\tilde{\bold{N}}_i$.

First, we prove that the interference representation generated by a mean aggregation cannot capture DNE. 
The mean aggregation uses the same importance for different neighbors, so it is unable to address the issue (I) of DNE. Furthermore, the mean aggregation cannot address the issue (II) of DNE for some local networks. We can derive from the mean aggregation, i.e., equality~(\ref{pro:meanagg}) as follows:
 \begin{equation}
       \boldsymbol{p}_i' = \sigma\left(\sum_{k \in \tilde{\mathbf{N}}_i} \frac{1}{|\tilde{\mathbf{N}}_i|}\boldsymbol{W}\boldsymbol{p}_k\right) = \sigma\left(\boldsymbol{W}\sum_{k \in \tilde{\mathbf{N}}_i} \frac{1}{|\tilde{\mathbf{N}}_i|}\boldsymbol{p}_k\right) = \sigma\left( \boldsymbol{W}\bar{\boldsymbol{p}}_{\tilde{\bold{N}}_i}\right), \label{pfmeanaggfils}
 \end{equation}
 where $\bar{\boldsymbol{p}}_{\tilde{\bold{N}}_i}$ represents the mean of interference-related information of neighbors $\tilde{\bold{N}}_i$ with the self-loop of the individual $i$. Based on equality~(\ref{pfmeanaggfils}), we consider a case that individuals $i$ and $j$ are exposed to two different local networks, where $\bar{\boldsymbol{p}}_{\tilde{\bold{N}}_i} = \bar{\boldsymbol{p}}_{\tilde{\bold{N}}_j}$, and $|{\tilde{\bold{N}}_i}|\neq |{\tilde{\bold{N}}_j}|$. In this case, the mean aggregation cannot generate distinct representations for individuals $i$ and $j$, even though they are exposed to different local networks, where $|{\tilde{\bold{N}}_i}|\neq |{\tilde{\bold{N}}_j}|$, which can be proved as follows:
 \begin{equation}
       \boldsymbol{p}_i' = \sigma\left( \boldsymbol{W}\bar{\boldsymbol{p}}_{\tilde{\bold{N}}_i}\right) = \sigma\left( \boldsymbol{W}\bar{\boldsymbol{p}}_{\tilde{\bold{N}}_j}\right) = \boldsymbol{p}_j'. \label{meanaggfailpf}
 \end{equation}
 This reveals that the mean aggregation cannot address the issue (II) of DNE for this case, as it generates the same representations for individuals $i$ and $j$  but $|{\tilde{\bold{N}}_i}|\neq|{\tilde{\bold{N}}_j}|$. We can consider another case (similar to examples in Figure~\ref{exp:fails}) that similar individuals $i$ and $j$ (i.e., $\boldsymbol{p}_i = \boldsymbol{p}_j$) are exposed to two different local networks where  $\boldsymbol{p}_i=\boldsymbol{p}_{k_1}, \forall k_1 \in {\tilde{\bold{N}}_i}$,  $\boldsymbol{p}_j=\boldsymbol{p}_{k_2}, \forall k_2 \in {\tilde{\bold{N}}_j}$, and $|{\tilde{\bold{N}}_i}|\neq |{\tilde{\bold{N}}_j}|$. In this case, $\bar{\boldsymbol{p}}_{\tilde{\bold{N}}_i}=\boldsymbol{p}_i$, and $\bar{\boldsymbol{p}}_{\tilde{\bold{N}}_j}=\boldsymbol{p}_j$ hold, we then have:
 \begin{equation}
       \boldsymbol{p}_i' = \sigma\left( \boldsymbol{W}\bar{\boldsymbol{p}}_{\tilde{\bold{N}}_i}\right) =
       \sigma\left(\boldsymbol{W}\boldsymbol{p}_i\right)  = \sigma\left(\boldsymbol{W}\boldsymbol{p}_j\right)= \sigma\left( \boldsymbol{W}\bar{\boldsymbol{p}}_{\tilde{\bold{N}}_j}\right) = \boldsymbol{p}_j'. \label{meanaggfailpf2}
 \end{equation}
 This indicates that the mean aggregation is also unable to address the issue (II) of DNE for this case, as it generates the same representations for individuals $i$ and $j$. Therefore, the mean aggregation cannot address the issue (II) of DNE for some local networks, as shown in equalities~(\ref{meanaggfailpf}) and (\ref{meanaggfailpf2}). As a result, the mean aggregation cannot capture DNE, as it cannot address issue (I) while it cannot address issue (II) for some local networks.

 Next, we prove that GCN cannot capture DNE for some local networks. 
 GCN can degenerate to mean aggregation when $\tilde{d}_i=\tilde{d}_{k}, \forall k \in  {\tilde{\bold{N}}_i}$, which can be proved as follows:
 \begin{equation}
\begin{aligned}
       \boldsymbol{p}_i'
       =&\;\sigma\left(\sum_{k \in \tilde{\mathbf{N}}_i}  \frac{1}{\sqrt{\tilde{d}_i \tilde{d}_{k}}}\boldsymbol{W}\boldsymbol{p}_{k}\right)\\=&\;\sigma\left(\sum_{k \in \tilde{\mathbf{N}}_i}  \frac{1}{\sqrt{\tilde{d}_i \tilde{d}_{i}}}\boldsymbol{W}\boldsymbol{p}_{k}\right) 
       \\=&\;\sigma\left(\sum_{k \in \tilde{\mathbf{N}}_i}  \frac{1}{|{\tilde{\bold{N}}_i}|}\boldsymbol{W}\boldsymbol{p}_{k}\right). \label{GCNaggcannotpf}
\end{aligned}
 \end{equation}
 Therefore, GCN is unable to address the issue (I) of DNE for some networks, where $\tilde{d}_i=\tilde{d}_{k}, \forall k \in  {\tilde{\bold{N}}_i}$.
 Furthermore, similar to the mean aggregation, it cannot address the issue (II) of DNE for some local networks.
 We consider the case that individuals $i$ and $j$ are exposed to two different local networks, where $\tilde{d}_i=\tilde{d}_{k_1}, \forall k_1 \in  {\tilde{\bold{N}}_i}$, $\tilde{d}_j=\tilde{d}_{k_2}, \forall k_2 \in  {\tilde{\bold{N}}_j}$, 
 $\bar{\boldsymbol{p}}_{\tilde{\bold{N}}_i} = \bar{\boldsymbol{p}}_{\tilde{\bold{N}}_j}$, and $|{\tilde{\bold{N}}_i}|\neq |{\tilde{\bold{N}}_j}|$. In this case, we can derive from  equality~(\ref{GCNaggcannotpf}) as follows:
 \begin{equation}
 \begin{aligned}
       \boldsymbol{p}_i'
       &=\sigma\left(\sum_{k_1 \in \tilde{\mathbf{N}}_i}  \frac{1}{\sqrt{\tilde{d}_i \tilde{d}_{k_1}}}\boldsymbol{W}\boldsymbol{p}_{k_1}\right)=\sigma\left(\sum_{k_1 \in \tilde{\mathbf{N}}_i}  \frac{1}{|{\tilde{\bold{N}}_i}|}\boldsymbol{W}\boldsymbol{p}_{k_1}\right) =\sigma\biggr(\boldsymbol{W}\bar{\boldsymbol{p}}_{\tilde{\bold{N}}_i}\biggr), \\ \boldsymbol{p}_j'&= \sigma\left(\sum_{k_2 \in \tilde{\mathbf{N}}_j}  \frac{1}{\sqrt{\tilde{d}_j \tilde{d}_{k_2}}}\boldsymbol{W}\boldsymbol{p}_{k_2}\right) =\sigma\left(\sum_{k_2 \in \tilde{\mathbf{N}}_j}  \frac{1}{|{\tilde{\bold{N}}_j}|}\boldsymbol{W}\boldsymbol{p}_{k_2}\right)   =\sigma\biggr(\boldsymbol{W}\bar{\boldsymbol{p}}_{\tilde{\bold{N}}_j}\biggr).
 \label{GCNaggcannotpf2}
 \end{aligned}
 \end{equation}
 Based on equalities~(\ref{GCNaggcannotpf2}) and $\bar{\boldsymbol{p}}_{\tilde{\bold{N}}_i} = \bar{\boldsymbol{p}}_{\tilde{\bold{N}}_j}$, we can have $\boldsymbol{p}_i'=\boldsymbol{p}_j'$ in this case, which  indicates that the GCN cannot address both issues (I) and (II) of DNE in this scenario, as it assigns equal importance to all neighbors and generates the same representations for individuals $i$ and $j$,  but $|{\tilde{\bold{N}}_i}|\neq|{\tilde{\bold{N}}_j}|$.
Similarly, we consider another case that similar individuals $i$ and $j$ (i.e., $\boldsymbol{p}_i = \boldsymbol{p}_j$) are exposed to two different local networks where $\tilde{d}_i=\tilde{d}_{k_1}, \boldsymbol{p}_i=\boldsymbol{p}_{k_1}, \forall k_1 \in {\tilde{\bold{N}}_i}$, $\tilde{d}_j=\tilde{d}_{k_2},\boldsymbol{p}_j=\boldsymbol{p}_{k_2}, \forall k_2 \in {\tilde{\bold{N}}_j}$, and $|{\tilde{\bold{N}}_i}|\neq |{\tilde{\bold{N}}_j}|$. 
In this case, we can have 
\vspace{-10pt}
 \begin{equation}
 \begin{aligned}
       \boldsymbol{p}_i'
       &=\sigma\left(\sum_{k_1 \in \tilde{\mathbf{N}}_i}  \frac{1}{\sqrt{\tilde{d}_i \tilde{d}_{k_1}}}\boldsymbol{W}\boldsymbol{p}_{k_1}\right)=\sigma\left(\sum_{k_1 \in \tilde{\mathbf{N}}_i}  \frac{1}{\sqrt{\tilde{d}_i \tilde{d}_{i}}}\boldsymbol{W}\boldsymbol{p}_{i}\right) =\sigma\left(\boldsymbol{W}\boldsymbol{p}_i\right),
       \\ \boldsymbol{p}_j'&= \sigma\left(\sum_{k_2 \in \tilde{\mathbf{N}}_j}  \frac{1}{\sqrt{\tilde{d}_j \tilde{d}_{k_2}}}\boldsymbol{W}\boldsymbol{p}_{k_1}\right) = \sigma\left(\sum_{k_2 \in \tilde{\mathbf{N}}_j}  \frac{1}{\sqrt{\tilde{d}_j \tilde{d}_{j}}}\boldsymbol{W}\boldsymbol{p}_{j}\right)  =\sigma\left(\boldsymbol{W}\boldsymbol{p}_j\right). \label{GCNaggcannotpf3}
 \end{aligned}
 \end{equation}
 Based on equalities~(\ref{GCNaggcannotpf3}) and $\boldsymbol{p}_i = \boldsymbol{p}_j$, we can have  $\boldsymbol{p}_i'=\boldsymbol{p}_j'$.
 This indicates that the GCN is also unable to address the issues (I) and (II) of DNE for this case, as it still generates the same representations for individuals $i$ and $j$ but $|{\tilde{\bold{N}}_i}|\neq |{\tilde{\bold{N}}_j}|$.
 As a result, the GCN cannot capture DNE for some local networks, as it cannot address issues (I) and (II) jointly for these cases.

 Lastly, we prove that GAT with softmax operation cannot capture DNE for some local networks.
 We can derive from the aggregation of GAT, i.e., equality~(\ref{pro:GATagg}) as follows:
 \begin{equation}\label{GATaggcannnotpf}
 \begin{aligned}
&\;\boldsymbol{p}'_i\\=&\;\sigma\left(\sum_{k \in \tilde{\mathbf{N}}_i} \alpha_{ik} \cdot \boldsymbol{W}\boldsymbol{p}_k\right)
\\=&\;\sigma\left(\alpha_{ii} \cdot \boldsymbol{W}\boldsymbol{p}_i + \alpha_{i2} \cdot \boldsymbol{W}\boldsymbol{p}_2 + \dots + \alpha_{in} \cdot \boldsymbol{W}\boldsymbol{p}_n\right)
\\=&\;\sigma\biggl(\left(\alpha_{ii} \cdot \boldsymbol{W}\boldsymbol{p}_i + \alpha_{i2} \cdot \boldsymbol{W}\boldsymbol{p}_i + \dots + \alpha_{in} \cdot \boldsymbol{W}\boldsymbol{p}_i\right) + \\&
\left(\alpha_{i2} \cdot \boldsymbol{W}\boldsymbol{p}_2 - \alpha_{i2} \cdot \boldsymbol{W}\boldsymbol{p}_i\right) + \dots + \left(\alpha_{in} \cdot \boldsymbol{W}\boldsymbol{p}_n - \alpha_{in} \cdot \boldsymbol{W}\boldsymbol{p}_i\right)\biggr)
\\=&\;
\sigma\Biggl(\left(\left[\sum_{k \in \tilde{\mathbf{N}}_i} \alpha_{ik}\right]\cdot \boldsymbol{W}\boldsymbol{p}_i\right) +\sum_{k \in \mathbf{N}_i} \alpha_{ik} \left(\boldsymbol{W}\boldsymbol{p}_k-\boldsymbol{W}\boldsymbol{p}_i\right)\Biggr)
\\=&\;
\sigma\Biggl(\boldsymbol{W}\boldsymbol{p}_i +\sum_{k \in \mathbf{N}_i} \alpha_{ik} \left(\boldsymbol{W}\boldsymbol{p}_k-\boldsymbol{W}\boldsymbol{p}_i\right)\Biggr),
 \end{aligned}
\end{equation}
where last equality holds as $\sum_{k \in \tilde{\mathbf{N}}_i} \alpha_{ik} = 1$ due to the softmax operation in GAT. Here, such attention weights $\alpha_{ik}$ are estimated by the GAT, which tends to capture a pattern that assigns importance between two individuals based on individual information, i.e., $\boldsymbol{p}_k$ and $\boldsymbol{p}_i$.

Although the GAT can adaptively estimates the importance of interference between individuals and their neighbors using individual information, it learns the same pattern across all local networks in a graph.
This may introduce risks arising from a potential \textbf{conflict issue} between the estimated importance and individual information differences during aggregation.
For example, a GAT learns a pattern from a graph. It tends to assign high importance to similar neighbors ($\boldsymbol{p}_{k}-\boldsymbol{p}_i = \boldsymbol{0}$) of the individual $i$; while assigning low importance to dissimilar neighbors ($\boldsymbol{p}_{k}-\boldsymbol{p}_i \gg \boldsymbol{0}$) of individual $i$, where $\boldsymbol{0}$ denotes zero vector.
Here, if a neighbor $\boldsymbol{p}_k$ is similar to $\boldsymbol{p}_i$, 
$\alpha_{ik}$ can be high, but the result of $\boldsymbol{W}\boldsymbol{p}_k-\boldsymbol{W}\boldsymbol{p}_i$ will be close to $\boldsymbol{0}$. In contrast,  if a neighbor $\boldsymbol{p}_k$ is totally different to $\boldsymbol{p}_i$, every element of $\boldsymbol{W}\boldsymbol{p}_k-\boldsymbol{W}\boldsymbol{p}_i$ will be large, but  the value of $\alpha_{ik}$ can be zero due to the softmax operation.
Based on the analysis above, we consider a case, where similar individuals $i$ and $j$ (i.e., $\boldsymbol{p}_i = \boldsymbol{p}_j$) are exposed to two different local networks. In their networks,  $\forall k_1 \in \tilde{\mathbf{N}}_i, \forall k_2 \in \tilde{\mathbf{N}}_j$, their neighbors are either similar to them, i.e, $\boldsymbol{p}_i=\boldsymbol{p}_{k_1}$ and $\boldsymbol{p}_j=\boldsymbol{p}_{k_2}$, or totally different from them, i.e., $\boldsymbol{p}_{k_1}-\boldsymbol{p}_i \gg 0$ and $\boldsymbol{p}_{k_2}-\boldsymbol{p}_j \gg 0$. In this case, we can have as follows: 
\begin{equation}\label{GATpfattcondi}
\begin{aligned}
&\alpha_{ik_1} &= 
\begin{cases}
> 0, & \text{if }  \boldsymbol{p}_{k_1}=\boldsymbol{p}_i, \\
0, & \text{if } \boldsymbol{p}_{k_1} - \boldsymbol{p}_i \gg \boldsymbol{0}.
\end{cases}\\
& 
\alpha_{jk_2} &= 
\begin{cases}
> 0, & \text{if }  \boldsymbol{p}_{k_2}=\boldsymbol{p}_j, \\
0, & \text{if } \boldsymbol{p}_{k_2} - \boldsymbol{p}_j \gg \boldsymbol{0}.
\end{cases}
\end{aligned}
\end{equation}
Based on (\ref{GATpfattcondi}), we can have the results for this case, as follows: 
 \begin{equation}\label{GATaggcannnotpf2}
 \begin{aligned}
\boldsymbol{p}'_i &=
\sigma\Biggl(\boldsymbol{W}\boldsymbol{p}_i +\sum_{{k_1} \in \mathbf{N}_i} \alpha_{i{k_1}} \left(\boldsymbol{W}\boldsymbol{p}_{k_1}-\boldsymbol{W}\boldsymbol{p}_i\right)\Biggr) =\sigma\left(\boldsymbol{W}\boldsymbol{p}_i\right),\\\boldsymbol{p}'_j&= \sigma\Biggl(\boldsymbol{W}\boldsymbol{p}_j +\sum_{{k_2} \in \mathbf{N}_j} \alpha_{i{k_2}} \left(\boldsymbol{W}\boldsymbol{p}_{k_2}-\boldsymbol{W}\boldsymbol{p}_j\right)\Biggr)=\sigma\left(\boldsymbol{W}\boldsymbol{p}_j\right).
 \end{aligned}
\end{equation}
Based on equalites~(\ref{GATaggcannnotpf2}) and $\boldsymbol{p}_i = \boldsymbol{p}_j$, we can have $\boldsymbol{p}'_i = \boldsymbol{p}'_j$. This indicates that the GAT is unable to address the issue (II) of DNE for this case, as it 
still generates the same representations for individuals $i$ and $j$ but $|{\tilde{\bold{N}}_i}|\neq |{\tilde{\bold{N}}_j}|$.
We consider another case, where that similar individuals $i$ and $j$ (i.e., $\boldsymbol{p}_i = \boldsymbol{p}_j$) are exposed to two different local networks, where $\boldsymbol{p}_i=\boldsymbol{p}_{k_1}, \forall k_1 \in {\tilde{\bold{N}}_i}$, $\boldsymbol{p}_j=\boldsymbol{p}_{k_2}, \forall k_2 \in {\tilde{\bold{N}}_j}$,  and $|{\tilde{\bold{N}}_i}|\neq |{\tilde{\bold{N}}_j}|$. We can derive from equality~(\ref{GATaggcannnotpf}), as follows:
\begin{equation}
    \begin{aligned}
\boldsymbol{p}'_i&=
\sigma\Biggl(\boldsymbol{W}\boldsymbol{p}_i +\sum_{{k_1} \in \mathbf{N}_i} \alpha_{i{k_1}} \left(\boldsymbol{W}\boldsymbol{p}_{k_1}-\boldsymbol{W}\boldsymbol{p}_i\right)\Biggr) =\sigma(\boldsymbol{W}\boldsymbol{p}_i),\\ \boldsymbol{p}'_j&=\sigma\Biggl(\boldsymbol{W}\boldsymbol{p}_j +\sum_{{k_2} \in \mathbf{N}_j} \alpha_{i{k_2}} \left(\boldsymbol{W}\boldsymbol{p}_{k_2}-\boldsymbol{W}\boldsymbol{p}_j\right)\Biggr) = \sigma(\boldsymbol{W}\boldsymbol{p}_j) .
\label{GATaggcannnotpf3} 
 \end{aligned}
\end{equation}
Based on equalites~(\ref{GATaggcannnotpf3}) and $\boldsymbol{p}_i = \boldsymbol{p}_j$, we can have $\boldsymbol{p}'_i = \boldsymbol{p}'_j$. This indicates that GAT is also unable to address issue (II) of DNE in this case, as it 
still generates the same representations for individuals $i$ and $j$ but $|{\tilde{\bold{N}}_i}|\neq |{\tilde{\bold{N}}_j}|$. Furthermore, the estimated importance of neighbors also tends to be the same for this case, which further prevents traditional GAT-based methods from capturing DNE. 
As a result, GAT cannot capture DNE for some local networks, such as those in equailites~(\ref{GATaggcannnotpf2}) and (\ref{GATaggcannnotpf3}), as it cannot address issues (I) and (II) jointly for these cases. 

As discussed above, interference representation generated by mean aggregation, GCN, or GAT, cannot capture DNE for some local networks. 
\end{proof}
This indicates that most existing interference modeling methods face challenges in capturing DNE effectively.

Apart from the above-mentioned aggregation-based methods for ITE estimation, there are some pooling-based methods in the graph community, the max pooling~\cite{hamilton2017inductive} is one of the popular pooling-based methods. 
\begin{proposition}\label{app:pro2}
Interference representation generated by max-pooling or min-pooling operations cannot capture DNE for some local networks.
\end{proposition}

We now prove Proposition~\ref{app:pro2}, as follows:
\begin{proof}
As pooling-based methods do not contain a mechanism to estimate the importance of interference among individuals,  they cannot address issue (I).  

Now, we discuss that pooling-based methods cannot address issue (II) for some networks.  For simplicity, we consider cases of local networks of individuals, which consist of individuals and their 1-hop neighbors. In this case,  the neighbor set with the self-loop of the individual can constitute the local network of the individual, e.g., $\tilde{\bold{N}}_i$. Let $\rm{Pool}$ be a max-pooling or min-pooling operation.
We consider the case (similar to examples in Figure~\ref{exp:fails}) that individuals $i$ and $j$ (i.e., $\boldsymbol{p}_i = \boldsymbol{p}_j$) are exposed to two different local networks where  $\boldsymbol{p}_i=\boldsymbol{p}_{k_1}, \forall k_1 \in {\tilde{\bold{N}}_i}$,  $\boldsymbol{p}_j=\boldsymbol{p}_{k_2}, \forall k_2 \in {\tilde{\bold{N}}_j}$, and $|{\tilde{\bold{N}}_i}|\neq |{\tilde{\bold{N}}_j}|$.

We have the result of the pooling operation as follows:
    \begin{equation}
    \begin{aligned}
 \boldsymbol{p}'_i =   &\;{\rm{Pool}}\biggr(\{\sigma(\boldsymbol{W}\boldsymbol{p}_{k_1}),\forall k_1\in\tilde{\bold{N}}_i\}\biggr) \\ = &\;  {\rm{Pool}}\biggr(\{\sigma(\boldsymbol{W}\boldsymbol{p}_{i}),\dots,\sigma(\boldsymbol{W}\boldsymbol{p}_{i})\}\biggr) \\ = &\; \sigma(\boldsymbol{W}\boldsymbol{p}_{i}), \\
 \boldsymbol{p}'_j = &\; {\rm{Pool}}\biggr(\{\sigma(\boldsymbol{W}\boldsymbol{p}_{k_2}),\forall k_2\in\tilde{\bold{N}}_j\}\biggr) \\= &\; {\rm{Pool}}\biggr(\{\sigma(\boldsymbol{W}\boldsymbol{p}_{j}),\dots,\sigma(\boldsymbol{W}\boldsymbol{p}_{j})\}\biggr) \\= &\;\sigma(\boldsymbol{W}\boldsymbol{p}_{j}) .\label{poolcannnotpf3}
\end{aligned}
    \end{equation}
Based on equalites~(\ref{poolcannnotpf3}) and $\boldsymbol{p}_i = \boldsymbol{p}_j$, we can have $\boldsymbol{p}'_i = \boldsymbol{p}'_j$. 
This indicates that the pooling operation is  unable to address the issue (II) of DNE for this case, as it 
still generates the same representations for individuals $i$ and $j$ but $|{\tilde{\bold{N}}_i}|\neq |{\tilde{\bold{N}}_j}|$. 
As a result, the pooling operation cannot capture DNE, as it cannot address issues (I) and (II) jointly.

\end{proof}

\section{Proof of Proposition~\ref{sec:pro3}}\label{app:motivation:degree}
In this section, we prove that the proposed NIM layer can capture DNE. A proper aggregation function that can capture DNE should address two sub-issues.
(I) The importance of different neighbors in contributing to interference varies~\cite{lin2023estimating}. 
(II) The scale of neighbors varies, leading to different levels of  interference (see Figure~\ref{exp:fails}).
IPAtt and SPAtt mechanisms can address the issue (I) of DNE in an adaptive manner, while SPAtt mechanism can mitigate the conflict issue (see Appendix~\ref{proof:attfail}) of GAT-based methods.
However, for the case where individuals in the same local networks have similar individual information, we still suffer from a similar issue to GAT-based methods~\cite{huang2023modeling,lin2023estimating,ma2022learning} when applying the normalization operation to the estimated importance without the message amplifier.
Here, we consider a case that similar individuals $i$ and $j$ (i.e., $\boldsymbol{z}_{\bullet_i}^{(l-1)}=\boldsymbol{z}_{\bullet_j}^{(l-1)}$, $\bullet \in \{\bold{X},\bold{T}\}$) are exposed to two different local networks, which consist of individuals and their 1-hop neighbors.
In their local networks,  $\boldsymbol{z}_{\bullet_i}^{(l-1)}=\boldsymbol{z}_{\bullet_{k_1}}^{(l-1)}, \forall k_1 \in \tilde{\bold{N}}_i$,  $\boldsymbol{z}_{\bullet_j}^{(l-1)}=\boldsymbol{z}_{\bullet_{k_2}}^{(l-1)}, \forall k_2 \in \tilde{\bold{N}}_j$,  $|\tilde{\mathbf{N}}_i| = m$,  $|\tilde{\mathbf{N}}_j| = n$, and $n \gg m$.
In this case, $\boldsymbol{W}^{(l)}_{\bullet_{\rm{in}}}\boldsymbol{z}_{\bullet_i}^{(l-1)} = \boldsymbol{W}^{(l)}_{\bullet_{\rm{in}}}\boldsymbol{z}_{\bullet_{k_1}}^{(l-1)}$,  $\boldsymbol{W}^{(l)}_{\bullet_{\rm{st}}}\boldsymbol{z}_{\bullet_i}^{(l-1)} = \boldsymbol{W}^{(l)}_{\bullet_{\rm{st}}}\boldsymbol{z}_{\bullet_{k_1}}^{(l-1)}$, $\boldsymbol{W}^{(l)}_{\bullet_{\rm{in}}}\boldsymbol{z}_{\bullet_j}^{(l-1)} = \boldsymbol{W}^{(l)}_{\bullet_{\rm{in}}}\boldsymbol{z}_{\bullet_{k_2}}^{(l-1)}$, and   $\boldsymbol{W}^{(l)}_{\bullet_{\rm{st}}}\boldsymbol{z}_{\bullet_j}^{(l-1)} = \boldsymbol{W}^{(l)}_{\bullet_{\rm{st}}}\boldsymbol{z}_{\bullet_{k_2}}^{(l-1)}$ hold. Therefore, we omit  $\boldsymbol{W}^{(l)}_{\bullet_{\rm{in}}}$ and $\boldsymbol{W}^{(l)}_{\bullet_{\rm{st}}}$ in our proof for simplicity.  
Then, we have:
\begin{equation}
\begin{aligned}
\mathring{\boldsymbol{z}}_{\bullet_i}^{(l)} =& \;  \mathring{\pi}_{\bullet}^{(l)}\cdot \sigma\Biggr(\sum_{k_1\in\tilde{\bold{N}}_i} \alpha_{i{k_1}}^{\rm{in}}\cdot\boldsymbol{z}_{\bullet_{k_1}}^{(l-1)}\Biggr) +\\ &\; \left(1-\mathring{\pi}_{\bullet}^{(l)}\right)\cdot \sigma\Biggr(\sum_{{k_1}\in\tilde{\bold{N}}_i} \alpha_{i{k_1}}^{\rm{st}}\cdot\boldsymbol{z}_{\bullet_{k_1}}^{(l-1)}\Biggr)
\\=& \; \mathring{\pi}_{\bullet}^{(l)}\cdot \sigma\Biggr(\sum_{{k_1}\in\tilde{\bold{N}}_i} \alpha_{i{k_1}}^{\rm{in}}\cdot\boldsymbol{z}_{\bullet_i}^{(l-1)}\Biggr) + \\ &\; \left(1-\mathring{\pi}_{\bullet}^{(l)}\right)\cdot \sigma\Biggr(\sum_{{k_1}\in\tilde{\bold{N}}_i} \alpha_{i{k_1}}^{\rm{st}}\cdot\boldsymbol{z}_{\bullet_i}^{(l-1)}\Biggr)\\=& \;
\mathring{\pi}_{\bullet}^{(l)}\cdot \sigma\left(\boldsymbol{z}_{\bullet_i}^{(l-1)}\right) + \left(1-\mathring{\pi}_{\bullet}^{(l)}\right)\cdot \sigma\left(\boldsymbol{z}_{{\bullet_i}}^{(l-1)}\right),
\\
\mathring{\boldsymbol{z}}_{\bullet_j}^{(l)} =& \;  \mathring{\pi}_{\bullet}^{(l)}\cdot \sigma\Biggr(\sum_{k_2\in\tilde{\bold{N}}_j} \alpha_{j{k_2}}^{\rm{in}}\cdot\boldsymbol{z}_{\bullet_{k_2}}^{(l-1)}\Biggr) +\\ &\; \left(1-\mathring{\pi}_{\bullet}^{(l)}\right)\cdot \sigma\Biggr(\sum_{{k_2}\in\tilde{\bold{N}}_j} \alpha_{j{k_2}}^{\rm{st}}\cdot\boldsymbol{z}_{\bullet_{k_2}}^{(l-1)}\Biggr)
\\=& \; \mathring{\pi}_{\bullet}^{(l)}\cdot \sigma\Biggr(\sum_{{k_2}\in\tilde{\bold{N}}_j} \alpha_{j{k_2}}^{\rm{in}}\cdot\boldsymbol{z}_{\bullet_j}^{(l-1)}\Biggr) +\\ &\; \left(1-\mathring{\pi}_{\bullet}^{(l)}\right)\cdot \sigma\Biggr(\sum_{{k_2}\in\tilde{\bold{N}}_j} \alpha_{j{k_2}}^{\rm{st}}\cdot\boldsymbol{z}_{\bullet_j}^{(l-1)}\Biggr)\\=& \;
\mathring{\pi}_{\bullet}^{(l)}\cdot \sigma\left(\boldsymbol{z}_{\bullet_j}^{(l-1)}\right) + \left(1-\mathring{\pi}_{\bullet}^{(l)}\right)\cdot \sigma\left(\boldsymbol{z}_{{\bullet_j}}^{(l-1)}\right).\label{proof:failforours}
\end{aligned}
\end{equation}

These equalities hold because the sum of importance estimated by every partial attention mechanism is 1 due to the normalization operation.
 Based on equalities~(\ref{proof:failforours}) and $\boldsymbol{z}_{\bullet_i}^{(l-1)}=\boldsymbol{z}_{\bullet_j}^{(l-1)}$, we can have $\mathring{\boldsymbol{z}}_{\bullet_i}^{(l)} = \mathring{\boldsymbol{z}}_{\bullet_j}^{(l)}$.
 We can observe that,  when the message amplifier is not applied, the aggregated results are identical regardless of the scales of neighbors in different local networks for this case. This holds even when the scales of their neighbors differ significantly, i.e., $|\tilde{\mathbf{N}}_i| = m$,  $|\tilde{\mathbf{N}}_j| = n$, and $n \gg m$. 
Therefore, we apply the message amplifier to ensure the generated representations of two individuals differ according to the size of neighbors, i.e., $|\tilde{\mathbf{N}}_i|$ and $|\tilde{\mathbf{N}}_j|$, which correspond to their degrees $\tilde{d}_i$ and $\tilde{d}_j$ with self-loops.

\begin{proposition}[Proposition~\ref{sec:pro3}, main text]\label{app:pro3}
Interference representation generated by NIM layers after applying the message amplifier can address the issue (II), even in some local networks where all individuals have similar interference-related information.
\end{proposition}
We now prove Proposition~\ref{app:pro3}.
\begin{proof}
For the individual $i$ with $|\tilde{\mathbf{N}}_i| = m$, the individual $j$ with $|\tilde{\mathbf{N}}_j| = n$, and $n \gg m$, we have:
\begin{equation}
\begin{aligned}
(1+\eta_i)\cdot\mathring{\boldsymbol{z}}_{\bullet_i}^{(l)} &= \Biggr(1+\pi_{\eta}\cdot\frac{\log(\tilde{d}_i)}{\sum_{i=1}^{n_{\rm{tr}}}{\log(\tilde{d}_i)}}\Biggr) \cdot\mathring{\boldsymbol{z}}_{\bullet_i}^{(l)} \\ &\;=
\Biggr(1+\pi_{\eta}\cdot\frac{\log(m)}{\sum_{i=1}^{n_{\rm{tr}}}{\log(\tilde{d}_i)}}\Biggr) \cdot\mathring{\boldsymbol{z}}_{\bullet_i}^{(l)},\\
(1+\eta_j)\cdot\mathring{\boldsymbol{z}}_{\bullet_j}^{(l)}&= \Biggr(1+\pi_{\eta}\cdot\frac{\log(\tilde{d}_j)}{\sum_{j=1}^{n_{\rm{tr}}}{\log(\tilde{d}_j)}}\Biggr) \cdot\mathring{\boldsymbol{z}}_{\bullet_j}^{(l)} \\ &\;=
\Biggr(1+\pi_{\eta}\cdot\frac{\log(n)}{\sum_{j=1}^{n_{\rm{tr}}}{\log(\tilde{d}_j)}}\Biggr) \cdot\mathring{\boldsymbol{z}}_{\bullet_j}^{(l)}. \label{proofwithAMn}
\end{aligned}
\end{equation}
This shows that the generated interference representation with the message amplifier differs according to the degree of neighbors, which can address the issue (II) of DNE, even in local networks where all individuals have similar information. 
\end{proof}
Therefore, by applying two partial attention mechanisms to address issue (I) and applying the message amplifier to address issue (II), the proposed NIM layer can capture DNE.

\section{Error  bound}\label{app:bound}

Our theoretical analysis for the error bound is inspired by~\citet{pmlr-v70-shalit17a}, we extend their theoretical analysis of non-graph data to graph data by using the proposed representation balancing strategy.  
Let $\Phi$ be a map functions, assume it is twice-differentiable and invertible, following~\citet{pmlr-v70-shalit17a,cai2023generalization,wang2023optimal}. Let $\Phi^{-1}$ be the inverse of $\Phi$ and  $\boldsymbol{u} = (\boldsymbol{x},\boldsymbol{x}_{\bold{G}},\boldsymbol{t}_{\bold{G}})$ for simplicity. 
Then, the expected loss for an individual is as follows:
\begin{align}
    \mathcal{L}_{f,\Phi}(\boldsymbol{u},t)=\mathlarger{\int}_{\mathcal{Y}}\mathcal{L}\biggl(y(t,\boldsymbol{t}_{\bold{G}}),f\left(\Phi(\boldsymbol{u}),t\right)\biggr)p\biggl(y(t,\boldsymbol{t}_{\bold{G}})\mid\boldsymbol{x},\boldsymbol{x}_{\bold{G}}\biggr)dy(t,\boldsymbol{t}_{\bold{G}}).
\end{align}

We consider MSE for $\mathcal{L}(y(t,\boldsymbol{t}_{\bold{G}}),f(\Phi(\boldsymbol{u}),t))$. 
The expected losses of factual and  counterfactual outcomes are, respectively:
\begin{equation}
\begin{aligned}
&{\epsilon}_{{\mathrm{F}}}{(f,\Phi)}\coloneqq {\int}_{\mathcal{U}\times\{0,1\}}\mathcal{L}_{f,\Phi}(\boldsymbol{u},t)p(\boldsymbol{u},t)d\boldsymbol{u}dt,\\& 
    {\epsilon}_{{\mathrm{CF}}}{(f,\Phi)}\coloneqq {\int}_{\mathcal{U}\times\{0,1\}}\mathcal{L}_{f,\Phi}(\boldsymbol{u},1-t)p(\boldsymbol{u},t)d\boldsymbol{u}dt.
\end{aligned}
\end{equation}
We can decompose $p(\boldsymbol{u},t)=p(t)p(\boldsymbol{u}\mid t)$. 
Let $p^{t=1}(\boldsymbol{u})=p(\boldsymbol{u}\mid t=1)$ and $p^{t=0}(\boldsymbol{u})=p(\boldsymbol{u}\mid t=0)$.
Then, the factual and  counterfactual outcomes of the treated and control groups are, respectively:
\begin{equation}
\begin{aligned}
    \epsilon_{\text{F}}^{t=1}(f,\Phi)\coloneqq \int_{\mathcal{U}}L_{f,\Phi}(\boldsymbol{u},1)p^{t=1}(\boldsymbol{u})d\boldsymbol{u},\\
    \epsilon_{\text{F}}^{t=0}(f,\Phi)\coloneqq \int_{\mathcal{U}}L_{f,\Phi}(\boldsymbol{u},0)p^{t=0}(\boldsymbol{u})d\boldsymbol{u},\\
    \epsilon_{\text{CF}}^{t=1}(f,\Phi)\coloneqq \int_{\mathcal{U}}L_{f,\Phi}(\boldsymbol{u},1)p^{t=0}(\boldsymbol{u})d\boldsymbol{u},\\
    \epsilon_{\text{CF}}^{t=0}(f,\Phi)\coloneqq \int_{\mathcal{U}}L_{f,\Phi}(\boldsymbol{u},0)p^{t=1}(\boldsymbol{u})d\boldsymbol{u}.
\end{aligned}
\end{equation}
Let $p_t\coloneqq p(t=1)$. We then have the following results:
\begin{equation}\label{ef:ecf1}
\begin{aligned}
    \epsilon_{\text{F}}(f,\Phi)=p_t\cdot\epsilon_{\text{F}}^{t=1}(f,\Phi)+(1-p_t)\cdot\epsilon_{\text{F}}^{t=0}(f,\Phi),\\
    \epsilon_{\text{CF}}(f,\Phi)=(1-p_t)\cdot\epsilon_{\text{CF}}^{t=1}(f,\Phi)+p_t\cdot\epsilon_{\text{CF}}^{t=0}(f,\Phi). 
\end{aligned}
\end{equation}
Integral probability metric (IPM) is defined as follows~\cite{pmlr-v70-shalit17a}:
\begin{definition}
Let $\bold{O}$ be a function family consisting of functions $o: \mathcal{R} \rightarrow \mathbb{R}$. For a pair of distributions $p_1$ and $p_2$ over $\mathcal{R}$, define IPM:
    \begin{equation}
    {\mathrm{IPM}}_{\bold{O}}(p_1,p_2)\coloneqq {\mathrm{sup}}_{o\in \bold{O}} \Biggr\vert \mathlarger{\int}_{\mathcal{R}}o(\boldsymbol{r})\biggl(p_1(\boldsymbol{r})-p_2(\boldsymbol{r})\biggr)d\boldsymbol{r}\Biggr\vert.
\end{equation}
\end{definition}
When $\bold{O}$ is the family of 1-Lipschitz functions, $\text{IPM}_{\bold{O}}(p_1,p_2) = \mathcal{W}(p_1,p_2)$, as demonstrated by~\cite{villani2008optimal}.
We use Wasserstein discrepancy for representation balancing.
For the Wasserstein discrepancy, we follow a Kantorovich problem~\cite{kantorovich2006translocation} and  add an entropic regularization to reduce computational costs~\cite{wang2023optimal}, as follows:
\begin{definition}\label{def:wass}
\begin{equation}
\begin{aligned}
   & \mathcal{W}(p_1,p_2)\coloneqq \langle\boldsymbol{D},\boldsymbol{\pi}^{\xi}\rangle, \quad
     \boldsymbol{\pi}^{\xi}\coloneqq \mathop{\arg\min}_{\boldsymbol{\pi}\in \Pi(p_1,p_2)} 
     \langle\boldsymbol{D},\boldsymbol{\pi}\rangle - \xi H(\boldsymbol{\pi}), \\
   & \Pi(p_1,p_2)\coloneqq \{ \boldsymbol{\pi} \in \mathbb{R}_{+}^{n\times m} : \boldsymbol{\pi}\bold{1}_m = \bold{p}_1,\ \boldsymbol{\pi}^\top\bold{1}_n = \bold{p}_2 \}, \\
   & H(\boldsymbol{\pi})\coloneqq -\sum_{i,j} \boldsymbol{\pi}_{i,j}(\log(\boldsymbol{\pi}_{i,j}) - 1).
\end{aligned}
\end{equation}
\end{definition}
Here, $\mathcal{W}$ is Wasserstein discrepancy between $p_1$ and $p_2$, $\boldsymbol{D}$ consists of unit-wise distance between $p_1$ and $p_2$. $\mathcal{W}(p_1,p_2)$ can be solved by Sinkhorn algorithm~\cite{cuturi2013sinkhorn}.

By using $\mathcal{W}$ for representation balancing, we can have the error bound for the counterfactual outcome, as stated in Lemma~\ref{LemmaCF}. 
\begin{lemma}\label{LemmaCF}
$\bold{O}$ be a family of functions o: $\mathcal{R}\rightarrow\mathbb{R}$. Assume there exists a constant $B_{\Phi}>0$, such that for $t\in\{0,1\}$, the function $\frac{1}{B_{\Phi}}\cdot \mathcal{L}_{f,\Phi}\left(\Phi^{-1}(\boldsymbol{r}),t\right) \in \bold{O}$ holds. Then, the bound for counterfactual outcome is: 
\begin{align}\label{eq:18}  
    \epsilon_{\rm{CF}}(f,\Phi)\leq &(1-p_t)\cdot\epsilon_{\rm{F}}^{t=1}(f,\Phi)+p_t\cdot\epsilon_{\rm{F}}^{t=0}(f,\Phi)+\\ &\;B_{\Phi}\cdot\mathcal{W}\biggr(p^{t=1}_{\Phi}(\boldsymbol{r}),p^{t=0}_{\Phi}(\boldsymbol{r})\biggr),
\end{align}
 where $p^{t=1}_{\Phi}$ and $p^{t=0}_{\Phi}$ are distributions of representations $\boldsymbol{r}$ with $t=1$ and $t=0$, respectively.
\end{lemma}
 We provide the proof for Lemma~\ref{LemmaCF} as follows:
\begin{proof}
\begin{align}
    &\epsilon_{\text{CF}}(f,\Phi)-\epsilon_{\text{F}}(f,\Phi)\nonumber\\=&\;
    \epsilon_{\text{CF}}(f,\Phi)-\biggr( (1-p_t)\cdot\epsilon_{\text{F}}^{t=1}(f,\Phi)+p_t\cdot\epsilon_{\text{F}}^{t=0}(f,\Phi)\biggr) \nonumber\\
    =&\;\biggr((1-p_t)\cdot\epsilon_{\text{CF}}^{t=1}(f,\Phi)+p_t\cdot\epsilon_{\text{CF}}^{t=0}(f,\Phi)\biggr) -\nonumber\\ &\;\biggr( (1-p_t)\cdot\epsilon_{\text{F}}^{t=1}(f,\Phi)+p_t\cdot\epsilon_{\text{F}}^{t=0}(f,\Phi)\biggr)\nonumber\\
    =&\;(1-p_t)\cdot\biggr(\epsilon_{\text{CF}}^{t=1}(f,\Phi)-\epsilon_{\text{F}}^{t=1}(f,\Phi)\biggr)+\nonumber\\ &\;p_t\cdot\biggr(\epsilon_{\text{CF}}^{t=0}(f,\Phi)-\epsilon_{\text{F}}^{t=0}(f,\Phi)\biggr)\nonumber \\=&\;(1-p_t)\cdot\mathlarger{\int}_{\mathcal{U}}\mathcal{L}_{f,\Phi}(\boldsymbol{u},1)\biggr(p^{t=0}(\boldsymbol{u})-  p^{t=1}(\boldsymbol{u})\biggr)d\boldsymbol{u}+ \nonumber\\ &\;
p_t \cdot \mathlarger{\int}_{\mathcal{U}}\mathcal{L}_{f,\Phi}(\boldsymbol{u},0)\biggr(p^{t=1}(\boldsymbol{u})-p^{t=0}(\boldsymbol{u})\biggr)d\boldsymbol{u},
   \nonumber \\=&\;B_{\Phi}\cdot(1-p_t)\cdot\mathlarger{\int}_{\mathcal{R}}\frac{1}{B_{\Phi}}\mathcal{L}_{f,\Phi}\biggr(\Phi^{-1}(\boldsymbol{r}),1\biggr)\biggr(p_{\Phi}^{t=0}(\boldsymbol{r})- p_{\Phi}^{t=1}(\boldsymbol{r})\biggr)d\boldsymbol{r} + \nonumber \\ &
    B_{\Phi}\cdot p_t \cdot \mathlarger{\int}_{\mathcal{R}}\frac{1}{B_{\Phi}}\mathcal{L}_{f,\Phi}\biggr(\Phi^{-1}(\boldsymbol{r}),0\biggr)\biggr(p^{t=1}_{\Phi}(\boldsymbol{r})-p^{t=0}_{\Phi}(\boldsymbol{r})\biggr)d\boldsymbol{r}
    \label{eq:cf2}\\ \leq
    &\;B_{\Phi}\cdot(1-p_t)\cdot\sup_{o\in\bold{O}} \Biggr| \mathlarger{\int}_{\mathcal{R}}o(\boldsymbol{r})\biggl(p^{t=0}_{\Phi}(\boldsymbol{r})- p^{t=1}_{\Phi}(\boldsymbol{r})\biggr)d\boldsymbol{r}\Biggr| + \nonumber\\& B_{\Phi}\cdot p_t\cdot\sup_{o\in\bold{O}} \cdot \Biggr| \mathlarger{\int}_{\mathcal{R}}o(\boldsymbol{r})\biggl(p^{t=1}_{\Phi}(\boldsymbol{r})- p^{t=0}_{\Phi}(\boldsymbol{r})\biggr)d\boldsymbol{r}\Biggr|
    \label{eq:cf3} \\ 
    =&\; B_{\Phi} \cdot{\mathrm{IPM}}_{\bold{O}}\biggl(p^{t=0}_{\Phi}(\boldsymbol{r}),p_{\Phi}^{t=1}(\boldsymbol{r})\biggr). \label{eq:cf4}
\end{align}
Here, inequality~(\ref{eq:cf2}) and equality~(\ref{eq:cf3}) are from the definition of IPM. Then, by applying $\mathcal{W}$ for IPM, we can have:
\begin{equation}
    (\ref{eq:cf4}) = B_{\Phi} \cdot \mathcal{W}\biggr(p^{t=0}_{\Phi}(\boldsymbol{r}),p^{t=1}_{\Phi}(\boldsymbol{r})\biggr).
\end{equation}
\end{proof}
Let $\hat{\tau}$ denote the proposed individual treatment effect estimator and  $\tau$ denote the true treatment effect estimator, we have the following definitions.
\begin{definition}
Let $f_1\left(\Phi(\boldsymbol{u})\right)=f\left(\Phi(\boldsymbol{u}),1\right)$ and $f_0\left(\Phi(\boldsymbol{u})\right)=f\left(\Phi(\boldsymbol{u}),0\right)$. The individual treatment estimator for an individual on graph data is: 
    \begin{equation}
        \hat{\tau} = f_1\biggr(\Phi(\boldsymbol{u})\biggr)-f_0\biggr(\Phi(\boldsymbol{u})\biggr).
    \end{equation}
\end{definition}

\begin{definition}
The expected precision in estimation of heterogeneous effect (PEHE) is: 
\begin{equation}
\epsilon_{{\mathrm{PEHE}}}(f) = \int_{\mathcal{X}\times\mathcal{X}_{\bold{G}}}(\hat{\tau}-\tau)^2p(\boldsymbol{x},\boldsymbol{x}_{\bold{G}})d\boldsymbol{x}d\boldsymbol{x}_{\bold{G}}.
\end{equation}
\end{definition}
 The error bound for PEHE is stated in Theorem~\ref{theorem:pehe}.
\begin{theorem} \label{theorem:pehe} $\bold{O}$ be a family of functions o: $\mathcal{R}\rightarrow\mathbb{R}$. Assume there exists a constant $B_{\Phi}>0$, such that for $t\in\{0,1\}$, the function $\frac{1}{B_{\Phi}}\cdot \mathcal{L}_{f,\Phi}(\Phi^{-1}(\boldsymbol{u}),t) \in \bold{O}$ holds. Then, we can have:
\vspace{-10pt}
\begin{align}
\epsilon_{{\mathrm{PEHE}}}(f) \leq 2\Biggr(&\epsilon_{\rm{F}}^{t=0}(f,\Phi)+\epsilon_{\rm{F}}^{t=1}(f,\Phi)+B_{\Phi}\cdot\mathcal{W}\biggl(p^{t=0}_{\Phi}(\boldsymbol{u}),p^{t=1}_{\Phi}(\boldsymbol{u})\biggr)\Biggr).    
\end{align}
\end{theorem}

\begin{proof}
    Let $m_t(\boldsymbol{x},\boldsymbol{x}_{\bold{G}})\coloneqq \mathbb{E}[y(t,\boldsymbol{t}_{\bold{G}})\mid \boldsymbol{x},\boldsymbol{x}_{\bold{G}}]$ for $t\in\{1,0\}$, and $\tau=m_1(\boldsymbol{x},\boldsymbol{x}_{\bold{G}})-m_0(\boldsymbol{x},\boldsymbol{x}_{\bold{G}})$. We have:
    \begin{align}
        &\;\epsilon_{\mathrm{PEHE}}(f) \nonumber\\ =&\mathlarger{\int}_{\mathcal{X}\times\mathcal{X}_{\bold{G}}}(\hat{\tau}-\tau)^2p(\boldsymbol{x},\boldsymbol{x}_{\bold{G}})d\boldsymbol{x}d\boldsymbol{x}_{\bold{G}}\nonumber\\=& \mathlarger{\int}_{\mathcal{X}\times\mathcal{X}_{\bold{G}}}\Biggr(\biggl[f_1\biggl(\Phi(\boldsymbol{u})\biggr)-f_0\biggl(\Phi(\boldsymbol{u})\biggr)\biggr]- \nonumber\\&\quad\quad\quad\;\;\biggl[m_1(\boldsymbol{x},\boldsymbol{x}_{\bold{G}})- m_0(\boldsymbol{x},\boldsymbol{x}_{\bold{G}})\biggr]\Biggr)^2p(\boldsymbol{x},\boldsymbol{x}_{\bold{G}})d\boldsymbol{x}d\boldsymbol{x}_{\bold{G}}\nonumber\\
        =&\mathlarger{\int}_{\mathcal{X}\times\mathcal{X}_{\bold{G}}}\Biggr(\biggl[f_1\biggl(\Phi(\boldsymbol{u})\biggr)-m_1(\boldsymbol{x},\boldsymbol{x}_{\bold{G}})\biggr]+\nonumber\\&\quad\quad\quad\;\;\biggl[m_0(\boldsymbol{x},\boldsymbol{x}_{\bold{G}})-f_0\biggl(\Phi(\boldsymbol{u})\biggr)\biggr]\Biggr)^2p(\boldsymbol{x},\boldsymbol{x}_{\bold{G}})d\boldsymbol{x}d\boldsymbol{x}_{\bold{G}}.\label{eq:pehe1}
    \end{align}
    Based on the inequality $(a+b)^2\leq 2(a^2+b^2)$, we can have:
    \begin{equation}
    \begin{aligned}
        (\ref{eq:pehe1}) \leq & 2\mathlarger{\int}_{\mathcal{X}\times\mathcal{X}_{\bold{G}}}\Biggr(\biggl[f_1\biggl(\Phi(\boldsymbol{u})\biggr)-m_1(\boldsymbol{x},\boldsymbol{x}_{\bold{G}})\biggr]^2+\\&\quad\quad\quad\;\;\biggl[m_0(\boldsymbol{x},\boldsymbol{x}_{\bold{G}})-f_0\biggl(\Phi(\boldsymbol{u})\biggr)\biggr]^2\Biggr)p(\boldsymbol{x},\boldsymbol{x}_{\bold{G}})d\boldsymbol{x}d\boldsymbol{x}_{\bold{G}}
        \label{eq:pehe2}
    \end{aligned}
    \end{equation}
By using $p(\boldsymbol{x},\boldsymbol{x}_{\bold{G}})=\int_{\mathcal{T}_{\bold{G}}} p(\boldsymbol{x},\boldsymbol{x},\boldsymbol{t}_{\bold{G}})d\boldsymbol{t}_{\bold{G}}$, which can further be decompose with $t=1$ and $t=0$, i.e,  $\int_{\mathcal{T}_{\bold{G}}} p(\boldsymbol{x},\boldsymbol{x},\boldsymbol{t}_{\bold{G}},t=1)d\boldsymbol{t}_{\bold{G}} + \int_{\mathcal{T}_{\bold{G}}} p(\boldsymbol{x},\boldsymbol{x},\boldsymbol{t}_{\bold{G}},t=0)d\boldsymbol{t}_{\bold{G}}$. Then, replacing $(\boldsymbol{x},\boldsymbol{x}_{\bold{G}},\boldsymbol{t}_{\bold{G}})$ with $\boldsymbol{u}$, we can have:
    \begin{align}
       (\ref{eq:pehe2})=
        2\Biggr(&\mathlarger{\int}_{\mathcal{U}}\biggl[f_1\biggl(\Phi(\boldsymbol{u})\biggr)-m_1(\boldsymbol{x},\boldsymbol{x}_{\bold{G}})\biggr]^2p(\boldsymbol{u},t=1)d\boldsymbol{u}
    +\nonumber\\&\mathlarger{\int}_{\mathcal{U}}\biggl[m_0(\boldsymbol{x},\boldsymbol{x}_{\bold{G}})-f_0\biggl(\Phi(\boldsymbol{u})\biggr)\biggl]^2p(\boldsymbol{u},t=0)d\boldsymbol{u} +\nonumber\\&
        \mathlarger{\int}_{\mathcal{U}}\biggl[f_1\biggl(\Phi(\boldsymbol{u})\biggr)-m_1(\boldsymbol{x},\boldsymbol{x}_{\bold{G}})\biggr]^2p(\boldsymbol{u},t=0)d\boldsymbol{u}
+\nonumber\\&\mathlarger{\int}_{\mathcal{U}}\biggl[m_0(\boldsymbol{x},\boldsymbol{x}_{\bold{G}})-f_0\biggl(\Phi(\boldsymbol{u})\biggr)\biggr]^2p(\boldsymbol{u},t=1)d\boldsymbol{u}\Biggr)
    \nonumber\\=2\Biggr(&\mathlarger{\int}_{\mathcal{U}\times \{0,1\}}\biggl[f_t\biggl(\Phi(\boldsymbol{u})\biggr)-m_t(\boldsymbol{x},\boldsymbol{x}_{\bold{G}})\biggr]^2p(\boldsymbol{u},t)d\boldsymbol{u}dt+\nonumber\\&\mathlarger{\int}_{\mathcal{U}\times\{0,1\}}\biggl[m_t(\boldsymbol{x},\boldsymbol{x}_{\bold{G}})-f_t\biggl(\Phi(\boldsymbol{u})\biggr)\biggr]^2p(\boldsymbol{u},1-t)d\boldsymbol{u}dt\Biggr) \label{eq:pehe3}\\
        =  2\Biggr(&\biggl(\epsilon_{\mathrm{F}}(f,\Phi)-\sigma_y\biggr)+\biggl(\epsilon_{\mathrm{CF}}(f,\Phi)-\sigma_y\biggr)\Biggr)\nonumber\\=2\Biggr(&\epsilon_{\mathrm{F}}(f,\Phi)+\epsilon_{\mathrm{CF}}(f,\Phi)-2\sigma_y\Biggr),\label{eq:pehe4}
    \end{align}
    where equality~(\ref{eq:pehe3}) holds as:
    \vspace{-10pt}
    \begin{equation}
    \begin{aligned}
    &\epsilon_{\rm{F}}{(f,\Phi)} = \mathlarger{\mathlarger{\int}}_{\mathcal{U}\times \{0,1\}}\Biggr(f_t\biggr(\Phi(\boldsymbol{u})\biggr)-m_t(\boldsymbol{x},\boldsymbol{x}_{\bold{G}})\Biggr)^2p(\boldsymbol{u},t)d\boldsymbol{u}dt + \sigma_y, \\ & \epsilon_{\rm{CF}}{(f,\Phi)} = \mathlarger{\mathlarger{\int}}_{\mathcal{U}\times \{0,1\}}\Biggl(f_t\biggr(\Phi(\boldsymbol{u})\biggr)-m_t(\boldsymbol{x},\boldsymbol{x}_{\bold{G}})\Biggr)^2p(\boldsymbol{u},1-t)d\boldsymbol{u}dt + \sigma_y,\\&
    \sigma_y = {\int}_{\mathcal{U}\times\{0,1\}\times \mathcal{Y}}\Biggl(m_t(\boldsymbol{x},\boldsymbol{x}_{\bold{G}})-y(t,\boldsymbol{t}_{\bold{G}})\Biggr)^2\cdot 
\\&\quad \quad \quad \quad \quad \quad \quad\quad \quad p\left(y(t,\boldsymbol{t}_{\bold{G}})\mid\boldsymbol{x},\boldsymbol{x}_{\bold{G}}\right)p(\boldsymbol{u},t)dy(t,\boldsymbol{t}_{\bold{G}})d\boldsymbol{u}dt\nonumber,
    \end{aligned}
    \end{equation}
    where $\sigma_y$ will be $0$ when  $y(t,\boldsymbol{t}_{\bold{G}})$ are deterministic functions of $\boldsymbol{u}$ and $t$. Let $\sigma_y$ be zero for simplicity. 
    We provide the proof for $\epsilon_{\rm{F}}(f,\Phi)$. The proof for $\epsilon_{\rm{CF}}(f,\Phi)$ is similar. We derive for $\epsilon_{\rm{F}}(f,\Phi)$ as follows:
\begin{align}
&\quad\;{\epsilon}_{{\mathrm{F}}}{(f,\Phi)}\nonumber\\&=\mathlarger{\int}_{\mathcal{U}\times\{0,1\}}\mathcal{L}_{f,\Phi}(\boldsymbol{u},t)p(\boldsymbol{u},t)d\boldsymbol{u}dt,\nonumber
\\&=\mathlarger{\mathlarger{\int}}_{\mathcal{U}\times\{0,1\}\times \mathcal{Y}}\Biggl(f_t\biggr(\Phi(\boldsymbol{u})\biggr)-y(t,\boldsymbol{t}_{\bold{G}})\Biggr)^2\cdot\nonumber\\&\quad \quad \quad \quad \quad \quad \quad
p\biggl(y(t,\boldsymbol{t}_{\bold{G}})\mid\boldsymbol{x},\boldsymbol{x}_{\bold{G}}\biggr)p(\boldsymbol{u},t)dy(t,\boldsymbol{t}_{\bold{G}})d\boldsymbol{u}dt,\nonumber
\\&= \mathlarger{\mathlarger{\int}}_{\mathcal{U}\times\{0,1\}\times \mathcal{Y}}\Biggl[\Biggl(f_t\biggr(\Phi(\boldsymbol{u})\biggr)-m_t(\boldsymbol{x},\boldsymbol{x}_{\bold{G}})\Biggr)+\nonumber\\&\quad\quad\quad\quad\quad\quad\quad\biggl(m_t(\boldsymbol{x},\boldsymbol{x}_{\bold{G}})-y(t,\boldsymbol{t}_{\bold{G}})\biggr)\Biggr]^2 \cdot\nonumber\\&\quad \quad \quad \quad \quad \quad \quad
p\biggl(y(t,\boldsymbol{t}_{\bold{G}})\mid\boldsymbol{x},\boldsymbol{x}_{\bold{G}}\biggr)p(\boldsymbol{u},t)dy(t,\boldsymbol{t}_{\bold{G}})d\boldsymbol{u}dt,\nonumber\\&=
\mathlarger{\mathlarger{\int}}_{\mathcal{U}\times\{0,1\}\times \mathcal{Y}}\Biggl(f_t\biggr(\Phi(\boldsymbol{u})\biggr)-m_t(\boldsymbol{x},\boldsymbol{x}_{\bold{G}})\Biggr)^2 \cdot\nonumber\\&\quad \quad \quad \quad \quad \quad \quad
p\biggl(y(t,\boldsymbol{t}_{\bold{G}})\mid\boldsymbol{x},\boldsymbol{x}_{\bold{G}}\biggr)p(\boldsymbol{u},t)dy(t,\boldsymbol{t}_{\bold{G}})d\boldsymbol{u}dt +\nonumber\\&\quad
\;\,\mathlarger{\mathlarger{\int}}_{\mathcal{U}\times\{0,1\}\times \mathcal{Y}}\biggl(m_t(\boldsymbol{x},\boldsymbol{x}_{\bold{G}})-y(t,\boldsymbol{t}_{\bold{G}})\biggr)^2 \cdot\nonumber\\&\quad \quad \quad \quad \quad \quad \quad
p\biggl(y(t,\boldsymbol{t}_{\bold{G}})\mid\boldsymbol{x},\boldsymbol{x}_{\bold{G}}\biggr)p(\boldsymbol{u},t)dy(t,\boldsymbol{t}_{\bold{G}})d\boldsymbol{u}dt+ \nonumber\\ &\quad
2\mathlarger{\mathlarger{\int}}_{\mathcal{U}\times\{0,1\}\times \mathcal{Y}}\biggl(m_t(\boldsymbol{x},\boldsymbol{x}_{\bold{G}})-y(t,\boldsymbol{t}_{\bold{G}})\biggr)\cdot\nonumber\\&\quad\quad\quad\quad\quad\quad\;\;\,\Biggl(f_t\biggr(\Phi(\boldsymbol{u})\biggr)-m_t(\boldsymbol{x},\boldsymbol{x}_{\bold{G}})\Biggr) \cdot\nonumber\\&\quad \quad \quad \quad \quad \quad \quad
p\biggl(y(t,\boldsymbol{t}_{\bold{G}})\mid\boldsymbol{x},\boldsymbol{x}_{\bold{G}}\biggr)p(\boldsymbol{u},t)dy(t,\boldsymbol{t}_{\bold{G}})d\boldsymbol{u}dt,\label{eq:var1}
\\&=\; \mathlarger{\mathlarger{\int}}_{\mathcal{U}\times\{0,1\}}\Biggl(f_t\biggr(\Phi(\boldsymbol{u})\biggr)-m_t(\boldsymbol{x},\boldsymbol{x}_{\bold{G}})\Biggr)^2
p(\boldsymbol{u},t)d\boldsymbol{u}dt + \sigma_y + 0, \nonumber
\end{align}
where this equality holds because the final integral in equality~(\ref{eq:var1}) is zero due to the definition of $m_t(\boldsymbol{x},\boldsymbol{x}_{\bold{G}})$. 
     Based on this equality, equality~(\ref{ef:ecf1}), and Lemma~\ref{LemmaCF}, we can have:
    \begin{align}
        (\ref{eq:pehe4}) \leq 2\Biggr(&\epsilon_{\rm{F}}^{t=0}(f,\Phi)+\epsilon_{\rm{F}}^{t=1}(f,\Phi)+B_{\Phi}\cdot\mathcal{W}\biggl(p^{t=0}_{\Phi}(\boldsymbol{r}),p^{t=1}_{\Phi}(\boldsymbol{r})\biggr)\Biggr).\nonumber
    \end{align}
    
\end{proof}
This tells if there are confounding and interference biases,
we can minimize the discrepancy $\mathcal{W}$ in joint representation $\boldsymbol{r}$ between treated and control groups to mitigate the confoundering and interference biases.
Here, the discrepancy $\mathcal{W}$ is differentiable with respect to the map function $\Phi$~\cite{wang2023optimal}. Thus, it can be minimized by updating the map function $\Phi$. In our approach, this is implemented by minimizing the loss function, i.e., Equation~(\ref{eq:maintotalloss}).

\section{Time complexity for the NIM layer}\label{app:timecomplex}
In this section, we focus on analyzing the complexity of the NIM layer, which is the main source of time complexity. Let 
 $c_d$ denotes input and output dimension of each NIM layer and $L$ be the layer depth for every sub-networks for simplicity.  To simplify the results, our analysis is based on the straightforward runtime for performing matrix multiplications and remove constant factors for big-o runtimes. For the structure encoder of each individual, the time complexity is $O(|\bold{N}_i|\cdot L \cdot (c_d)^2)$, where $|\bold{N}_i|$ is the size of the set of neighbors of the individual $i$. For each partial attention mechanism of each individual, the time complexity is $O(|\bold{N}_i|\cdot L \cdot (c_d)^2 + |\bold{N}_i|\cdot c_d)$. By combining them, we can have the time complexity of  NIM layers is $O(|\bold{N}_i|\cdot L \cdot (c_d)^2 + |\bold{N}_i|\cdot c_d)$ per individual.

 By applying the neighbor sampling mechanism~\cite{hamilton2017inductive}, the time complexity for each partial attention mechanism can be reduced to $O(k \cdot L \cdot (c_d)^2 + k \cdot c_d)$, where $k$ is the size of sampled neighbors ($|\bold{N}_i| \gg k $). For the structure encoder, we can use the pre-computation technologies in \cite{LIN2024SITE}, which pre-computes the summary information and saved the it for training MLPs.  In this case, the time complexity for the structure encoder can be reduced to $O( L \cdot (c_d)^2)$. As a result, the time complexity of  NIM layers can be reduced to $O(k \cdot L \cdot (c_d)^2 + k \cdot c_d + L \cdot (c_d)^2)$ per individual.  However, such techniques often involve a trade-off between efficiency and performance. Therefore, accelerating the implementation of the NIM layer while preserving the performance of ITE estimation with DNE can be a promising future research direction.

\section{Theoretical guarantee for PFOR}
In this section, we discuss the necessary assumption for PFOR. We extend a monotonicity assumption from \citet{wang2023optimal}.
\begin{assumption}
    For all observed variable $U=\boldsymbol{u}$ in the population of interest, we have $\mathbb{E}[Y|U=\boldsymbol{u},V=\boldsymbol{v},T=t]$ is monotonically increasing or decreasing with respect to $\boldsymbol{v}$.
\end{assumption}
When this assumption holds and we have balanced joint representation $\boldsymbol{r}$, and identical treatment $t$, the only variable reflecting the variation of $\boldsymbol{v}$ is the outcome. In this case, we can calibrate distance using outcome differences rather than unobserved $\boldsymbol{v}$. Since unobserved $\boldsymbol{v}$
 affect outcomes under similar observed $\boldsymbol{u}$
 and treatments, outcome differences serve as reasonable calibration. PFOR fails to handle variables that add constant effects to all units. However, in real scenarios, it is rare that different values of 
 only add a constant effect to the outcome~\cite{wang2023optimal}, making PFOR still effective in a wide range of application scenarios~\cite{wang2023optimal}.

\section{Additional related work}\label{app:relatedwork}
\noindent\textbf{Treatment effect estimation from observational data without interference.}
The potential outcome framework underlies many existing treatment effect estimators for non-graph data~\cite{rubin1980randomization,rubin2005causal}. Most existing methods with this context focus on mitigating confounding bias by minimizing the discrepancy between different treatment groups~\cite{guo2023estimating,Johansson2016,johansson2022generalization,pmlr-v70-shalit17a,wang2023optimal,wen2023variational,yao2018representation}. Several studies consider estimating treatment effect with a budget constraint~\cite{jesson2021causal,qin2021budgeted,wen2025enhancing,Wen2025kdd}, unobserved confounders~\cite{frauen2022estimating,jesson2021quantifying,oprescu2023b,wang2022estimating,wu2022instrumental}, or other complex situation~\cite{frauen2024model,kaddour2021causal,melnychuk2023bounds,shi2019adapting,wu2022learning}. 
These methods generally assume no interference among individuals~\cite{rubin1980randomization,rubin2005causal}. Although this assumption is reasonable for many scenarios, it does not always hold in real-world data, such as graph data, where individuals are typically connected and exchange information.   Other studies explore estimating treatment effects from graph data by using GNN~\cite{2017gat,kipf2016semi} to capture networked confounders~\cite{chu2021graph,cui2024treatment,guo2020learning,guo2021ignite,ma2021deconfounding,thorat2023estimation,veitch2019using}, model high-dimensional treatment~\cite{harada2021graphite}, or identify unreliable estimation~\cite{wen2023predict}, but still do not take interference into account. For a comprehensive review on the topic of treatment effect estimation from observational data without interference, refer to recent survey literatures~\cite{kuang2020causal,yao2021survey}.

\vspace{1em}
\noindent\textbf{Treatment effect estimation from experimental data with interference.}
Several studies focused on treatment effect estimation from experimental data with interference~\cite{Aronow2017EstimatingAC,awan2020almost,basse2018analyzing,doi:10.1198/016214508000000292,liu2014large,rosenbaum2007interference,tchetgen2012causal,toulis2013estimation}. To collect experimental data, these studies first conducted experiments based on random treatment assignment strategies that they designed. Next, they estimated treatment effects from the collected experimental data. 
Although using experimental data is an ideal choice, which provides the most reliable evidence and serve as the gold standard in treatment effect estimation, the cost of experimental data collection is expensive and time-consuming. In contrast, estimating treatment effect from observational data provides a low-cost alternative. 

\vspace{1em}
\noindent\textbf{Comparison of different methods for interference modeling.} We compare different methods for interference modeling in Table~\ref{tabcom:baselines}.

\begin{table*}
\centering
\caption{Comparison of different methods. ``\cmark" or ``-" indicate whether the corresponding issue is considered or not.
``Neighbor" represents neighbor interference, whereas ``Networked" represents networked interference.
DNE consists of two key sub-issues. (I) The importance of
different neighbors in contributing to interference varies.  (II)
The scale of neighbors varies, leading to different levels of interference. 
}
\begin{tabular}{lcccccccc}
\hline
  Method& Confounder & Unobserved Confounder & Neighbor & Networked & Sub-issue (I) of DNE & Sub-issue (II) of DNE\\ \hline
  \hline
TARNet~\cite{pmlr-v70-shalit17a}   & - & - & - & - & - & -    
\\
BNN~\cite{Johansson2016}  & \cmark & - & - & - & -   & -
\\
CFR~\cite{pmlr-v70-shalit17a}  & \cmark & - & - & - & - & - 
\\
ESCFR~\cite{wang2023optimal}  & \cmark & \cmark & - & - & - & - 
\\
RERUM~\cite{he2024rankability}   & \cmark & - & - & - & - & - 
\\
NetDeconf~\cite{guo2020learning} & \cmark & \cmark & - & - & - & - 
\\
NetEst~\cite{jiang2022estimating} & \cmark & - & \cmark & - & - & -
\\
SPNet~\cite{huang2023modeling} & \cmark & - & \cmark & -& \cmark & - 
\\
GCN-HSIC~\cite{pmlr-v130-ma21c}
& \cmark & - & \cmark & \cmark & - & -
\\
SAGE-HSIC~\cite{pmlr-v130-ma21c} & \cmark & - & \cmark & \cmark & - & -
\\
SITE~\cite{LIN2024SITE} & \cmark & - & \cmark & \cmark & - & -
\\
HyperSCI~\cite{ma2022learning}
& \cmark & - & \cmark & \cmark & \cmark & - 
\\
HINITE~\cite{lin2023estimating}
& \cmark & - & \cmark & \cmark &  \cmark & - 
\\
DWR~\cite{zhao2024learning} & \cmark & - & \cmark & - & \cmark & -
\\
IDE-NET~\cite{adhikari2025inferring} & \cmark & - & \cmark & \cmark & \cmark & -
\\
CauGramer~\cite{wucausal} & \cmark & \cmark & \cmark & -  & \cmark & - 
\\
\hline
\hline
 GITE (proposed) & \cmark & \cmark & \cmark & \cmark & \cmark  & \cmark 
 \\
\hline
\end{tabular}
\label{tabcom:baselines}
\end{table*}

\section{Details of dataset description}\label{app:datasets}
We describe the details of each dataset in the section.

\vspace{1em}
\noindent\textbf{Flickr dataset~\cite{wang2013learning}:} 
Flickr is an online social website, where users share their images. 
The Flickr dataset~\cite{wang2013learning} is collected from this website.
In this dataset, each individual is a user of Flickr.
There are 7,575 individuals with 479,476 directed edges.
Here, we aim to estimate how much recommending a hot photo (treatment) to a user affects the experience of the user  (outcome) of this photo.
In this case, users may share recommended photos with their friends (related individuals), which constitutes networked interference.
We used the 1,206-dimensional embeddings of user profiles as covariates that were provided by \citet{guo2020learning}.

\vspace{1em}
\noindent\textbf{BlogCatalog dataset~\cite{li2015unsupervised}:}
BlogCatalog is an online community, where users post their blogs.  
The BlogCatalog (abbreviated as Blog) dataset~\cite{li2015unsupervised,li2019adaptive} is collected from the online community. 
Every individual in this dataset is a user of BlogCatalog.
There are 5,196 individuals with 343,486 edges.
Here, we aim to estimate how much a recommended blog (treatment) to a user affects the experience of the user (outcome) of this blog.
In this case, users may share recommended blogs with their friends (related individuals), which constitutes networked interference. We used embeddings of each individual as covariates that were provided by~\citet{guo2020learning}.
 

\vspace{1em}
\noindent\textbf{Amazon negative dataset~\cite{he2016ups}:}  Amazon negative dataset (abbreviated as AMZ-N) was extracted from the Amazon dataset~\cite{he2016ups} by~\citet{rakesh2018linked}, to study the 
 effect of negative reviews on the sales of products and the issue of interference.
Every unit in the AMZ-N dataset is an item, and every edge indicates that the two items are always purchased together by customers.
In the AMZ-N dataset, there are 14,538 units with 15,011 directed edges.
The treatment $t \in \{0,1\}$ depends on the number of negative reviews: if a unit has more than three negative reviews ($t=1$) or if a unit has less than three negative reviews ($t=0$)~\cite{rakesh2018linked}.
The covariate $\boldsymbol{x}$ (with 300 features) of each unit is created by applying the doc2vec method~\cite{le2014distributed} to encode the review of the user.
We used covariates, treatments, outcomes, and ITE, all of which were provided by~\cite{rakesh2018linked}.
As values of $y$ fluctuate in a large range, we applied the z-score normalization to $y$ during the training and testing phases, following~\citet{LIN2024SITE}.

\section{Simulation}\label{App:simulation}
We simulated treatments and outcomes for the Flickr and Blog datasets.

\vspace{1em}
\noindent\textbf{Treatment simulation.} 
We generate treatments for the Flickr and Blog datasets as follows:
\begin{align}
t_i\sim \text{Ber}\biggl(\text{sigmoid}(0.5\cdot \boldsymbol{z}_i+0.5\cdot \boldsymbol{z}_{\bold{X}_i})+\epsilon_{t_i}\biggr).
\end{align}
$\boldsymbol{z}_i = \boldsymbol{w}_{t}^{\top}\boldsymbol{x}_i$ represents individual confounders.   $\boldsymbol{z}_{\bold{X}_i} = {\textrm{AGG}_{\bold{X}}}(\Tilde{\boldsymbol{x}}_{\bold{G}_i})$ represents networked   confounders from related individuals of $i$.   $\Tilde{\boldsymbol{x}}_{\bold{G}_i}=\{\Tilde{\boldsymbol{x}}_k\mid k\in\bold{G}_i\}$. Importantly, to properly simulate interference from related individuals of $i$ with DNE, we rescale values of covariates as $\Tilde{\boldsymbol{x}}_k = \boldsymbol{w}_{\bold{G}}^{\top}(\boldsymbol{x}_k+\tilde{\eta}_k\cdot \boldsymbol{x}_k$), where $\tilde{\eta}_k={{\log}(\tilde{d}_k)}/({\sum_{k=1}^{N}{{\log}(\tilde{d}_k)}})$, every element of $\boldsymbol{w}_t$ and $\boldsymbol{w}_{\bold{G}}$ was generated from a normal distribution $\mathcal{N}(0,1)$ or uniform distribution $\mathcal{U}(-1,1)$, and 
$\epsilon_{t_i}$ is a random noise generated from a normal distribution $\mathcal{N}(0,1)$.
We achieved ${\textrm{AGG}_{\bold{X}}}(\Tilde{\boldsymbol{x}}_{\bold{G}_i})$ by repeating summary computation of the one-hop neighbor information  $\sum_{k \in \bold{N}_i}e_{ij}\Tilde{\boldsymbol{x}}_k$ three times for every individual, where $e_{ij}$ was generated from uniform distribution $\mathcal{U}(0,1)$. 

\vspace{1em}
\noindent\textbf{Outcome simulation.} 
We generated outcomes for the Flickr and Blog datasets as follows:
\begin{equation}
    y_i = f_{\boldsymbol{x}}(\boldsymbol{x}_i) + f_{t}(t_i,\boldsymbol{x}_i) + f_{\bold{G}}(\boldsymbol{x}_{\bold{G}_i},t_{\bold{G}_i}) + \epsilon_{y_i}.
\end{equation}
 $f_{\boldsymbol{x}}(\boldsymbol{x}_i)=\boldsymbol{w}_{\boldsymbol{x}}^{\top}\boldsymbol{x}_i$ is the synthetic outcome of the individual $i$ without treatment effect and the effect from the related individuals, where every element of $\boldsymbol{w}_{\boldsymbol{x}}$ independently follows $\mathcal{N}(0,1)$. 
$f_{t}(t_i,\boldsymbol{x}_i)=t_i\cdot\boldsymbol{w}_1^{\top}\boldsymbol{x}_i$ synthesizes ITE,  where $\boldsymbol{w}_1$ also follows $\mathcal{N}(0,1)$.
$f_{\bold{G}}(\boldsymbol{x}_{\bold{G}_i},\boldsymbol{t}_{\bold{G}_i})=g_{\bold{X}}(\boldsymbol{x}_{\bold{G}_i}) + g_{\bold{T}}(t_{\bold{G}_i})=\textrm{AGG}_{\bold{X}}(\Tilde{\boldsymbol{x}}_{\bold{G}_i})+\textrm{AGG}_{\bold{T}}(\Tilde{\boldsymbol{t}}_{\bold{G}_i})$ synthesizes effect from related individuals of the individual $i$ on a graph, where 
$\Tilde{\boldsymbol{t}}_{\bold{G}_i}=\{\Tilde{t}_k\mid k \in\bold{G}_i\}$,
where $\Tilde{t}_k = t_k+\tilde{\eta}_k\cdot t_k$, similar to the rescaling operation for covariates. ${\textrm{AGG}_{\bold{T}}}(\Tilde{\boldsymbol{t}}_{\bold{G}_i})$ was achieved by repeating summary computation of the one-hop neighbor information  $\sum_{k \in \bold{N}_i}e_{ij}\Tilde{t}_k$ three times for every individual. Moreover, $\epsilon_{y_i}$ is a random noise generated from a normal distribution $\mathcal{N}(0,1)$.

\section{Details of baseline methods}\label{app:detailsbaselines}
Our methods were compared with several methods, which are divided into the following four categories:

\vspace{1em}
 \noindent  \textbf{ITE estimator for non-graph data.} Balancing neural network (BNN)~\citep{Johansson2016}, counterfactual regression with maximum mean discrepancy (CFR-MMD)~\citep{pmlr-v70-shalit17a}, and counterfactual regression with Wasserstein discrepancy (CFR-Wass)~\citep{pmlr-v70-shalit17a} address confounders by minimizing distribution discrepancies, maximum mean discrepancy (MMD), and Wasserstein discrepancy between control and treated groups, respectively.
     TARNet~\citep{pmlr-v70-shalit17a} has the same model architecture as the CFR but no measures for confounders. Entire space CFR (ESCFR)~\cite{wang2023optimal} have the same architecture and representation balancing technology as CFR-Wass, but modified the calculation of Wasserstein distance for mini-batch training. RERUM~\cite{he2024rankability} includes an outcome ranking loss and an ITE ranking loss to enhance the ranking ability of model. Following ~\citet{he2024rankability}, we implement it by combing the outcome ranking loss and the ITE ranking loss  with CFR-MMD.
     All of these methods ignore interference and networked confounders. 
     
     \vspace{1em}
 \noindent  \textbf{ITE estimator for graph data without addressing interference.} 
     Network deconfounder (NetDeconf)~\cite{guo2020learning} models networked confounders by GCN~\cite{kipf2016semi} without modeling interference. Moreover, it balances representations by Wasserstein discrepancy.

     \vspace{1em}
 \noindent  \textbf{ITE estimators for graph data with addressing neighbor interference.} 
     Networked causal effect estimation (NetEst)~\cite{jiang2022estimating} models neighbor confounders by a single GCN layer and neighbor interference by a mean aggregation. NetEst balance representations by an adversarial learning technology. SPNet~\cite{huang2023modeling} models networked confounders by GCN and neighbor interference by GAT~\cite{2017gat}. SPNet balances representations by minimizing the Wasserstein discrepancy between different treatment groups. DWR~\cite{zhao2024learning} models neighbor confounders and neighbor interference by GAT~\cite{2017gat}. CauGramer~\cite{wucausal} use GCN to model interference, while combining Transformer to  discover interfernce from unknown neighbors.

     \vspace{1em}
 \noindent  \textbf{ITE estimators for graph data with addressing networked interference.} GCN-HSIC~\cite{pmlr-v130-ma21c} models networked interference by GCN and balances representations by Hibelt-Schmidt independence criterion (HSIC)~\cite{10.1007/11564089_7}. SAGE-HSIC~\cite{pmlr-v130-ma21c} uses the same representation balancing technology as GCN-HSIC but replaces the GCN of GCN-HSIC with the mean aggregator of GraphSAGE~\cite{hamilton2017inductive}. Scalable individual treatment effect estimator~\cite{LIN2024SITE} (SITE) reduces computation in the aggregation of GCN-HSIC by a pre-aggregation technology for related individuals and balances representations by MMD regularization.  IDENet~\cite{adhikari2025inferring} combines ego MLP, GCN, and a similarity score to model networked interference.  
HyperSCI~\cite{ma2022learning} and HINITE~\cite{lin2023estimating} are designed for more complex graphs: hypergraph and heterogeneous graphs, respectively. HyperSCI and HINITE can be extended to estimate treatment effect from ordinary graphs. In this case, they use GAT~\cite{2017gat} to model networked interference. HyperSCI balances representations by applying Wasserstein discrepancy and HINITE balances representations by applying HSIC regularization.

\section{Compute resources}\label{app:Computeresources}
The compute resources used in our experiments are as follows:
\begin{itemize}
    \item  \textbf{Operating system}: Ubuntu 22.04 LTS.
    \item  \textbf{GPU}: H100 with 80GB GPU memory.
    \item  \textbf{CPU}: AMD EPYC 7343 (16C/32T, 3.2GHz, 128M Cache).
\end{itemize}
\section{Implementation details}\label{app:Implementation}

For the representation learning module, $\boldsymbol{z}^{(0)}_{\bold{S}_i}$ is initialized as a one-dimensional vector with an element value of 1, 
$\boldsymbol{z}^{(0)}_{\bold{X}_i}$ is initialized as $\boldsymbol{x}_i$, and 
$\boldsymbol{z}^{(0)}_{\bold{T}_i}$ is initialized as $t_i$. The importance estimation mechanism $a$ of IPAtt and SPAtt in Equation~(\ref{eq:importance}) is implemented by GAT~\cite{2017gat}  by default for our experiments.  We conducted experiments with the attention mechanism of Transformer~\cite{vaswani2017attention,NEURIPS2021_f1c15925} in Appendix~\ref{appexp:diffatt}.
Following~\citet{2017gat} and \citet{vaswani2017attention}, we use the softmax operation for the normalization operation in Equation~(\ref{eq:importance}). We apply the layer normalization~\cite{ba2016layer} for the proxy module.
We computed the weight $\mathring{\pi}_{\bullet}^{(l)}=\exp(\mathring{\pi}_{\bullet}^{(l)})/(\exp(\pi_{\bullet}^{(l)})+\exp(1-\pi_{\bullet}^{(l)}))$ for $\bullet\in\{\bold{X},\bold{T}\}$. 
For simplicity, we set $\pi_{\eta}=1$ of the message amplifier of GITE. We conducted experiments with a different strategy for $\pi_{\eta}$, i.e., setting $\pi_{\eta}$ as a learnable parameter in Appendix~\ref{appexp:diffMA}. 

Following \citet{kipf2016semi}, \citet{guo2020learning}, \citet{pmlr-v130-ma21c}, \citet{jiang2022estimating}, and \citet{LIN2024SITE}, we consider a transductive setting, i.e., graph structure $\boldsymbol{A}$, covariates $\boldsymbol{X}$, and treatments $\boldsymbol{T}$ were given during the training, validation, and testing phases; Whereas only observed outcomes of individuals in the training dataset were provided for all graph-based methods during training.  Importantly, the proposed methods are not limited to the transductive setting. If researchers want to consider an inductive setting, they can mask covariates of individuals and edges connecting to individuals that are in the validation and test sets on graph during the training phase.

We use the Adam optimizer~\cite{kingma2014adam} with a learning rate of 0.001 and a decay rate of 0.001 to train proposed methods and set the maximum training iterations to 2,000. 
We set the number of layers for the proxy component to 1 for GITE and 3 for GITE$_{\rm{v}}$, for all other components to 3.
We set the dimension of the proxy component to 100 or 300 for GITE and set the dimensions for the proxy component to $(100,200,300)$ for GITE$\rm{v}$. We set the dimension for layers of other components to 100.
We adopted ReLU for the activation function.
We searched for hyperparameters by checking the results on the validation set. Specifically, we searched $\lambda$ from the range of $\{0.001,0.01,0.1,0.2,0.5,1.0\}$, $\beta$ from the range of $\{0.001,0.01,0.1,0.2,0.5,1.0\}$, $\lambda_{D}$ from $\{0.1,0.5,1.0,5.0,10.0\}$. We searched  $\lambda_{P}$ from the range of $\{0.1,0.5,1.0,5.0,10.0\}$ for GITE$_{\rm{v}}$.  
Early stopping and dropout were applied to the proposed methods to avoid overfitting.  

We used default hyperparameters or searched hyperparameters from the ranges suggested in the literature to implement the baseline methods. We also applied early stopping and dropout to the baselines for all datasets to avoid overfitting.

\section{Additional experiments}

\subsection{Additional sensitivity experiments}\label{app:expsen}

We conducted sensitivity experiments with different values of $\lambda_D$ in the range $\{0.1,0.5,1.0,5.0,10.0\}$ for GITE and GITE$_{\rm{v}}$. Results are shown in Figure~\ref{app:addsens}. Results reveal that there is no significant performance change with different values of $\lambda_D$. 

We conducted sensitivity experiments with different values of $\lambda_P$ in the range $\{0.1,0.5,1.0,5.0,10.0\}$ for GITE$_{\rm{v}}$. Results are shown in Figure~\ref{app:addsens}. Results reveal that there is a performance degradation for GITE$_{\rm{v}}$ when setting a large value ($>5$) of $\lambda_P$, which suggests that searching the value of $\lambda_P$ in the range of $(0,5]$.

We conducted sensitivity experiments with different values of dimension for different layers of proxy module in the range $\{100, 300\}$ for GITE. Results are shown in Table~\ref{table:diffdimension}. Results show that both dimensions are reasonable for GITE.

\begin{figure*}[htbp]
	\centering
 \hspace{-10pt}   \subfloat[ AMZ-N,  $\lambda_D$,  $\sqrt{\epsilon_{\textrm{MSE}}}$.]{\includegraphics[width=.24\textwidth]{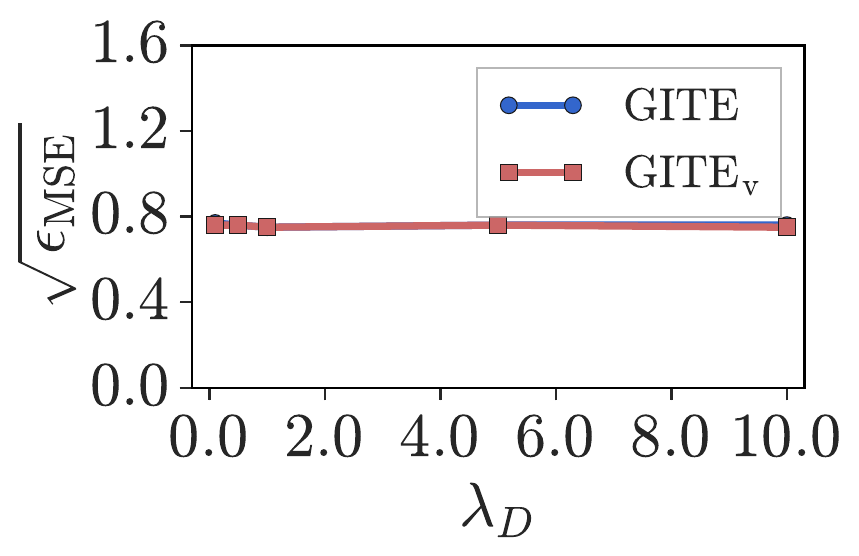}} \hspace{3pt}
	\subfloat[ AMZ-N, $\lambda_D$, $\sqrt{\epsilon_{\textrm{PEHE}}}$.]
 {\includegraphics[width=.24\textwidth]{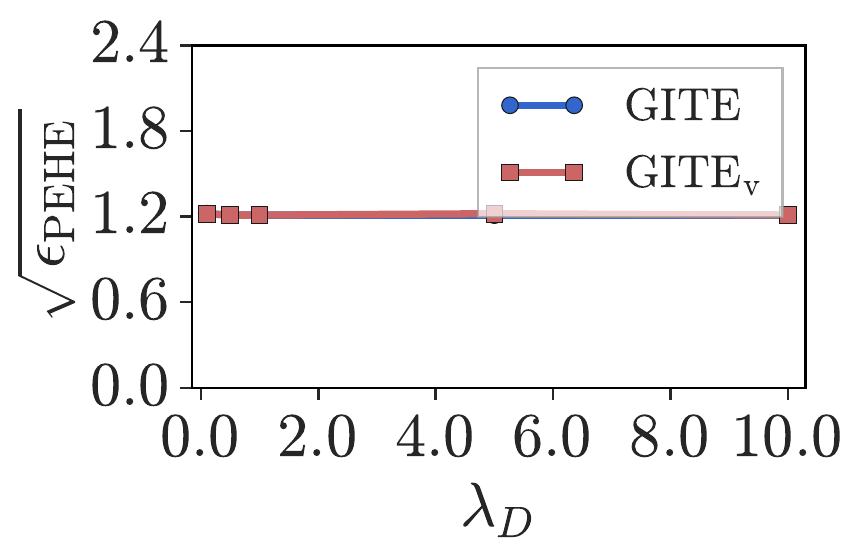}} \hspace{3pt}
 \subfloat[AMZ-N, $\lambda_P$, $\sqrt{\epsilon_{\textrm{MSE}}}$.]{\includegraphics[width=.24\textwidth]{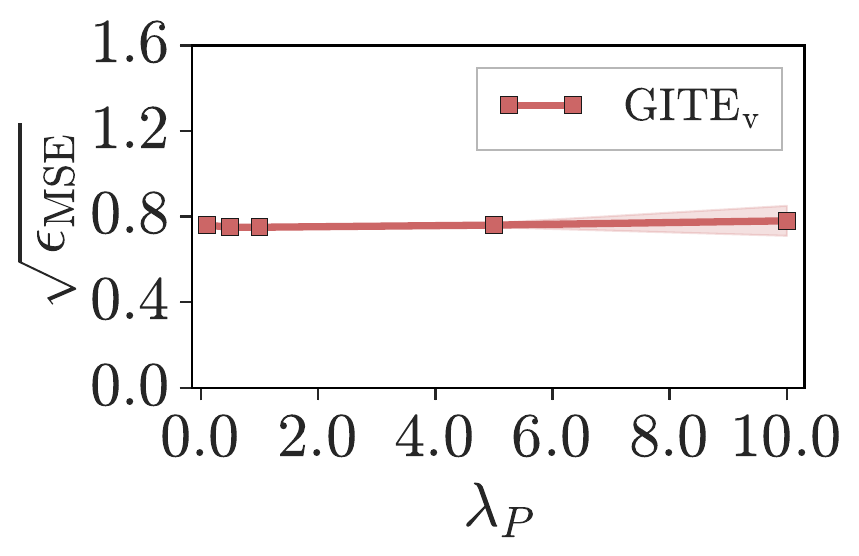}} \hspace{0pt}
	\subfloat[AMZ-N, $\lambda_P$, $\sqrt{\epsilon_{\textrm{PEHE}}}$.]
 {\includegraphics[width=.24\textwidth]{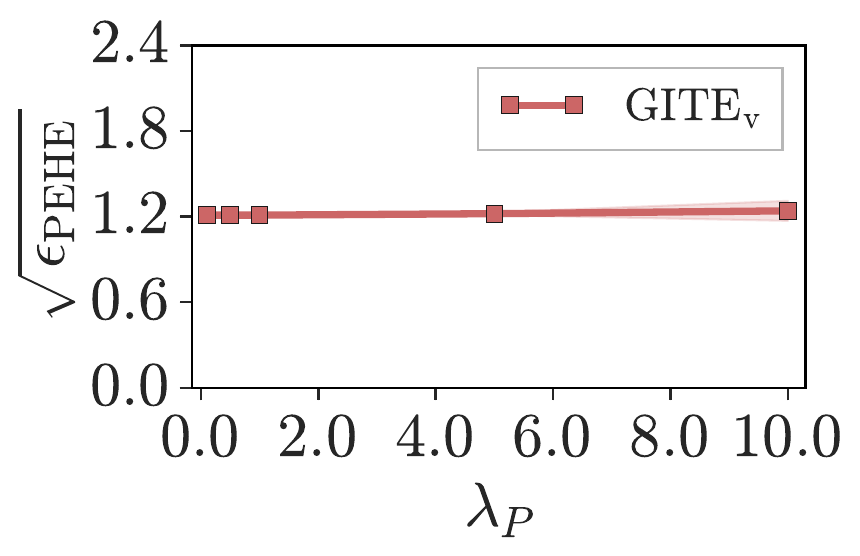}}\\
 \hspace{-8pt}\subfloat[ Flickr,  $\lambda_D$,  $\sqrt{\epsilon_{\textrm{MSE}}}$.]{\includegraphics[width=.24\textwidth]{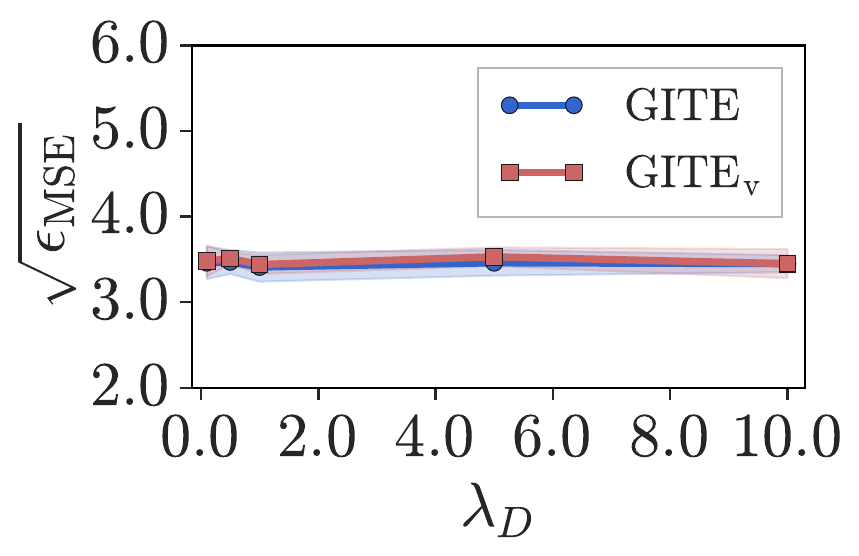}}\hspace{3pt}
	\subfloat[ Flickr, $\lambda_D$, $\sqrt{\epsilon_{\textrm{PEHE}}}$.]
 {\includegraphics[width=.24\textwidth]{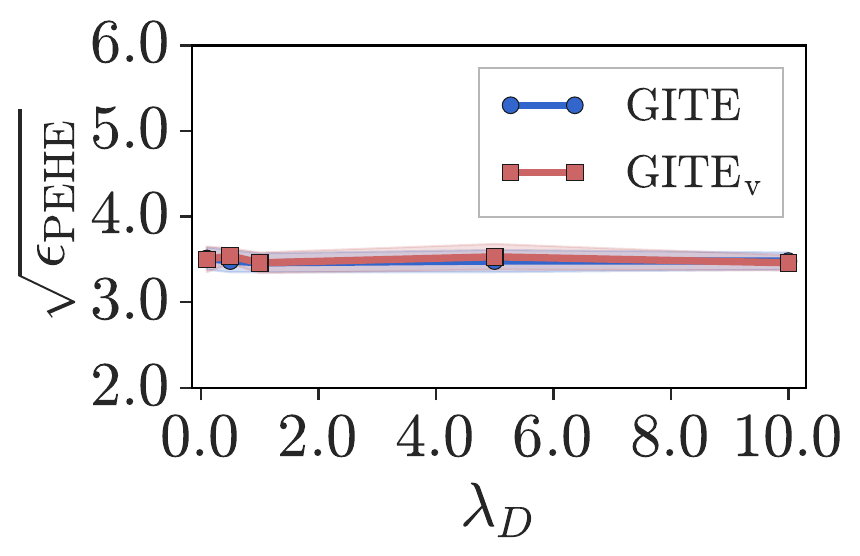}}\hspace{3pt}
 \subfloat[  Flickr, $\lambda_P$, $\sqrt{\epsilon_{\textrm{MSE}}}$.]{\includegraphics[width=.24\textwidth]{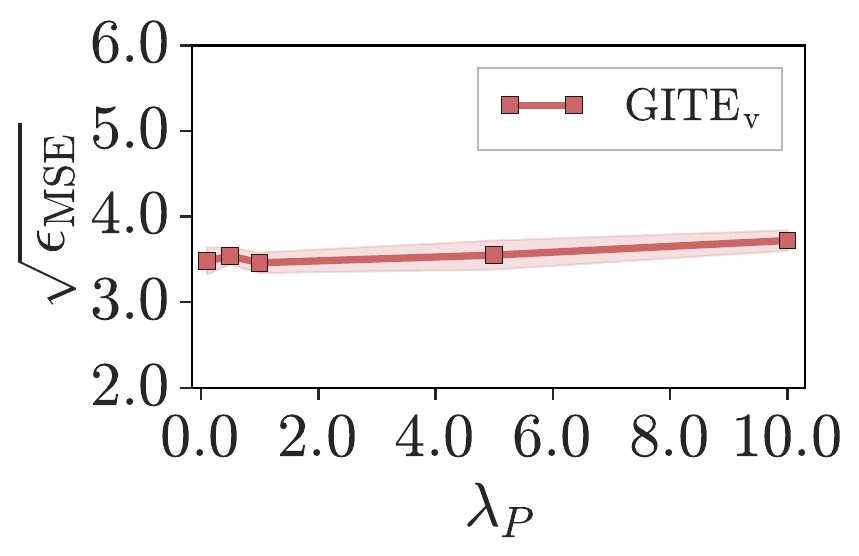}} \hspace{0pt} \subfloat[ Flickr,  $\lambda_P$,  $\sqrt{\epsilon_{\textrm{PEHE}}}$.]{\includegraphics[width=.24\textwidth]{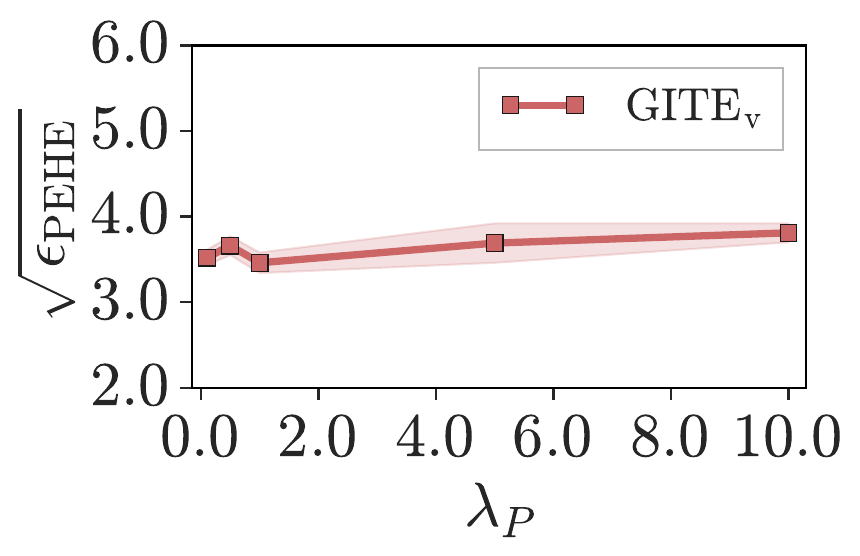}} \\
\hspace{-18pt} \subfloat[Blog,  $\lambda_D$,  $\sqrt{\epsilon_{\textrm{MSE}}}$.]{\includegraphics[width=.25\textwidth]{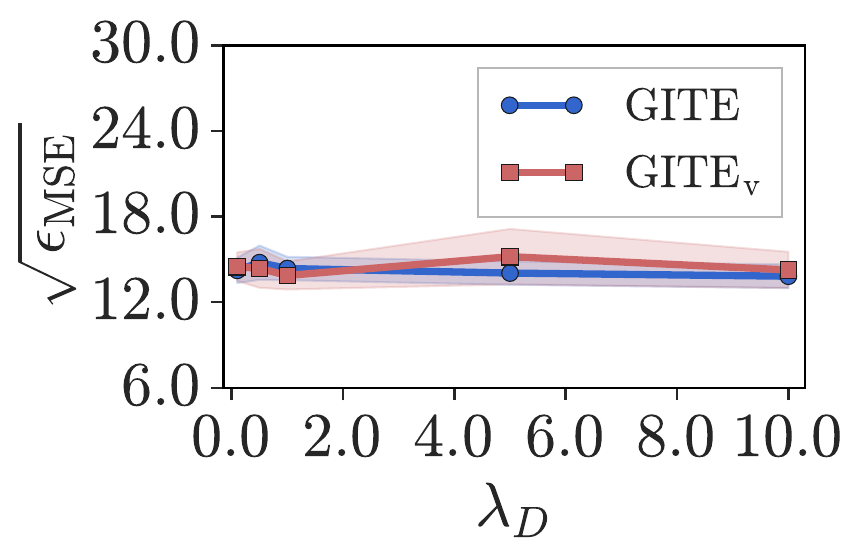}}\hspace{-2pt}
	\subfloat[Blog,  $\lambda_D$,  $\sqrt{\epsilon_{\textrm{PEHE}}}$.]
 {\includegraphics[width=.25\textwidth]{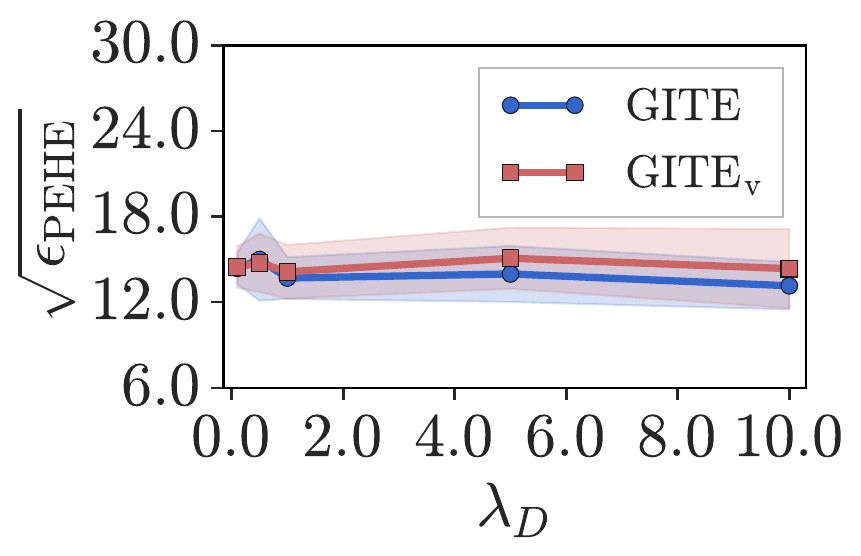}}
 \subfloat[Blog, $\lambda_P$, $\sqrt{\epsilon_{\textrm{MSE}}}$.]{\includegraphics[width=.25\textwidth]{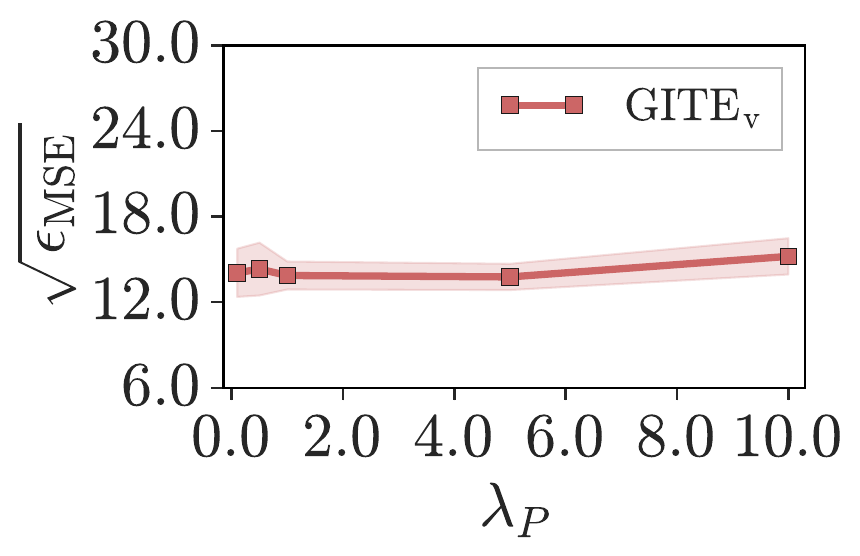}} 
	\subfloat[Blog,  $\lambda_P$, $\sqrt{\epsilon_{\textrm{PEHE}}}$.]
{\includegraphics[width=.25\textwidth]{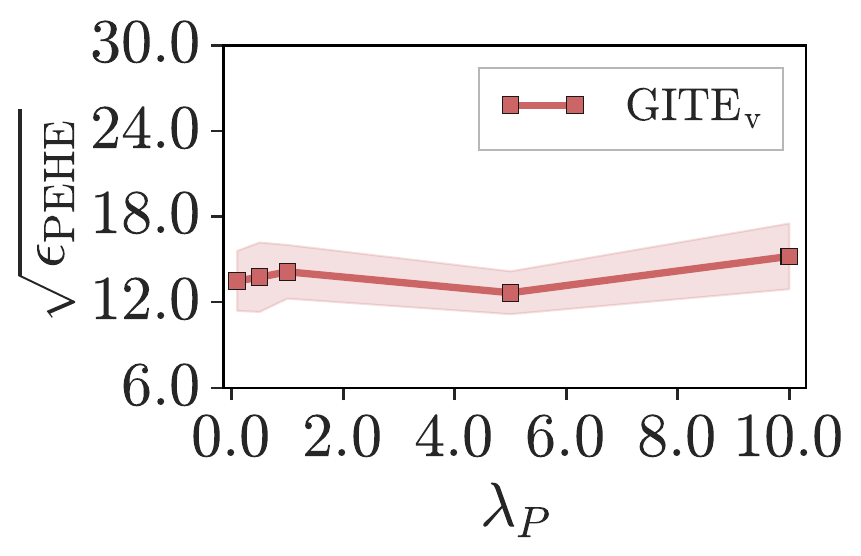}}

\caption{Results (mean and standard errors) of additional sensitivity experiments for hyperparameters $\lambda_D$ and $\lambda_P$. Results are averaged over ten executions.}\label{app:addsens}
\end{figure*}

\begin{table*}
\centering
\caption{Results (mean and standard errors) of experiments with different dimensions of different layers on the test sets.  
Results are averaged over ten executions. 
}
\vspace{0.0cm}
\begin{tabular}{c|ccccccccccc}
\hline
  & \multicolumn{2}{c}{AMZ-N}       
  & \multicolumn{2}{c}{Flickr}& \multicolumn{2}{c}{Blog}\\
Dimension & $\sqrt{\epsilon_{\textrm{MSE}}}$ & $\sqrt{\epsilon_{\textrm{PEHE}}}$    & 
$\sqrt{\epsilon_{\textrm{MSE}}}$ & $\sqrt{\epsilon_{\textrm{PEHE}}}$    & 
$\sqrt{\epsilon_{\textrm{MSE}}}$ & $\sqrt{\epsilon_{\textrm{PEHE}}}$     
\\ \hline \hline
100 & \textbf{0.75 ± 0.01}   & \textbf{1.21 ± 0.01} 
&
3.45 ± 0.12& 3.49 ± 0.20
& 13.99 ± 1.09& \textbf{13.33 ± 1.46}
\\
300  & 0.76 ± 0.01   & \textbf{1.21 ± 0.01} 
& 3.41 ± 0.17& \textbf{3.46 ± 0.11}
& 14.36 ± 0.82& 13.69 ± 1.46
\\\hline
\end{tabular}
\label{table:diffdimension}
\end{table*}

\subsection{RQ~4: is the training time of the proposed methods reasonable}\label{app:trainingtime}
This section investigates the answer to RQ~4. Experimental results are presented in Table~\ref{table:trainingtime}. Results show that the training time for the proposed methods is reasonable. 

Training efficiency might be further improved by applying acceleration technologies for GNN and partial attention mechanisms, such as neighbor sampling~\cite {hamilton2017inductive}. We discuss the time complexity of the NIM layer in Appendix~\ref{app:timecomplex}.  However, such techniques often involve a trade-off between efficiency and performance. Therefore, accelerating the implementation of the NIM layer while preserving the performance of ITE estimation with DNE can be a promising future research direction.
\begin{table}
\centering
\caption{Training time (in minutes) of GITE and GITE$_{\text{v}}$  for one execution.  
}\vspace{0.0cm}\begin{tabular}{l|ccccccccccc}\hline  Method& \multicolumn{1}{c}{AMZ-N}       & \multicolumn{1}{c}{Flickr}& \multicolumn{1}{c}{Blog} \\\hline \hline GITE  & 7.7  & 1.5 & 0.7 
\\GITE$_{\text{v}}$  & 7.9  & 1.5 & 0.8 \\\hline\end{tabular}\label{table:trainingtime}
\vspace{-5pt}
\end{table}

\subsection{RQ~5: can the IPAtt and SPAtt mechanisms be implemented with a different technology other than GAT}\label{appexp:diffatt}
This section investigates the answer to RQ~5.
In Table~\ref{table5:diffatt}, we compared performance for GITE with two different attention mechanisms: GAT~\cite{2017gat} and the attention mechanism of Transformer based on query and key vectors (QK-based AT)~\cite{vaswani2017attention,NEURIPS2021_f1c15925}. Results indicate that the proposed IPAtt and SPAtt mechanisms can achieve comparable performance with different attention mechanisms. This reveals that the IPAtt and SPAtt mechanisms can be implemented using an attention mechanism other than GAT, and do not rely on a specific attention mechanism.
\begin{table*}
\centering
\caption{Results (mean and standard errors) of experiments with different attention mechanisms on the test sets.  
Results are averaged over ten executions.
GITE (GAT) represents that GITE applies GAT for the IPAtt and SPAtt mechanisms, whereas GITE (QK-AT) represents that GITE applies QK-based attention mechanism for the IPAtt and SPAtt mechanisms.
}
\vspace{0.0cm}
\begin{tabular}{l|ccccccccccc}
\hline
  & \multicolumn{2}{c}{AMZ-N}       
  & \multicolumn{2}{c}{Flickr}& \multicolumn{2}{c}{Blog}\\
Method & $\sqrt{\epsilon_{\textrm{MSE}}}$ & $\sqrt{\epsilon_{\textrm{PEHE}}}$    & 
$\sqrt{\epsilon_{\textrm{MSE}}}$ & $\sqrt{\epsilon_{\textrm{PEHE}}}$    & 
$\sqrt{\epsilon_{\textrm{MSE}}}$ & $\sqrt{\epsilon_{\textrm{PEHE}}}$      
\\ \hline \hline
GITE (GAT) & \textbf{0.75 ± 0.01}   & \textbf{1.21 ± 0.01} 
&
3.41 ± 0.17& \textbf{3.46 ± 0.11}
& 14.36 ± 0.82& \textbf{13.69 ± 1.46}
\\
GITE (QK-AT)  & 0.76 ± 0.01   & \textbf{1.21 ± 0.01} 
&
\textbf{3.38 ± 0.18}& 3.51 ± 0.05& 
\textbf{13.58 ± 0.77}& 13.92 ± 1.79
\\\hline
\end{tabular}
\label{table5:diffatt}
\end{table*}
\begin{table*}
\centering
\caption{Results (mean and standard errors) of experiments with different strategies for the hyperparameter $\pi_{\eta}$ on the test sets.  
Results are averaged over ten executions.
$\pi_{\eta} = 1$ represents that GITE sets $\pi_{\eta}$ as a hyperparameter with a fixed value of 1, whereas Auto represents that GITE sets $\pi_{\eta}$ as a learnable parameter with an initial value of 1.
}
\vspace{0.0cm}
\begin{tabular}{l|ccccccccccc}
\hline
  & \multicolumn{2}{c}{AMZ-N}       
  & \multicolumn{2}{c}{Flickr}& \multicolumn{2}{c}{Blog}\\
Method & $\sqrt{\epsilon_{\textrm{MSE}}}$ & $\sqrt{\epsilon_{\textrm{PEHE}}}$    & 
$\sqrt{\epsilon_{\textrm{MSE}}}$ & $\sqrt{\epsilon_{\textrm{PEHE}}}$    & 
$\sqrt{\epsilon_{\textrm{MSE}}}$ & $\sqrt{\epsilon_{\textrm{PEHE}}}$     
\\ \hline \hline
$\pi_{\eta} = 1$ & \textbf{0.75 ± 0.01}   & \textbf{1.21 ± 0.01} 
&
3.41 ± 0.17& 3.46 ± 0.11
& 14.36 ± 0.82& 13.69 ± 1.46
\\
\text{Auto}  & \textbf{0.75 ± 0.01}   & \textbf{1.21 ± 0.01} 
&
\textbf{3.40 ± 0.14}& \textbf{3.43 ± 0.08}& 
\textbf{14.20 ± 1.77}& \textbf{13.32 ± 1.96}
\\\hline
\end{tabular}
\label{table6:diffMA}
\end{table*}
\subsection{RQ~6: can the message amplifier be implemented using different strategies}\label{appexp:diffMA}
This section investigates the answer to RQ~6.
As introduced in Equation~(\ref{eq:MAP}), we have two strategies for $\pi_{\eta}$ of the message amplifier: (I) setting it as a hyperparameter and (II) setting it as an adaptive learnable parameter.  In Table~\ref{table6:diffMA}, we conducted experiments to compare the performance between different strategies for $\pi_{\eta}$. Results indicate that different strategies can achieve comparable performance in outcome prediction and ITE estimation, but applying the strategy of learnable $\pi_{\eta}$ can sightly improve the performance of GITE.

\section{Impact}\label{app:Impact}
This paper proposes GITE, a method for estimating treatment effects from observational graph data. Although randomized controlled trials (RCTs) remain the gold standard for estimating individual treatment effects, they are often costly and time-consuming. In contrast, observational graph data offers a low-cost alternative. As such, the proposed approach holds potential for applications in various domains, including decision-making in commerce and medicine.

\section{Future work}\label{app:futurework}
We introduce four research directions that are promising for future exploration based on this study. First, extending ITE estimation from more convoluted graphs in the presence of DNE, such as hypergraphs~\cite {ma2022learning} and heterogeneous graphs~\cite{lin2023estimating}. In such convoluted graphs, DNE is more complex compared to ordinary graphs. Second, applying the proposed GITE to various areas for making reasonable decisions, such as  medicine~\cite{chang2023estimating,ma2022assessing,schnitzer2022estimands} and commerce~\cite{nabi2022causal}, as well as task-specific applications, such as recommendation systems~\cite {li2023tdrcl,li2023stabledr}.
Third, as discussed in Section~\ref{app:trainingtime}, although GITE exhibits reasonable training time, developing more efficient implementations of the NIM layer without performance degradation is a valuable direction. Finally, with  severe violations of the overlap assumption, our methods face limitations, which is a  constraint and a promising future direction.

\end{document}